\documentclass{article}
\PassOptionsToPackage{sort&compress}{natbib}
\usepackage[main,final]{neurips_2026}

\usepackage{algorithm}
\usepackage[noend]{algpseudocode}
\usepackage[utf8]{inputenc}
\usepackage{amssymb,amsmath,amsthm}
\usepackage{mathtools}
\usepackage{graphicx}
\usepackage{float}
\usepackage{xcolor}
\usepackage{hyperref}
\setcitestyle{numbers,square}
\usepackage[shortlabels]{enumitem}
\usepackage{booktabs}
\usepackage{threeparttable} 
\usepackage{mdframed}
\usepackage{caption}
\newif\ifpreprintmode
\preprintmodetrue  
\ifpreprintmode
  \newcommand{\repomain}{https://github.com/shaulTol/ABM}
\else
  \newcommand{\repomain}{https://anonymous.4open.science/r/ABM-1423/README.md}
\fi

\newcommand{\R}{\mathbb{R}}
\newcommand{\E}{\mathbb{E}}
\newcommand{\Var}{\mathrm{Var}}
\newcommand{\Cov}{\mathrm{Cov}}
\newcommand{\tr}{\mathrm{tr}}
\newcommand{\be}{\begin{equation}}
\newcommand{\ee}{\end{equation}}
\newcommand{\bea}{\begin{align}}
\newcommand{\eea}{\end{align}}

\newtheorem{theorem}{Theorem}
\newtheorem{proposition}{Proposition}
\newtheorem{corollary}{Corollary}
\newtheorem{lemma}{Lemma}
\newtheorem{definition}{Definition}

\newtheorem{remark}{Remark}

\title{Measuring and Decomposing Mode Separation\\via the Canonical Diffusion}

\author{
  Shaul Tolkowsky\\
  Hebrew University of Jerusalem\\
  \texttt{shaul.tolkowsky@mail.huji.ac.il}
  \And
  Ori Meidler\\
  Hebrew University of Jerusalem\\
  \texttt{ori.meidler@mail.huji.ac.il}
  \And
  Or Zuk\\
  Hebrew University of Jerusalem\\
  \texttt{or.zuk@mail.huji.ac.il}
}

\begin{document}
\maketitle

\begin{abstract}
Mode separation, namely how sharply a distribution fragments into barrier-separated clusters, is a fundamental geometric property of densities, difficult to quantify in high dimensions. Modes are often the meaningful alternatives in the data, such as the conformations of a molecule or the distinct semantic types among a generative model's outputs. Mode separation is distinct from dispersion, which is what standard label-free tools capture: differential entropy rises with spread regardless of fragmentation, and PCA orders directions by variance regardless of barriers. Mutual information does capture it, but only given the latent component labels, which are usually unavailable. We measure mode separation through a single stochastic process intrinsic to the density, namely the unique reversible diffusion with $f$ as its stationary distribution and a constant scalar diffusion coefficient. From its autocovariance matrix we extract two readouts. The first, SSA (Sum of Squared Autocorrelations), is a scalar barrier-sensitive measure. The second, DA (Dominant Autocorrelation directions), gives linear projections ordered by metastability rather than variance. Under an isotropic-Gaussian null, we derive the limiting spectrum of the empirical autocovariance in closed form, generalizing the Marchenko--Pastur law. Its upper edge gives a noise floor that separates metastable directions from sampling noise. Both readouts use only samples and a score function, scaling to high dimensions through pretrained score-based generative models via Tweedie's identity. We apply our framework to three settings: (i) synthetic Gaussian mixtures, where SSA tracks mutual information; (ii) SDXL text-to-image generations, where SSA and DA capture structure that entropy and PCA miss; and (iii) molecular dynamics of alanine dipeptide, where DA recovers the known slow backbone dihedrals from static samples alone.
\end{abstract}

\section{Introduction}
\label{sec:intro}

Mode separation is a geometric property of a density: how sharply its mass fragments into barrier-separated regions. Dispersion, how widely mass is spread, is a different property. Differential entropy rises with spread regardless of whether the mass is unimodal or fragmented; PCA orders directions by variance regardless of whether that variance comes from a wide mode or two narrow ones; mutual information requires a mixture decomposition one usually does not have. What is missing is an intrinsic, continuous quantity sensitive to barriers, together with a method for identifying the linear directions along which barriers are most pronounced.

We propose reading two quantities off a single \emph{canonical
diffusion} attached to $f$ (Theorem~\ref{thm:uniqueness}), depending
on $f$ only through its score and first two moments. Its generator's spectrum is the same object diffusion maps~\cite{coifman2006,nadler2005diffusion} target via graph Laplacians. We instead read off two new quantities from the time-lagged sample covariance. This sidesteps the dimensionality issues of kernel Laplacians and connects to score-based generative models via Tweedie's identity. \textbf{SSA}, the trace-integral of the autocovariance, is a scalar measure of mode separation with no analog in prior diffusion-process spectral methods. It is monotone in the generator eigenvalues (Proposition~\ref{prop:ssa_monotonicity}).
\textbf{DA}, the eigendecomposition of the same autocovariance, gives
linear directions ordered by metastability rather than variance. We also characterize the null spectrum of the empirical autocovariance exactly. Theorem~\ref{thm:limit-spectrum} shows it is the free additive convolution of two scaled Marchenko--Pastur laws at every $\gamma = d/N \in (0, \infty)$, with an analytic upper edge $\lambda_+(\tau)$. This edge is exact under the theorem's null (isotropic Gaussian, exact score, independent lag pairs). Elsewhere it serves as a reference floor, and we calibrate by matched Monte Carlo (Algorithm~\ref{alg:operational_floor}, Appendix~\ref{app:validation}).

We validate on synthetic GMMs, where SSA tracks mutual information
on rank ordering and diverges from entropy. On SDXL text-to-image
generations across ten prompts, SSA peaks at intermediate guidance
($w \in \{3, 7.5\}$) for every prompt, never at the extremes, while
entropy declines monotonically with guidance.
On alanine dipeptide, DA recovers the ordering of the known slow dihedral coordinates $\phi$ and $\psi$ from static samples alone.
Appendix~\ref{app:rebuttal_robustness} reports score-error, numerical, null-calibration, anisotropy, and nonlinear-feature controls.

\section{Related Work}
\label{sec:related}

\paragraph{Prior scalar measures of multimodality.}
Existing scalar quantities either operate in restricted settings or measure
a different property. The dip test~\cite{hartigan1985} is strictly
one-dimensional and returns a binary verdict on unimodality. Mutual
information $I(X;Z)$ between observations and a latent component label
requires access to the mixture decomposition, which in many applied scenarios is unavailable.
Clustering returns a discrete count rather than a degree of separation, and degrades in high dimensions where distances concentrate. Differential
entropy and its estimators measure spread rather than fragmentation, ranking
a wide unimodal distribution above a sharply bimodal one of equal variance
(Section~\ref{sec:exp_entropy}). SSA fills this gap: a continuous scalar depending on $f$ only through its score and first two moments, responding to barriers rather than dispersion.

\paragraph{Prior directional decompositions.}
Methods for extracting slow or mode-separating directions fall into two
families. Kernel-based approaches work from static samples: diffusion maps~\cite{coifman2006} embed samples via a data-driven heat kernel, and Nadler et al.~\cite{nadler2005diffusion} showed that the normalized graph-Laplacian spectrum converges to the eigen-system of a Fokker--Planck operator essentially identical to $\mathcal{L}_f$. DA thus targets the same spectral object diffusion maps estimate. Spectral Map~\cite{rydzewski2023,rydzewski2024} trains a
neural network to embed samples by maximizing a learned
spectral gap. Trajectory-based approaches require temporally ordered
simulation data: TICA~\cite{molgedey1994,perez2013}, the standard linear
method in molecular simulation, extracts slow directions from
time-lagged covariances of a trajectory. DA shares the generator spectrum of diffusion maps and Spectral Map but reads off as a ranked linear basis in $\R^d$ from the time-lagged covariance $\hat C(\tau)$ rather than a nonlinear sample embedding. DA admits the closed-form noise theory of Theorem~\ref{thm:limit-spectrum}, and uses pretrained denoisers as score oracles where kernel methods suffer the curse of dimensionality (e.g.\ SDXL latent space, $d = 65{,}536$). The scalar SSA, derived from the same generator, has no analog in any of these frameworks. 

\section{Problem setup and the canonical diffusion}
\label{sec:canonical}

Let $f : \R^d \to (0, \infty)$ be a probability density. We want to quantify how sharply its mass splits into regions separated by low-density barriers, and to identify the directions along which it splits. Barriers are hard to characterize directly from $f$ in high dimensions, but they have a clear dynamical signature: a diffusion with stationary density $f$ crosses them rarely and therefore mixes slowly.

Thus, to capture this for a given $f$, we attach to it a canonical stochastic process whose mixing reflects that barrier structure. A reversible It\^o diffusion with stationary density
$f$ has a self-adjoint generator $\mathcal{L}_f$ on $L^2(f\,dx)$ with
discrete spectrum $0 = \mu_1 < \mu_2 \leq \mu_3 \leq \cdots$ and
orthonormal eigenfunctions $\{\psi_k\}$; small eigenvalues correspond to
slow relaxation modes, each reflecting a barrier the process crosses
rarely, so the number and magnitude of near-zero eigenvalues encode the
metastable structure of $f$~\cite{Pavliotis2014}.
We single out a specific diffusion by reversibility with respect to $f$ and constant scalar diffusion $A_0 = \sigma_f^2 I_d$, $\sigma_f^2 = \tr(\Cov_f(X))/d$, giving the \emph{canonical diffusion} of $f$:
\begin{equation}
dX_t \;=\; \tfrac12\,\sigma_f^2\,\nabla \log f(X_t)\,dt \;+\; \sigma_f\,dW_t,
\label{eq:canonical_sde}
\end{equation}
with $W_t$ a standard $d$-dimensional Brownian motion. The construction
depends on $f$ only through its score $\nabla\log f$ and its first two moments. Motivations for each design choice are in Appendix~\ref{app:design_rationale}.

\begin{theorem}[Uniqueness of the canonical diffusion]
\label{thm:uniqueness}
Fix any $A_0\succ 0$ and let $f$ be sufficiently regular. Among reversible It\^o diffusions with stationary density $f$ and constant diffusion matrix $A_0$, the drift is uniquely $b(x) = \tfrac12 A_0\nabla\log f(x)$. (Detailed assumptions and proof: Appendix~\ref{app:uniqueness_proof}.)
\end{theorem}

\begin{proposition}[Similarity Invariance]
\label{prop:similarity_invariance}
The canonical diffusion is covariant under similarity transformations: if $Y = cQX + b$ for any $c > 0$, orthogonal $Q$, and $b \in \R^d$, then $Y_t = cQX_t + b$ is the canonical diffusion for the density of $Y$. In particular, the generator spectrum is preserved, so SSA and DA are invariant under rotation, uniform rescaling, and translation. (Proof: Appendix~\ref{app:similarity_proof}.)
\end{proposition}

\paragraph{Robustness to the coefficient choice.}
Other constant $A_0 \succ 0$ and kinetic Langevin dynamics change SSA and DA quantitatively, but in the deep-barrier regime they preserve the ordering of slow relaxation rates across densities (Appendix~\ref{app:robustness}).

\paragraph{Autocovariance and its spectral representation.}
Both of our readouts derive from the stationary autocovariance of the centered process $Y_t = X_t - \bar x_f$, where $\bar x_f = \E_f[X]$. Let $C(\tau) = \E_f[Y_\tau Y_0^\top]$ and $\rho(\tau) = \tr C(\tau) / \tr C(0)$. Define projection coefficients
$\alpha_k = \int (x - \bar x_f)\psi_k(x) f(x)\,dx$, with $\alpha_1 = 0$
since $\psi_1 \equiv 1$. Design choices in the definition of $\rho(\tau)$ (trace normalization, centering by $\bar x_f$) are motivated in
Appendix~\ref{app:design_rationale}. The empirical estimator from
$N$ canonical-diffusion trajectories of length $T+1$ (simulated via
the score; Algorithm~\ref{alg:ssa_da}) is the trajectory-averaged
lag-$\tau$ cross product
\begin{equation}
\hat C(\tau) \;:=\; \frac{1}{N(T+1-\tau)}
\sum_{n=1}^{N}\sum_{t=0}^{T-\tau}
(X_t^{(n)} - \bar x)(X_{t+\tau}^{(n)} - \bar x)^\top,
\label{eq:c_hat_def}
\end{equation}
with $\bar x = (1/[N(T+1)]) \sum_{n,t} X_t^{(n)}$, and
$\hat C^{\mathrm{sym}}(\tau) = \tfrac{1}{2}[\hat C(\tau) + \hat C(\tau)^\top]$.
At $\tau = 0$ this is the sample covariance. Theorem~\ref{thm:limit-spectrum} analyzes the simple-pair analogue exactly; we use the trajectory-averaged $\hat C(\tau)$ operationally (Section~\ref{sec:da}).

\begin{lemma}[Spectral representation of $C(\tau)$]
\label{lem:spectral}
The stationary autocovariance satisfies:
\[
C(\tau) = \sum_{k \geq 2} e^{-\mu_k \tau}\,\alpha_k \alpha_k^\top,
\quad
\rho(\tau) = \sum_{k \geq 2} e^{-\mu_k \tau}\, w_k,
\quad
w_k = \tfrac{\|\alpha_k\|^2}{\sum_{j \geq 2}\|\alpha_j\|^2}
\]
\end{lemma}

We prove this in Appendix~\ref{app:spectral} by applying the semigroup $P_\tau = e^{\tau\mathcal{L}_f}$ coordinate-wise to $Y$. Tracing $C(\tau)$ and integrating yields SSA (Section~\ref{sec:ssa}); eigendecomposing $C(\tau)$ yields DA (Section~\ref{sec:da}).

\section{SSA: a scalar measure of mode separation}
\label{sec:ssa}

\begin{definition}[SSA]
\label{def:ssa}
The SSA of $f$ is
\begin{equation}
S(f) \;:=\; \int_0^\infty \rho(\tau)^2\,d\tau.
\label{eq:ssa_score}
\end{equation}
\end{definition}

\begin{corollary}[Spectral formula for SSA]
\label{thm:ssa_spectral}
$S(f) \;=\; \displaystyle\sum_{k,\ell \geq 2} \frac{w_k w_\ell}{\mu_k + \mu_\ell}$,
with the convention $1/0 = +\infty$. $S(f)<\infty$ whenever the canonical diffusion has a spectral gap. (Proof: Appendix~\ref{app:spectral}.)
\end{corollary}

\begin{proposition}[Monotonicity in every eigenvalue]
\label{prop:ssa_monotonicity}
In the spectral formula of Corollary~\ref{thm:ssa_spectral}, holding the
weights $\{w_k\}_{k\ge 2}$ and all other eigenvalues fixed, $S(f)$ is
nonincreasing in each generator eigenvalue $\mu_k$, and strictly decreasing in $\mu_k$ whenever $w_k > 0$. The same holds for the finite-horizon score $S_T(f) = \int_0^T \rho(\tau)^2\,d\tau$. (Proof:
differentiate spectral formula from Corollary~\ref{thm:ssa_spectral};
Appendix~\ref{app:monotonicity_proof}.)
\end{proposition}

\noindent SSA depends on both the slow eigenvalues and the variance share $w_k$ each slow mode carries. Proposition~\ref{prop:ssa_monotonicity} is therefore a statement about the spectral formula with fixed weights, not a claim that every density deformation that deepens a barrier must increase SSA: a deformation that also shifts variance between modes changes the $w_k$ as well.

\paragraph{Empirical estimator.}
We estimate SSA from $N$ trajectories with step-size $\Delta t$ and horizon $T_{\max}$:

\begin{equation}
\widehat S \;:=\; \Delta t \sum_{\ell = 1}^{L_{\max}} \hat\rho(\ell\,\Delta t)^2,
\label{eq:ssa_estimator}
\end{equation}
where $L_{\max} = \lfloor T_{\max}/\Delta t \rfloor$,
with $\hat\rho(\tau_\ell) = \tr\hat C(\tau_\ell)/\tr\hat C(0)$ at lag $\tau_\ell = \ell\,\Delta t$. The horizon $T_{\max}$ is a bias--variance parameter: at large $\tau$, fewer effectively-independent lag samples make $\Var[\hat\rho(\tau)]$ grow while $\rho(\tau)^2$ decays. A principled stopping lag is the first $\tau$ at which $\hat\rho(\tau)^2$ falls below an upper bound on the null variance of the trace-normalized estimator under the hypothesis $X_0 \perp X_\tau$:
\begin{equation}
T_{\max} \;=\; \min\!\left\{\tau \,:\, \hat\rho(\tau)^2 \;\le\; \frac{1}{N}\right\},
\label{eq:tmax_stopping}
\end{equation}
with $1/N$ a distribution-free upper bound on the asymptotic null variance of $\hat\rho(\tau)$ that holds for any $f$ with finite fourth moments (derivation in Appendix~\ref{app:rho_clt}). Algorithm~\ref{alg:ssa_da} details the full procedure. When mixing is too slow to reach a principled $T_{\max}$ within budget, monotonicity justifies truncating at a common horizon under comparable weights, with bootstrap ordering certificates (App.~\ref{app:compute_limited}).

\section{DA: directional decomposition by metastability}
\label{sec:da}

We next present DA, the eigendecomposition of the same autocovariance $C(\tau)$ that SSA traces.

\begin{definition}[Dominant Autocorrelation directions]
\label{def:da}
For each lag $\tau > 0$, let
\[
C(\tau) = \sum_{j=1}^d \lambda_j(\tau)\, v_j(\tau) v_j(\tau)^\top,
\qquad
\lambda_1(\tau) \geq \lambda_2(\tau) \geq \cdots \geq \lambda_d(\tau),
\]
be the eigendecomposition of the population autocovariance. The DA
directions at lag $\tau$ are the eigenvectors
$v_1(\tau), v_2(\tau), \ldots, v_d(\tau)$, ordered by decreasing
eigenvalue. 
\end{definition}

\begin{proposition}[Convergence of the leading DA direction]
\label{prop:da_convergence}
If $\mu_2 < \mu_3$ and $\|\alpha_2\| > 0$, the leading unit-norm eigenvector $v_1(\tau)$ of $C(\tau)$ satisfies $1 - \langle v_1(\tau), \alpha_2\rangle^2 / \|\alpha_2\|^2 = O(e^{-2(\mu_3 - \mu_2)\tau})$ as $\tau \to \infty$. (Proof: Appendix~\ref{app:da_convergence_proof}.)
\end{proposition}

\paragraph{Beyond DA1.} Assume $\alpha_2, \ldots, \alpha_{m+1}$ are linearly independent and $\mu_{m+1} < \mu_{m+2}$. Then the top-$m$ eigenspace of $C(\tau)$ converges to $\mathrm{span}(\alpha_2, \ldots, \alpha_{m+1})$ as $\tau \to \infty$ (Appendix~\ref{app:da_convergence_proof}). Recovering the individual slow modes as the eigenvectors $v_2(\tau), \ldots, v_m(\tau)$ requires additional gap and non-collinearity conditions on $\{\alpha_k\}$ that we do not establish in general; we validate per-direction recovery empirically per experiment (Section~\ref{sec:exp_molecular}, Appendix~\ref{app:validation}). A theory is left to future work. Two scope conditions follow from the same $\alpha_k$: a slow mode with $\alpha_k = 0$ is invisible to both readouts, and a near-degenerate slow block identifies only the subspace, not a unique DA1 (Appendix~\ref{app:rebuttal_geometry}).

\noindent Proposition~\ref{prop:da_convergence} suggests taking $\tau \to \infty$, but this fails in finite samples: the empirical $\hat C^{\mathrm{sym}}(\tau)$ has a noise floor that the signal eigenvalues eventually fall below. We therefore characterize the null spectrum exactly and use its upper edge as the noise floor when selecting a finite $\tau^*$.

\begin{theorem}[Limit spectrum of the lagged autocovariance under the isotropic Gaussian null]
\label{thm:limit-spectrum}
Let $f = \mathcal{N}(0,\sigma^2 I_d)$, $\gamma = d/N$, and fix $\tau > 0$. Let $(X_0^{(n)}, X_\tau^{(n)})_{n=1}^N$ be independent stationary OU lag-pairs from the canonical diffusion of $f$ (within each pair, $X_0$ and $X_\tau$ have the OU correlation $e^{-\tau/2}$; pairs are independent across $n$), and define the simple-pair estimator $\widetilde C(\tau) = (1/N) \sum_n X_0^{(n)}(X_\tau^{(n)})^\top$ and $\widetilde C^{\mathrm{sym}}(\tau) = \tfrac12(\widetilde C(\tau) + \widetilde C(\tau)^\top)$. In the proportional regime $d, N \to \infty$ with $d/N \to \gamma \in (0, \infty)$, the empirical spectral distribution of $\widetilde C^{\mathrm{sym}}(\tau)$ converges weakly almost surely to
\begin{equation}
\mathrm{d}\mu_\tau^\gamma(\lambda) = \max(0, 1 - 2/\gamma)\,\delta_0(\mathrm{d}\lambda) + \tilde\rho_\tau^\gamma(\lambda)\,\mathrm{d}\lambda,
\label{eq:mu_tau_decomp}
\end{equation}
where $\tilde\rho_\tau^\gamma$ is the continuous density of total mass $\min(1, 2/\gamma)$ (the normalized continuous density is $\rho_\tau^\gamma := \tilde\rho_\tau^\gamma / \min(1, 2/\gamma)$). Let $c_\pm(\tau) = (\sigma^2/2)(1 \pm e^{-\tau/2})$ and define
\begin{equation}
a(\tau) = c_+ c_- \gamma^2 \lambda, \quad b(\tau) = \gamma\bigl[c_+ c_- (2-\gamma) + \lambda(c_+ - c_-)\bigr], \quad c(\tau) = (c_+ - c_-)(1 - \gamma) - \lambda.
\label{eq:abc-coefficients}
\end{equation}
\begin{enumerate}[label=(\roman*),leftmargin=*,nosep]
\item \textbf{Edges} and support: the boundary points of the continuous support are real roots of $\Delta_\tau$, the discriminant of the Stieltjes cubic \eqref{eq:stieltjes-cubic} (a quartic in $\lambda$)
\begin{equation}
\Delta_\tau(\lambda) \;=\; 18\,a(\tau)\,b(\tau)\,c(\tau) - 4 b(\tau)^3 + b(\tau)^2 c(\tau)^2 - 4 a(\tau)\, c(\tau)^3 - 27 a(\tau)^2 \;=\; 0.
\label{eq:edge-quartic}
\end{equation}
We write $\lambda_+(\tau)$ for the largest real root and $\lambda_-(\tau)$ for the smallest; in the supercritical regime $\gamma > 2$, the quartic can have four real roots and the continuous support is two intervals separated by a gap around $0$.
\item \textbf{Density:} on the continuous support, $\tilde\rho_\tau^\gamma(\lambda) = -\pi^{-1}\,\mathrm{Im}\,G(\lambda+i0^+)$, where $G$ is the lower-half-plane root of the Stieltjes cubic ($\Delta_\tau$ is its discriminant)
\begin{equation}
a(\tau)\,G^3 + b(\tau)\,G^2 + c(\tau)\,G + 1 \;=\; 0.
\label{eq:stieltjes-cubic}
\end{equation}
\end{enumerate}
At $\tau = 0$, $\mu_0^\gamma$ is the standard Marchenko--Pastur law with edge $\lambda_+(0) = \sigma^2(1+\sqrt\gamma)^2$. As $\tau \to \infty$,
\begin{equation}
\lambda_+(\tau) \to \lambda_+(\infty) = \frac{\sigma^2}{8}\bigl(\sqrt{1+4\gamma}+3\bigr)^{3/2}\sqrt{\sqrt{1+4\gamma}-1}.
\label{eq:edge-tauinf-main}
\end{equation}
\end{theorem}

\paragraph{Operational noise floor.}\noindent\label{rem:floor-scope}
Under the null $f = \mathcal{N}(0, \sigma^2 I_d)$ there are no slow
modes, so the spectrum of $\hat C^{\mathrm{sym}}(\tau)$ is finite-sample
noise. Theorem~\ref{thm:limit-spectrum} characterizes it, giving
$\lambda_+(\tau)$ as a threshold: slow modes persist above it,
non-metastable directions decay into it. The proof (Appendix~\ref{app:limit-spectrum}) is for the simple-pair estimator $\widetilde C^{\mathrm{sym}}(\tau)$ with i.i.d.\ lag-pairs. The reported alanine and SDXL DA results are read at the terminal lag, where each trajectory contributes one $(X_0,X_T)$ pair and so matches that geometry exactly. At shorter lags the overlapping windows of $\hat C(\tau)$ are dependent, so a threshold at the nominal pair count can be anti-conservative (Appendix~\ref{app:rebuttal_null}).
A practical issue is that once $\nabla \log f$ is estimated rather than
known, and especially when DA is computed after a nonlinear embedding,
the empirical bulk is typically inflated above the idealized null. Section~\ref{sec:exp_synth_validation} verifies $\rho_\tau^\gamma$ and $\lambda_+(\tau)$ against empirical eigenvalue histograms under the exact null. For pretrained generative models we calibrate the floor by a
Monte-Carlo null through the full pipeline using a parallel iso-Gaussian
score oracle (Section~\ref{sec:score_high_dim}), so null and data
share the pipeline biases
(Algorithm~\ref{alg:operational_floor},
Appendix~\ref{app:sdxl_elderly}).

\paragraph{Lag-selection criterion and number of metastable directions.}
For $m \geq 1$ directions, we choose $\tau^*(m)$ as the largest lag that satisfies three conditions. (i) Each of the top $m$ eigenvalues of $\hat C^{\mathrm{sym}}(\tau)$ rejects the null at level $\alpha$, Bonferroni-corrected over the $m$ directions, using either the analytic null of Theorem~\ref{thm:limit-spectrum} or a matched Monte Carlo null (Appendix~\ref{app:validation}). (ii) The 1D projection onto each of $v_1, \ldots, v_m$ rejects unimodality under a 1D multimodality test, either Hartigan's dip test~\cite{hartigan1985} or Silverman's bandwidth test~\cite{silverman1981} (Appendix~\ref{app:validation}). This condition separates a barrier-separating direction from one that is merely slow: an anisotropic but unimodal Gaussian projects to a Gaussian along every direction, so it cannot reject (Appendix~\ref{app:rebuttal_geometry}). (iii) $\hat\rho(\tau) < 1/e$, a heuristic guard against under-converged short lags (Appendix~\ref{app:lag_selection}). The number of metastable directions recovered is the maximal $m$ for which $\tau^*(m)$ exists. Full methodology and per-experiment results are in Appendix~\ref{app:validation}.

\section{Score oracles and algorithms}
\label{sec:score_high_dim}

The canonical diffusion requires only the score $\nabla \log f$, not the density. At low-to-moderate $d$ we use kernel density estimation, as in the molecular experiments of Section~\ref{sec:exp_molecular}. For score-based generative models, Tweedie's identity (derived in Appendix~\ref{app:tweedie}) gives $\nabla \log f_t(x) = -\epsilon_\theta(x, t) / \sigma_t$ at noise level $t$ with schedule variance $\sigma_t^2$; at small $t$ this provides a score oracle for $f$ via a single denoiser pass with no density estimation and no dimensionality barrier. At large $t$ the same denoiser approximates the linear iso-Gaussian score $-x/\sigma_t^2$ used in Algorithm~\ref{alg:operational_floor}'s Monte-Carlo null. The mechanism extends to any generative model that learns a denoiser, score, or velocity field across noise levels (DDPMs~\cite{ho2020denoising}, score matching~\cite{song2019generative}, continuous-time diffusion~\cite{song2021scorebased}, flow matching~\cite{lipman2023flow}).

\paragraph{Algorithm.}
\label{sec:algorithms}
Algorithm~\ref{alg:ssa_da} consolidates both readouts into a single procedure: we simulate the canonical diffusion, estimate $\hat C(\tau)$ across a lag grid, then read off SSA from the trace and DA from the eigendecomposition. The two share all computation up to the autocovariance estimate. SSA integrates $\hat\rho(\tau)^2$ up to the stopping lag $L^*$ implied by the null floor of Appendix~\ref{app:rho_clt}; DA eigendecomposes $\hat C(\tau^*)$ at the lag $\tau^*(\hat m)$ selected by the joint multi-direction criterion of Section~\ref{sec:da}, with $\hat m$ the maximal number of recovered metastable directions; Appendix~\ref{app:lag_selection} gives the full criterion.

\begin{algorithm}[!htbp]
\caption{SSA and DA from a Score Oracle}
\label{alg:ssa_da}
\begin{algorithmic}[1]
\State \textbf{Input:} $\nabla \log \hat{f}$, samples from $f$, trajectories $N$, step-size $\Delta t$, max steps $T$, lag grid $\{0 = \tau_0 < \cdots < \tau_L\}$
\State \textbf{Output:} SSA score $\hat S$; DA directions $\{v_j, \lambda_j\}_{j=1}^d$
\State sample $X_0^{(n)} \sim f$ for $n = 1, \ldots, N$
\State $\sigma_f^2 \gets \tfrac{1}{d}\tr(\widehat{\Cov}(X_0))$ \Comment{sample covariance}
\For{$n = 1, \ldots, N$,\; $t = 0, \ldots, T-1$} \Comment{simulate canonical diffusion}
    \State $X_{t+1}^{(n)} \gets X_t^{(n)} + \tfrac{1}{2}\sigma_f^2\,\nabla \log \hat f(X_t^{(n)})\,\Delta t + \sigma_f \sqrt{\Delta t}\, Z$, \;\; $Z \sim \mathcal{N}(0, I_d)$
\EndFor
\State $\bar x \gets \tfrac{1}{N(T+1)}\sum_{n,t} X_t^{(n)}$ \Comment{pooled sample mean}
\For{$\ell = 0, 1, \ldots, L$} \Comment{autocovariance and trace-normalized autocorrelation}
    \State $\hat C(\tau_\ell) \gets \tfrac{1}{N(T+1-\tau_\ell)} \sum_{n,t} (X_t^{(n)} - \bar x)(X_{t+\tau_\ell}^{(n)} - \bar x)^\top$
    \State $\hat\rho(\tau_\ell) \gets \tr \hat C(\tau_\ell)/\tr \hat C(0)$
\EndFor
\State $L^* \gets \min\{\ell \ge 1 : \hat\rho(\tau_\ell)^2 \le 1/N\}$, or $L^* \gets L$ if no such $\ell$ \Comment{SSA floor, Appendix~\ref{app:rho_clt}}
\State $\hat S \gets \int_0^{\tau_{L^*}} \hat\rho(\tau)^2\,d\tau$
\State $\tau^* \gets \tau^*(\hat m)$, the lag-selection criterion of Section~\ref{sec:da} \Comment{DA lag}
\State $\hat C^{\mathrm{sym}} \gets \tfrac{1}{2}[\hat C(\tau^*) + \hat C(\tau^*)^\top]$
\State Eigendecompose $\hat C^{\mathrm{sym}} = \sum_j \lambda_j\, v_j v_j^\top$, $\lambda_1 \geq \lambda_2 \geq \cdots$ \Comment{DA directions}
\State \Return $\hat S$,\; $\{v_j, \lambda_j\}_{j=1}^d$
\end{algorithmic}
\end{algorithm}

\section{Experiments}
\label{sec:experiments}

\paragraph{Validation on synthetic distributions.}
For GMMs with known component labels $Z$, the mutual information $MI(X;Z)$ measures how well a single observation identifies its source component, providing a natural ground-truth that requires access to the mixture decomposition; see Appendix~\ref{app:mi_baseline}. We evaluate SSA against MI on 10D GMMs across four parameter sweeps in Figure~\ref{fig:gmm_sweeps}; Appendix~\ref{app:gmm_details} gives simulation details. SSA tracks MI across all four sweeps. The two diverge in shape rather than ordering: where MI saturates near $H(Z) = \log k$ as barriers deepen, SSA accelerates exponentially in barrier height, since Corollary~\ref{thm:ssa_spectral} contributes $w_k w_\ell / (\mu_k + \mu_\ell)$ for each slow-mode pair with $\mu_k \sim e^{-\Delta U}$ in the metastable regime. This is visible in the separation and variance sweeps, where MI plateaus while SSA keeps climbing. In the weight sweep both decrease away from equal weights and agree in ordering except near $\pi_1 = 0.5$, where both are flat; the mode-count sweep is consistent with Corollary~\ref{thm:ssa_spectral}. Across 50 randomly configured bimodal GMMs in $d = 10$, Spearman rank correlation is $\rho = 0.890$ ($p = 5.5 \times 10^{-18}$; Figure~\ref{fig:random_sweep}, Appendix). Rankings are stable across a $4\times$ range of Euler--Maruyama step sizes (Spearman $\geq 0.97$; Appendix~\ref{app:dt_robustness}).

\begin{figure}[!htbp]
\centering
\includegraphics[width=\textwidth]{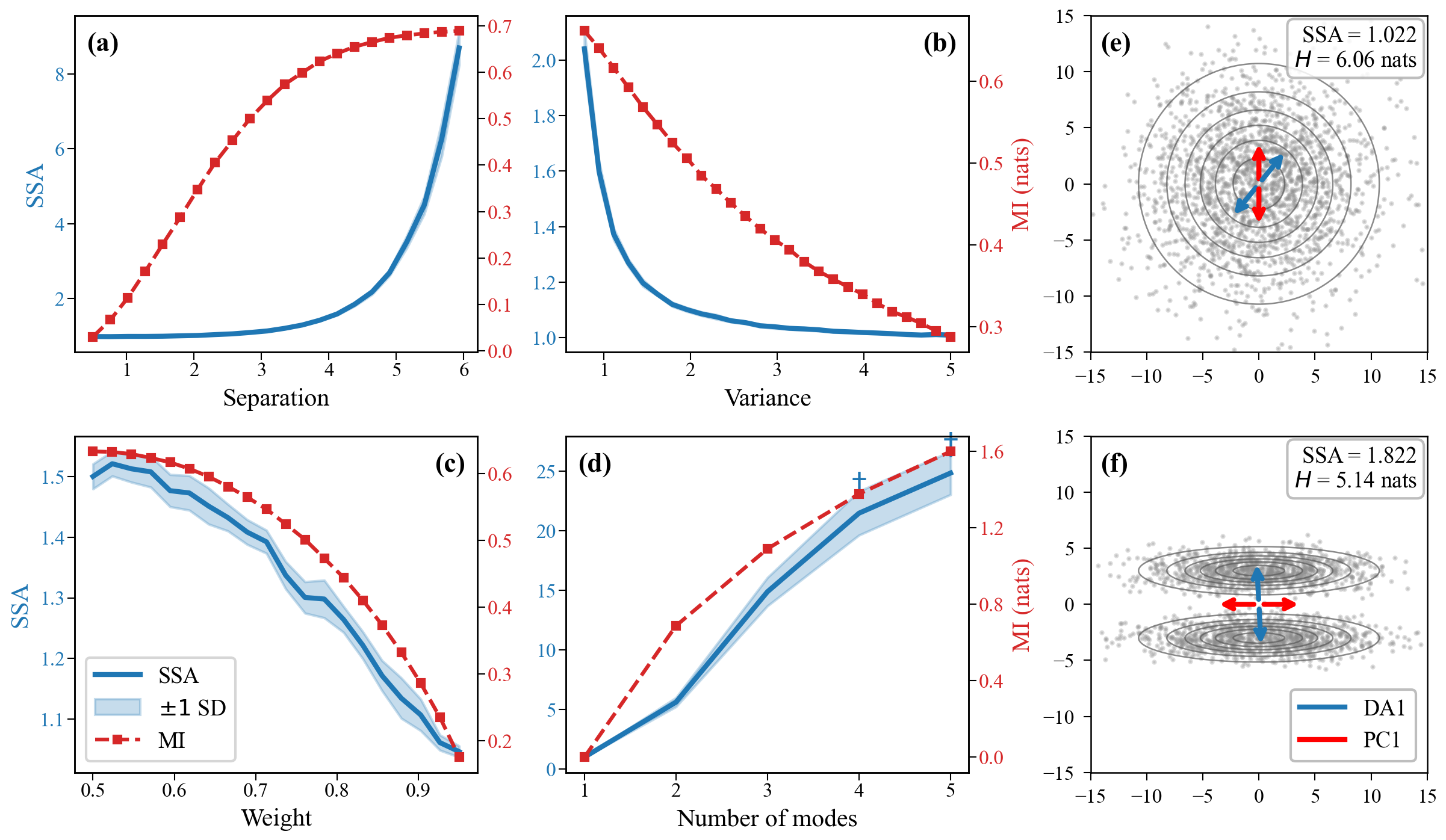}
\caption{ \textbf{Two readouts of the canonical diffusion.}
\textbf{(a--d) SSA on 10D GMMs} (left 2/3): SSA (blue, left axis) and MI (red, right axis) across sweeps in (a) mode separation, (b) isotropic variance at fixed separation, (c) weight asymmetry, (d) mode count. SSA tracks MI in ordering; MI saturates near $H(Z) = \log k$ while SSA grows with barrier height. Mean over $10$ seeds, $\pm 1$ SD. Daggered points ($K{=}4,5$) did not meet the $1/N$ criterion at $T_{\mathrm{phys}}=200$\,s; since $\rho^2 \ge 0$, truncation makes $\hat S$ a lower bound on the population value (ordering preserved, Appendix~\ref{app:gmm_bootstrap}).
\textbf{(e--f) DA vs.\ PCA on 2D distributions} (right 1/3): (e) isotropic $\mathcal{N}(0, 25\,I_2)$, $H = 6.06$, SSA$\,=1.022$: no population direction, arrows reflect finite-sample noise. (f) Parallel ellipses, means $(0, \mp 3)$, covariances $\mathrm{diag}(25, 1)$, $H = 5.14$, SSA$\,=1.822$: entropy ranks (e) above (f); SSA inverts. DA1 points between modes; PC1 along max variance. DA1 passes anisotropic-null and dip tests ($p < 10^{-3}$); PC1 fails dip ($p = 0.99$).
Details: Appendices~\ref{app:gmm_details},~\ref{app:validation}; snapshots: Figure~\ref{fig:gmm_snapshots_appendix}.}
\label{fig:gmm_sweeps}
\label{fig:structure}
\end{figure}

\label{sec:exp_entropy}
\label{sec:exp_parallel_ellipses}
Figure~\ref{fig:structure} (panels (e),(f)) contrasts two 2D distributions chosen to expose the gap between dispersion and fragmentation. The isotropic Gaussian (e) is more dispersed than the parallel-ellipses GMM (panel f), and entropy reflects this ($6.06$ vs.\ $5.14$ nats). SSA inverts the ordering ($1.022$ vs.\ $1.822$), identifying the inter-mode barrier that entropy is blind to; the isotropic value $\approx 1$ matches the closed-form unimodal baseline (Appendix~\ref{app:ou_autocovariance}). The same figure illustrates DA: on the parallel ellipses, PCA selects the high-variance $x$-axis ($25 > 10$), while DA selects the inter-mode $y$-axis. On the isotropic blob the population autocovariance is a scalar multiple of $I_2$ at every lag, so neither method has a preferred population direction (``DA1'' here fails the dip test). DA1 on the ellipses passes both an anisotropic-Gaussian null test ($p < 10^{-3}$) and Hartigan's dip test for unimodality ($p < 10^{-3}$); PC1 fails the dip test ($p = 0.99$), correctly flagged as dispersion (Appendix~\ref{app:validation}).

\label{sec:exp_synth_validation}
We also verify Theorem~\ref{thm:limit-spectrum} directly: drawing $N$ independent OU lag-pairs from $\mathcal{N}(0, I_d)$ at multiple aspect ratios $\gamma = d/N$ and computing $\widetilde C^{\mathrm{sym}}(\tau)$ at $\tau \in \{0, 0.5, 2, 10\}$, the empirical eigenvalue histograms match the analytic free-convolution density $\rho_\tau^\gamma$ tightly, and the largest empirical eigenvalue concentrates at the analytic edge $\lambda_+(\tau)$ (Figure~\ref{fig:free_conv_overlay_body}). The full $6 \times 4$ sweep across $\gamma \in \{0.25, 0.5, 0.75, 1, 2, 5\}$ (including the atom regime $\gamma > 2$) is in Appendix~\ref{app:scope}.

\begin{figure}[H]
\centering
\includegraphics[width=\textwidth, keepaspectratio=false]{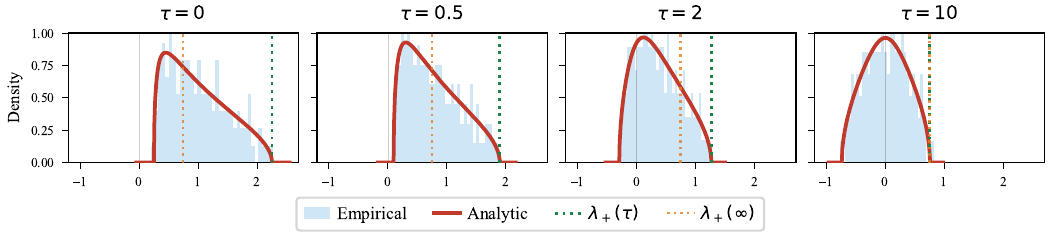}
\caption{Empirical eigenvalue distribution of the simple-pair $\widetilde C^{\mathrm{sym}}(\tau)$ under the isotropic Gaussian null $f = \mathcal{N}(0, I_d)$, at $\gamma = 0.25$ across $\tau \in \{0, 0.5, 2, 10\}$ (left to right); $N = 500$, single seed. Histogram: empirical; red: analytic density $\rho_\tau^\gamma$ from Theorem~\ref{thm:limit-spectrum}. Dark-green dotted: $\lambda_+(\tau)$; orange dotted: $\lambda_+(\infty)$ from Equation~\eqref{eq:edge-tauinf-main}; the two converge as $\tau$ grows. Full $6 \times 4$ sweep across $\gamma \in \{0.25, 0.5, 0.75, 1, 2, 5\}$ (including the atom regime $\gamma > 2$) in Appendix~\ref{app:scope}.}
\label{fig:free_conv_overlay_body}
\end{figure}

\begin{table}[!htbp]
\centering
\caption{ SSA across ten prompts and four guidance scales
(SDXL Base~1.0, $N\!=\!200$ trajectories per cell; every cell reaches the
$1/N$ stopping floor). Bold: per-prompt peak. All bootstrap SEs
($B = 2{,}000$) are $\le 0.05$ (full per-cell SEs in
Appendix~\ref{app:sdxl_bootstrap}). Four further prompts with peak
$\hat S < 1.1$ are excluded as inconclusive
(Appendix~\ref{app:sdxl_subgaussian}); the $1.1$ cutoff is a margin above
the unimodal-Gaussian baseline of $1$
(Appendix~\ref{app:ou_autocovariance}).}
\label{tab:sdxl_guidance}
\begin{tabular}{@{}lcccccccccc@{}}
\toprule
$w$ & person & elderly & food & car & bat & crane & mouse & seal & jaguar & palm \\
\midrule
$1$   & 0.99 & 0.96 & 0.88 & 0.92 & 0.98 & 1.04 & 0.95 & 0.89 & 0.93 & 1.18 \\
$3$   & 1.00 & 0.96 & 1.00 & 1.07 & \textbf{1.30} & 1.41 & \textbf{1.31} & \textbf{1.22} & \textbf{1.12} & \textbf{1.26} \\
$7.5$ & \textbf{1.32} & \textbf{1.17} & \textbf{1.19} & \textbf{1.13} & 1.28 & \textbf{1.54} & 1.19 & 1.17 & 0.99 & 1.25 \\
$15$  & 1.20 & 0.94 & 0.96 & 0.93 & 1.01 & 1.32 & 0.92 & 0.82 & 0.82 & 0.98 \\
\bottomrule
\end{tabular}
\end{table}

\paragraph{SSA and DA on SDXL generations in CLIP space.}
\label{sec:exp_sdxl}

We apply the framework to output distributions of SDXL
models~\cite{rombach2022ldm,podell2024sdxl}. The Tweedie identity
(Section~\ref{sec:score_high_dim}) gives a score oracle via the
denoiser, and the canonical diffusion runs in SDXL's latent space;
each trajectory snapshot is decoded through the VAE and embedded via
CLIP ViT-L/14~\cite{radford2021learning} into $\R^{768}$, where
    $\hat C(\tau)$ is computed (Appendix~\ref{app:sdxl_details}).
Both readouts therefore measure persistence under the latent dynamics as expressed in CLIP features, not the canonical diffusion of the CLIP-space density (Appendix~\ref{app:scope}).

Across fourteen prompts at four guidance scales each, entropy
declines monotonically in $w$ for every prompt (full values in
Appendix~\ref{app:sdxl_subgaussian}) Four prompts stay close to unimodal at every scale (peak $\hat S < 1.1$; the margin above the isotropic baseline $1$ allows for anisotropy, which alone raises SSA above $1$; Appendices~\ref{app:ou_autocovariance},~\ref{app:rebuttal_geometry}). For the remaining ten (Table~\ref{tab:sdxl_guidance}), $\hat S$ peaks at intermediate $w \in \{3, 7.5\}$, never at the extremes: at $w = 7.5$ for five and at $w = 3$ for five. Every cell reaches the $1/N$ stopping floor, so no value is truncation-limited. In a trajectory bootstrap ($B = 2{,}000$), the maximizer lies at an intermediate $w$ in at least $99.8\%$ of replicates for each of the ten prompts; which of the two intermediate scales wins is less certain for bat and palm, whose values there differ by less than one SE (Appendix~\ref{app:sdxl_bootstrap}).
Figure~\ref{fig:sdxl_combined} (left) shows the
``person'' prompt as a representative case: intermediate $w$
produces structured variation across distinct coherent types, while
$w = 15$ collapses to near-identical repeats, reproducing the fragmentation-vs-dispersion results of
Figure~\ref{fig:structure} on a real generative model.

DA on these distributions requires a finer integration step than
SSA (Appendix~\ref{app:sdxl_elderly}). For ``a portrait of an elderly person'' at $w = 7.5$, DA1
separates the distribution along a gender axis (Spearman
$|\rho| = 0.90$), visible in Figure~\ref{fig:sdxl_combined} (right),
and is distinct from PC1, which mixes gender with realism and style
(Appendix~\ref{app:sdxl_elderly}). The two-stage criterion gives
$m = 1$ (Appendix~\ref{app:sdxl_elderly}). The highest SSA value (crane, $1.54$) comes from a lexically ambiguous prompt, for which SDXL produces both senses of the word; all six ambiguous prompts (bat, crane, mouse, seal, jaguar, palm) clear the $1.1$ cutoff. On crane, DA1 and PC1 coincide, capturing the bird-vs-equipment split (Appendix~\ref{app:sdxl_crane}).

\begin{figure}[!htbp]
\centering
\includegraphics[width=\textwidth]{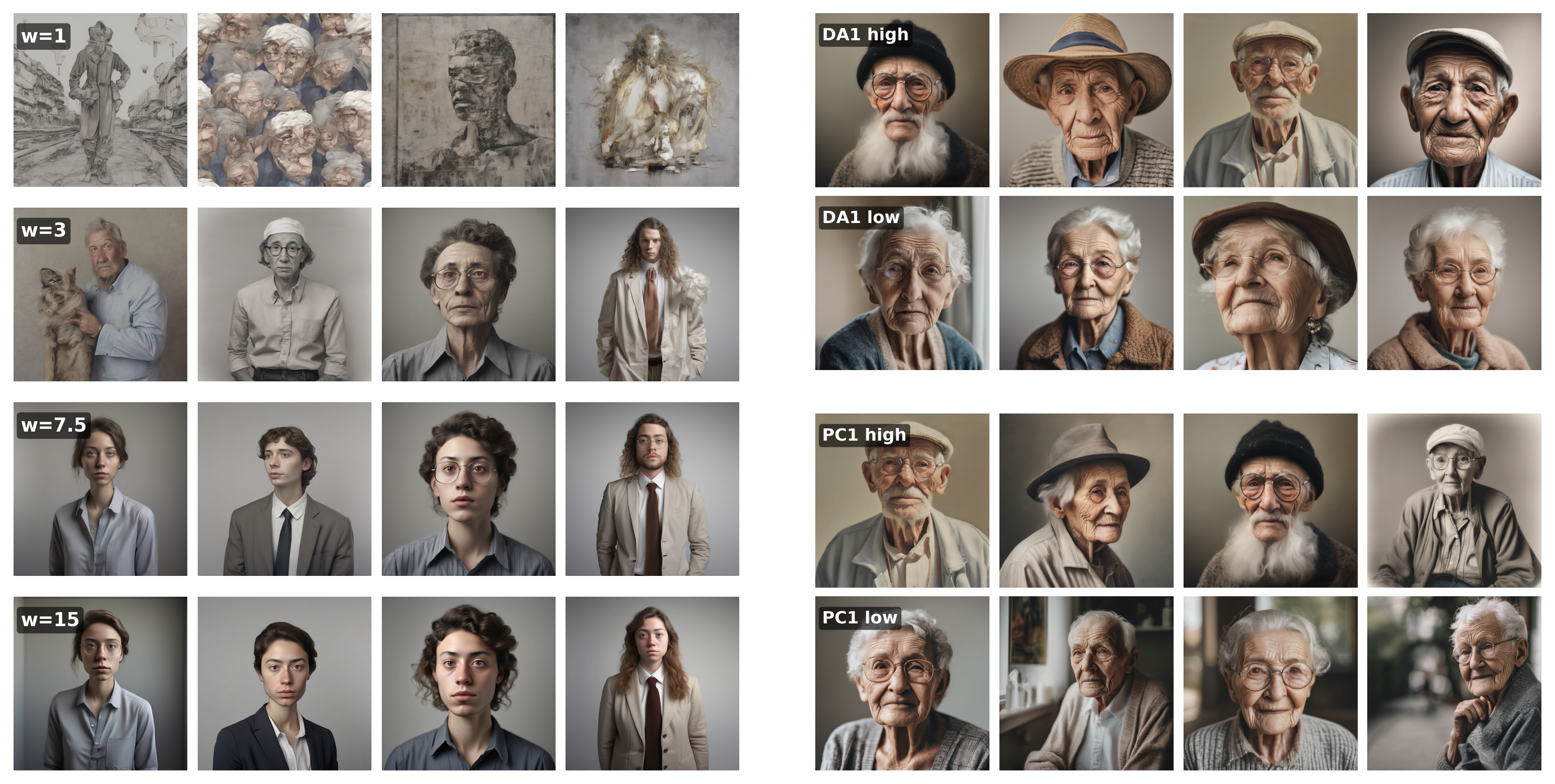}
\caption{ \textbf{SSA and DA on SDXL generations.} \textbf{Left:} Four representative samples per guidance scale for the prompt ``a professional photograph of a person'' (rows: $w \in \{1, 3, 7.5, 15\}$). At $w = 1$ outputs are diverse but weakly prompt-faithful; at $w = 3$ they become mostly on-prompt with broad stylistic variation; at $w = 7.5$ they show structured variation across distinct coherent types (highest SSA, Table~\ref{tab:sdxl_guidance}); at $w = 15$ the distribution narrows onto a small number of repeated types. The same qualitative pattern holds for ``elderly''. \textbf{Right:} Top-4 images sorted along DA1 (top two rows) and PC1 (bottom two rows) for ``a portrait of an elderly person'' at $w = 7.5$, $\tau^* = 2{,}000$, $\Delta t = 0.0025$. DA1 separates cleanly by gender (Spearman $|\rho| = 0.90$); PC1 mixes gender with realism and style. Spectrum and diagnostics in Appendix~\ref{app:sdxl_elderly}.}
\label{fig:sdxl_combined}
\end{figure}

\paragraph{Molecular conformational landscapes.}
\label{sec:exp_molecular}
A central problem in computational chemistry is identifying the slow
collective variables (CV    s) that resolve rare transitions between
metastable conformations. The barrier structure that fragments the
density is the same that makes physical transitions rare, so DA
directions on the canonical diffusion should recover the slow
coordinates that TICA~\cite{perez2013,molgedey1994} extracts from
trajectory data. We test this on alanine dipeptide, a 22-atom
benchmark whose conformational landscape is governed by two backbone dihedrals $\phi, \psi$, using the standard MDShare dataset~\cite{mdshare}
($3 \times 250$\,ns at 300\,K). The system is featurized as $45$
pairwise heavy-atom distances with Gaussian KDE scores
(Appendix~\ref{app:alanine_details}); PCA and TICA serve as linear baselines, with diffusion maps and
Spectral Map~\cite{rydzewski2023} as nonlinear baselines. 

DA at $\tau^* = 10{,}000$\,ps recovers the same $\psi$-then-$\phi$
ordering as TICA from \emph{static samples alone}, while the
trajectory-free baselines fall short: Spectral Map~\cite{rydzewski2023} misses $\psi$ entirely, PCA mixes $\psi$ and $\phi$ in both components, and diffusion maps
recover both dihedrals but with weaker dominance per direction
than DA (Table~\ref{tab:alanine}, Figure~\ref{fig:alanine_da_tica}); the
correlations reflect a linear-method ceiling, since dihedrals are
nonlinear functions of pairwise distances. The two-stage criterion of
Section~\ref{sec:da} gives $\hat m = 2$: $\lambda_1, \lambda_2$ both
reject the null (Appendix~\ref{app:alanine_bonferroni})
and reject unimodality, while no $\tau^*(m)$ exists for $m \geq 3$
(Table~\ref{tab:alanine_bonferroni}).
The chosen $\tau^*$ sits inside the window
$1/(\mu_4 - \mu_3) \ll \tau \ll \log N / (2\mu_3)$ implied by the
subspace-convergence statement
(Proposition~\ref{prop:da_convergence}, Appendix~\ref{app:lag_selection}),
and DA1 varies at most $2.1^\circ$ across
$\tau \in [700, 10{,}000]$\,ps
(Appendix~\ref{app:alanine_robustness}).

\begin{table}[!htbp]
\centering
\caption{Correlation of linear (left) and nonlinear (right) methods with ground-truth dihedrals on alanine dipeptide ($d = 45$). DA at $\tau^* = 10{,}000$\,ps; configuration details in Appendix~\ref{app:alanine_robustness}. DA/TICA/PCA/SM use all 750k frames; DM uses 2k subsampled (5-seed mean) for kernel scaling.}
\label{tab:alanine}
 
\setlength{\tabcolsep}{3pt}
\begin{tabular}{@{}lcccccc|cccc@{}}
\toprule
& \multicolumn{6}{c|}{\emph{Linear directions in feature space}} & \multicolumn{4}{c}{\emph{Nonlinear embedding coordinates}} \\
& DA1 & DA2 & TICA1 & TICA2 & PCA1 & PCA2 & SM1 & SM2 & DM\,$\psi_2$ & DM\,$\psi_3$ \\
\midrule
$|r(\psi)|$ & \textbf{0.672} & 0.069 & \textbf{0.753} & 0.200 & 0.563 & 0.406 & 0.044 & 0.001 & \textbf{0.615} & 0.348 \\
$|r(\phi)|$ & 0.351 & \textbf{0.873} & 0.265 & \textbf{0.523} & \textbf{0.595} & \textbf{0.598} & \textbf{0.752} & \textbf{0.572} & 0.523 & \textbf{0.659} \\
\bottomrule
\end{tabular}
\end{table}

\begin{figure}[!htbp]
\centering
\includegraphics[width=\textwidth]{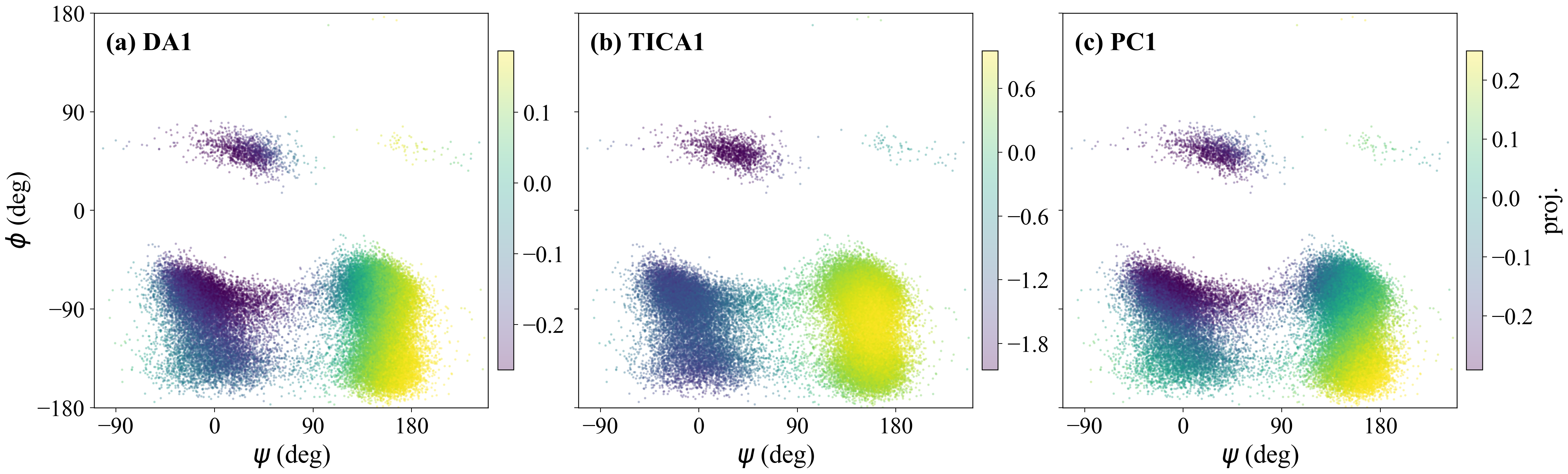}
\caption{ Ramachandran diagram colored by (a) DA1, (b) TICA1, (c) PC1. DA1 and TICA1 vary primarily along $\psi$ (the slowest dihedral); PC1 mixes $\psi$ and $\phi$. $\tau^* = 10{,}000$\,ps.}
\label{fig:alanine_da_tica}
\end{figure}

\section{Discussion}
\label{sec:discussion}
\label{sec:conclusion}

We introduced a canonical reversible diffusion uniquely determined by any density $f$ (Theorem~\ref{thm:uniqueness}) and two readouts of its autocovariance, SSA and DA, that capture barrier structure existing tools miss. Both depend on $f$ only through its score and first two moments, so the framework applies to pretrained score-based models via Tweedie's identity, with Theorem~\ref{thm:limit-spectrum} supplying an analytic upper edge $\lambda_+(\tau)$ for lag selection. On the ten SDXL prompts that clear the cutoff, SSA peaks at intermediate guidance for every prompt, never at the extremes; on alanine dipeptide, DA recovers the slow backbone dihedrals from static samples, matching TICA's ordering from MD trajectories.

\paragraph{Scope and limitations.}DA returns linear directions, so it cannot capture nonlinear mode geometry, and whether a slow mode is visible at all depends on the readout coordinates (Appendix~\ref{app:rebuttal_geometry}). Linearity is also an advantage. Each DA direction is an interpretable axis in feature space, and linearity is what makes the closed-form noise floor of Theorem~\ref{thm:limit-spectrum} possible. That theorem is exact under an isotropic-Gaussian null with an exact score and independent lag pairs. For trajectory-averaged estimators and nonlinear embeddings we calibrate by matched Monte Carlo instead (Appendix~\ref{app:rebuttal_null}). More broadly, the framework is only as good as its input score: score error biases both SSA and DA and, unlike Monte Carlo error, does not vanish with more samples. Calibrating the floor needs only fidelity on isotropic-Gaussian inputs, while interpreting the recovered directions needs fidelity on $f$ itself (Remark~\ref{rem:score_oracle}). Anisotropy alone raises SSA above its unimodal baseline of $1$ (Appendix~\ref{app:rebuttal_geometry}), so small excesses should be read against a unimodality test rather than as evidence of modes. The dominant cost is $N \times T$ score evaluations, which on large generative models bounds the reachable $(\Delta t, T, N)$ regimes. For fixed-spectral-profile comparisons these parameters affect resolution but not population ordering (Proposition~\ref{prop:ssa_monotonicity}, Appendix~\ref{app:hyperparameters}). As a measurement tool applied to existing generative models rather than a new generative capability, we see no direct deployment risk beyond that of the underlying models analyzed.

\paragraph{Future directions.} SSA and DA can probe the output distribution of any pretrained denoiser, and developing them into a standard diagnostic for fragmentation and slow directions in latent or output space is a natural next step. If a score-based model covers a molecule's equilibrium conformational distribution, DA and SSA could extract slow collective variables and metastable structure from the model alone, reducing the molecular-dynamics sampling classically required.

SSA is differentiable through the score, so it can serve as an objective that pushes a score-based model, or its conditioning inputs, toward more or less metastable outputs. This connects to inverse design for conditional distribution matching~\cite{meidler2026inverse}, which seeks an input whose induced output distribution matches a target, with SSA or another functional of the output distribution in place of the matching loss. The main obstacle is compute, and amortizing the $N \times T$ score evaluations is a problem of its own.

Theorem~\ref{thm:limit-spectrum} is a Marchenko--Pastur-type law for the lagged autocovariance, which is estimated from finite samples well beyond this paper, for example in time-series analysis and TICA-style methods. Natural extensions include anisotropic Gaussian nulls, which would separate anisotropy from barriers in DA's noise floor, the overlapping pairs of the trajectory-averaged estimator, and other reversible processes. A spiked-model analysis with Tracy--Widom edge fluctuations would give DA a detection threshold and false-positive rate, as available for PCA. Identifying individual slow modes beyond the leading subspace also remains open (Section~\ref{sec:da}).

The choice $A_0 = \sigma_f^2 I_d$ makes SSA and DA invariant to rotation, uniform scaling and translation (Proposition~\ref{prop:similarity_invariance}), treating only a change of units as outside a density's shape. One could argue the same for unequal rescaling of coordinates, which the whitening choice $A_0 = \Cov_f(X)$ would also remove, aligning DA with TICA (Appendices~\ref{app:diffusion_matrix_choice} and~\ref{app:robustness_kinetic}). Our preference for the smaller group rests on intuition and estimation cost rather than a formal criterion. Characterizing which transformations a mode-separation measure should be invariant to, and comparing both conventions empirically, remains open.

\clearpage
\bibliographystyle{unsrtnat}

\bibliography{ssa_short}

\newpage
\appendix

\section{Canonical Diffusion}
\label{app:proofs}\label{app:design_rationale}

This appendix collects the uniqueness proof of the canonical diffusion (Theorem~\ref{thm:uniqueness}), the proof of similarity invariance (Proposition~\ref{prop:similarity_invariance}), and the design rationale for the choice of diffusion matrix.

\subsection{Proof of Theorem~\ref{thm:uniqueness}}
\label{app:uniqueness_proof}

We assume $f:\R^d \to (0,\infty)$ is smooth, strictly positive, with tails such that integration by parts has no boundary terms. Consider It\^o diffusions on $\R^d$ with generator $\mathcal{L} \varphi(x) = b(x)\cdot \nabla \varphi(x) + \tfrac{1}{2}\tr(A(x) \nabla^2 \varphi(x))$, where $A(x) = \Sigma(x)\Sigma(x)^\top$.

\begin{proof}
For a general diffusion with $A(x)\equiv A_0$ constant, the Fokker--Planck
equation is
\[
\partial_t p(x,t)
= -\nabla \cdot (b(x)\,p(x,t))
+ \frac{1}{2}\,\nabla \cdot \big(A_0 \nabla p(x,t)\big).
\]
Define the probability current
$J_p(x) := b(x)\,p(x) - \frac{1}{2}A_0 \nabla p(x)$.
Stationarity with density $f$ requires $\nabla \cdot J_f(x) = 0$.
Reversibility with respect to $f$ is equivalent to the stronger condition $J_f(x)\equiv 0$ (vanishing probability current; standard characterization of reversible diffusions, see~\cite[Ch.~4]{Pavliotis2014}).
With $A(x)\equiv A_0$ and $J_f\equiv 0$, we have
$0 = b(x)f(x) - \frac{1}{2}A_0\nabla f(x)$, so
$b(x) = \frac{1}{2}A_0 \nabla \log f(x)$, which uniquely determines the
drift and hence the diffusion.
\end{proof}

\begin{remark}[Necessity of both conditions]
Both reversibility and the constant-diffusion-matrix assumption are needed for uniqueness.
Without reversibility, any smooth $f$-divergence-free vector field $u$ satisfying $\nabla \cdot (f(x)\,u(x)) = 0$ can be added to the reversible drift, producing $b(x) = \frac{1}{2}A_0 \nabla \log f(x) + u(x)$, which yields the same stationary density $f$ but different dynamics.
Without fixed $A_0$, if we allow arbitrary $A(x)$ but enforce $J_f\equiv 0$, then $b(x) = \frac{1}{2}\nabla \cdot (A(x)) + \frac{1}{2}A(x)\nabla \log f(x)$, so every choice of diffusion matrix produces a different reversible diffusion with stationary density $f$.
\end{remark}
\subsection{Proof of Proposition~\ref{prop:similarity_invariance}}
\label{app:similarity_proof}

Let $Y = cQX + b$. By change of variables, $Y$ has density
\[
f_Y(y) = c^{-d} f\!\left(c^{-1} Q^\top (y - b)\right).
\]
Applying It\^o's lemma to $Y_t = cQX_t + b$ yields
\[
dY_t = \tfrac{1}{2} c \sigma_f^2\, Q \nabla_x \log f(X_t)\,dt + c \sigma_f\, Q\,dW_t.
\]
The scalar variance of $Y$ is $\sigma_Y^2 = c^2 \sigma_f^2$, and the score transforms as $\nabla_y \log f_Y(y) = \frac{1}{c} Q \nabla_x \log f(x)$. Substituting gives
\[
dY_t = \tfrac{1}{2} \sigma_Y^2\, \nabla \log f_Y(Y_t)\,dt + \sigma_Y\,d\tilde{W}_t,
\]
where $\tilde{W}_t = Q W_t$ is again standard Brownian motion (orthogonal transformations preserve the distribution of $W$). This is exactly the canonical diffusion for $f_Y$. The change-of-variables map $U\varphi(y) := \varphi(c^{-1}Q^\top(y - b))$ is unitary $L^2(f\,dx) \to L^2(f_Y\,dy)$ (the density Jacobian $f_Y(y)\,dy = f(x)\,dx$ makes $U$ an isometry without further normalization), and intertwines the two generators: $\mathcal{L}_{f_Y} U = U \mathcal{L}_f$. Hence the spectra coincide.

\subsection{Why Reversibility (beyond uniqueness)}
\label{app:reversibility}

A diffusion with stationary distribution $f$ is not unique in general: for
any smooth vector field $u$ satisfying $\nabla \cdot (f(x)\,u(x)) = 0$, the
drift $b(x) = \frac{1}{2}A_0 \nabla \log f(x) + u(x)$ yields a different
process with the same stationary density. Such divergence-free components
$u$ introduce net probability currents that circulate mass around modes,
creating rotational dynamics whose mixing rate depends on the choice of $u$
rather than on the barrier structure of $f$ alone. Requiring reversibility
($J_f \equiv 0$) eliminates all such components, ensuring that the only
mechanism for transitioning between modes is diffusion against the barriers
of $f$. This makes metastability a property of $f$ rather than an artifact
of the dynamics.

\subsection{Choice of Diffusion Matrix: Why Not $I_d$ or $\Cov_f(X)$}
\label{app:diffusion_matrix_choice}

\textbf{Why not $A_0 = I_d$.}
Consider a bimodal density $f$ in $\R^d$ with modes at $\pm \mu$. Under the
rescaling $Y = cX$, the transformed density $f_Y$ has modes at $\pm c\mu$.
With $A_0 = I_d$, the canonical SDE for $X$ is
$dX_t = \frac{1}{2}\nabla \log f(X_t)\,dt + dW_t$, while for $Y$ it is
$dY_t = \frac{1}{2}\nabla \log f_Y(Y_t)\,dt + dW_t$. The drift scales as
$\nabla_y \log f_Y(y) = \frac{1}{c}\nabla_x \log f(x)$ while the noise is unchanged, so $Y_t/c$ follows the $A_0 = I_d$ diffusion of $f$ slowed down by a factor $c^2$: stretching the space ($c > 1$) makes the process mix slower and raises SSA, even though the distribution has the same shape. The multimodality
measure would depend on the choice of units rather than on the geometry
of~$f$.

\begin{figure}[h]
\centering
\includegraphics[width=\linewidth]{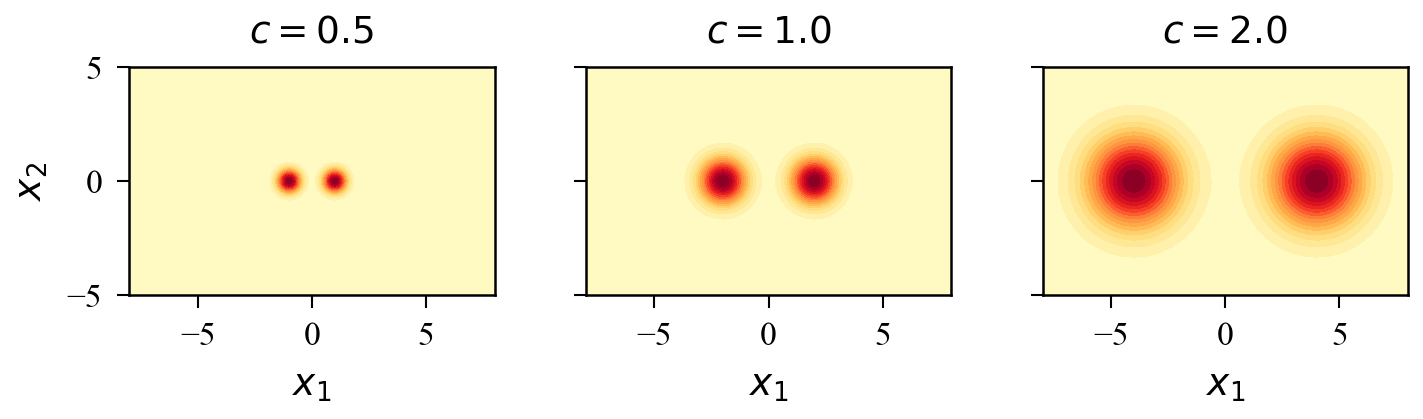}
\caption{Scale-dependence of $A_0 = I_d$. A bimodal density rescaled by
$c \in \{0.5, 1, 2\}$: the three densities have the same shape but
differ in units. Under $A_0 = I_d$, SSA varies across the three
scalings ($\hat S = 3.80, 16.44, 42.62$ for $c = 0.5, 1, 2$); under the
canonical choice $A_0 = \sigma_f^2 I_d$, SSA is constant at
$\hat S = 6.19$, as required by similarity invariance
(Proposition~\ref{prop:similarity_invariance}).}
\label{fig:a0_identity}
\end{figure}

\textbf{Why not $A_0 = \Cov_f(X)$.}
Full-covariance whitening, $A_0 = \Cov_f(X)$, is a natural alternative
that runs the diffusion in coordinates where the marginal covariance is
the identity. This corresponds to a stronger invariance principle than
ours: the resulting construction is invariant under arbitrary
invertible affine transformations of the data, whereas the scalar
variance choice $A_0 = \sigma_f^2 I_d$ is invariant only under
similarities (rotation, uniform scaling, translation;
Proposition~\ref{prop:similarity_invariance}). Both are defensible
design choices with different semantics: similarity invariance treats
SSA as a property of the density up to change of units, while affine
invariance additionally identifies all densities connected by arbitrary
linear reparameterizations. We prefer the minimal choice. It matches
the intuition that a multimodality measure should be a function of the
shape of $f$ up to a change of units, not up to arbitrary linear
distortion; it requires estimating a single scalar $\sigma_f^2$ rather
than $d(d+1)/2$ parameters; and it remains well defined in the
high-dimensional, sample-limited regime
(Appendix~\ref{app:noise_floor}), where $\Cov_f(X)$ is typically
rank-deficient or ill-conditioned and requires regularization that
interpolates back toward the scalar convention.

\section{Autocovariance and SSA}
\label{app:autocov_ssa}

This appendix proves the spectral representation of $C(\tau)$ (Lemma~\ref{lem:spectral}), motivates the design of the trace-normalized autocorrelation $\rho(\tau)$ (trace normalization, centering, squaring), states the spectral kernel monotonicity-and-saturation result, proves SSA monotonicity (Proposition~\ref{prop:ssa_monotonicity}), and gives the SSA Gaussian baseline.

\subsection{Proof of Lemma~\ref{lem:spectral}}
\label{app:spectral}

\begin{proof}

Let $P_\tau = e^{\tau\mathcal{L}_f}$ denote the Markov semigroup generated by $\mathcal{L}_f$. By the Kolmogorov backward equation, $P_\tau$ acts as conditional expectation: $\E_f[\varphi(X_\tau) \mid X_0 = x] = (P_\tau \varphi)(x)$ for any $\varphi \in L^2(f\,dx)$. By the spectral theorem for self-adjoint operators, $\{\psi_k\}_{k \geq 1}$ is a complete orthonormal basis of $L^2(f\,dx)$ and $P_\tau \psi_k = e^{-\mu_k \tau}\psi_k$.

The centered coordinate map $g(x) := x - \bar x_f$ is vector-valued; we assume $f$ has finite second moments so that each component $g_j \in L^2(f\,dx)$. Expanding each component in the eigenbasis,
\[
g_j(x) = \sum_{k \geq 1} (\alpha_k)_j\, \psi_k(x), \qquad (\alpha_k)_j = \langle g_j, \psi_k\rangle_f = \int (x_j - (\bar x_f)_j)\,\psi_k(x)\,f(x)\,dx.
\]
Stacking across coordinates gives $g(x) = \sum_{k \geq 1} \alpha_k\,\psi_k(x)$. The $k = 1$ term drops out because $\psi_1 \equiv 1$, so $(\alpha_1)_j = \E_f[X_j] - (\bar x_f)_j = 0$ by the definition of $\bar x_f$. Hence $g = \sum_{k \geq 2}\alpha_k\,\psi_k$ and, applying $P_\tau$ coordinate-wise,
\be
\E_f[g(X_\tau) \mid X_0 = x] = (P_\tau g)(x) = \sum_{k \geq 2} e^{-\mu_k \tau}\,\alpha_k\,\psi_k(x).
\label{eq:cond_exp_expansion}
\ee

\textit{(i)} Writing the expectation over the joint stationary distribution of $(X_0, X_\tau)$ and integrating out $X_\tau$ first via the tower property,
\[
C(\tau) = \E_f[Y_\tau Y_0^\top] = \int \E_f[Y_\tau \mid X_0 = x]\,(x - \bar x_f)^\top\, f(x)\,dx.
\]
Substituting~\eqref{eq:cond_exp_expansion} and pulling the sum outside the integral,
\[
C(\tau) = \sum_{k \geq 2} e^{-\mu_k \tau}\,\alpha_k \left[\int \psi_k(x)\,(x - \bar x_f)^\top f(x)\,dx\right].
\]
The bracketed quantity is a row vector whose $j$-th entry is $\int (x_j - (\bar x_f)_j)\,\psi_k(x)\,f(x)\,dx = (\alpha_k)_j$, i.e., it equals $\alpha_k^\top$. Therefore
\[
C(\tau) = \sum_{k \geq 2} e^{-\mu_k \tau}\,\alpha_k\,\alpha_k^\top.
\]

\textit{(ii)} Taking the trace of (i) and using $\tr(\alpha_k \alpha_k^\top) = \|\alpha_k\|^2$,
\[
\tr C(\tau) = \sum_{k \geq 2} e^{-\mu_k \tau}\,\|\alpha_k\|^2.
\]
At $\tau = 0$, $\tr C(0) = \sum_{k \geq 2}\|\alpha_k\|^2$, so
\[
\rho(\tau) = \frac{\tr C(\tau)}{\tr C(0)} = \sum_{k \geq 2} e^{-\mu_k \tau}\,w_k, \qquad w_k := \frac{\|\alpha_k\|^2}{\sum_{j \geq 2}\|\alpha_j\|^2}.
\]

\textit{(iii) Spectral formula for SSA (Corollary~\ref{thm:ssa_spectral}).}
From (ii), $\rho(\tau)^2 = \sum_{k, \ell \geq 2} w_k w_\ell\, e^{-(\mu_k + \mu_\ell)\tau}$.
Integrating term-by-term over $[0, \infty)$,
\begin{equation}
S(f) \;=\; \int_0^\infty \rho(\tau)^2\,d\tau
\;=\; \sum_{k,\ell \geq 2} w_k w_\ell \int_0^\infty e^{-(\mu_k+\mu_\ell)\tau}\,d\tau
\;=\; \sum_{k,\ell \geq 2} \frac{w_k w_\ell}{\mu_k + \mu_\ell},
\label{eq:ST_spectral}
\end{equation}
which is the spectral formula of Corollary~\ref{thm:ssa_spectral}.
\end{proof}

\subsection{Why Trace Normalization}
\label{app:trace_normalization}

\begin{proposition}[Orthogonal Invariance]
\label{prop:orthogonal_invariance}
Let $C_Y(\tau) = \E[Y_\tau Y_0^\top]$ and $\rho_Y(\tau) = \tr C_Y(\tau)/\tr C_Y(0)$ denote the autocovariance and trace-normalized autocorrelation of the centered process $Y_t$ from Section~\ref{sec:canonical}, and let $C_{QY}(\tau), \rho_{QY}(\tau)$ be the corresponding quantities for the rotated process $QY_t$. For any orthogonal $Q$, $\rho_{QY}(\tau) = \rho_Y(\tau)$.
\end{proposition}

\begin{proof}
$C_{QY}(\tau) = Q C_Y(\tau) Q^\top$, so $\tr(C_{QY}(\tau)) = \tr(C_Y(\tau))$ by cyclicity. Same at $\tau = 0$.
\end{proof}

Collapsing the $d \times d$ matrix $C(\tau)$ to a scalar requires a choice. A natural requirement is \emph{basis independence}: the score should not change under an orthogonal coordinate rotation. Using a single coordinate's autocorrelation fails this, since the score depends on which axis is picked. Two basis-independent choices remain: the trace $\tr C(\tau) = \sum_i \lambda_i(\tau)$ and the Frobenius norm $\|C(\tau)\|_F = \sqrt{\sum_i \lambda_i(\tau)^2}$. At first glance Frobenius is appealing because it emphasizes the directions with the largest $C(\tau)$-eigenvalues, which seems to reward barrier structure. But this conflates two distinct properties: by Lemma~\ref{lem:spectral}, $C(\tau)$ is a sum of rank-one components with amplitudes $e^{-\mu_k \tau} \|\alpha_k\|^2$, mixing how slow a mode is ($\mu_k$) with how much variance it carries ($\|\alpha_k\|^2$). The components are not orthogonal in general, so this is not the eigendecomposition of $C(\tau)$, but the Frobenius norm is dominated by whichever rank-one component has the largest current amplitude, which at finite $\tau$ can easily be a fast mode with large amplitude rather than a slow mode. The trace instead aggregates modes linearly, and persistence over $\tau$ is handled by integrating $\rho(\tau)^2$ over time: fast modes decay and drop out of the integrand, slow modes persist, and cross-terms in Equation~\eqref{eq:ST_spectral} reward pairs of simultaneously slow modes. Directional emphasis of slow modes is a separate task, handled by DA (Section~\ref{sec:da}) via eigendecomposition of $C(\tau)$ at large $\tau$, where the exponential weighting has already suppressed fast modes. Dividing by $\tr C(0)$ additionally removes absolute-scale dependence, giving $\rho(\tau) = \tr C(\tau) / \tr C(0)$, which is invariant under orthogonal coordinate changes (Proposition~\ref{prop:orthogonal_invariance}) and under isotropic rescaling of the data.

\subsection{Why Center by the Population Mean}
\label{app:centering}

The autocovariance $C(\tau) = \E_f[Y_\tau Y_0^\top]$ requires centering the
process, and there are two natural choices: the population mean
$\bar x_f = \E_f[X]$, computed once from the stationary distribution, or a
per-trajectory mean $\bar{X}^{(n)} = \frac{1}{T+1}\sum_t X_t^{(n)}$. For
ergodic processes observed over long horizons these coincide, but for the
metastable regime SSA is designed to detect, they differ in a way that
matters.

When barriers are impassable within $T$, a trajectory initialized in mode
$A$ stays in mode $A$ for the entire horizon, and its per-trajectory mean is
close to $\mu_A$ rather than $\bar x_f$. Subtracting $\bar{X}^{(n)}$ from such a
trajectory zeros out precisely the inter-mode component of interest,
producing $\hat\rho(\tau) \approx 0$ for any metastable target. Centering by
$\bar x_f$ instead preserves the signal: a trajectory trapped near $\mu_A$
contributes $(\mu_A - \bar x_f)(\mu_A - \bar x_f)^\top$ to $\hat{C}(\tau)$ at every
lag, registering as persistent autocorrelation, which is the correct
behavior for a metastability measure.

\subsection{Why Squaring the Autocorrelation}
\label{app:why_squaring}

Without squaring, the integrated autocorrelation
\[
\int_0^T \rho(\tau)\,d\tau
= \sum_{k \geq 2} w_k \,\frac{1 - e^{-\mu_k T}}{\mu_k}
\]
is (up to the $1/T$ prefactor) the integrated autocorrelation time
$\tau_{\mathrm{cor}}$, a standard mixing measure in MCMC and time-series
analysis~\cite{sokal1997,geyer1992}. It is a weighted sum of individual
mode contributions. At large $T$ each term saturates at $w_k / \mu_k$, so
the sum is dominated by the single smallest eigenvalue: a distribution
with one slow barrier and a distribution with five slow barriers score
nearly identically provided they share the same $\mu_2$.

Squaring before integrating yields cross-terms:
\[
S(T) = \sum_{k,\ell \geq 2} w_k w_\ell\,
\frac{1 - e^{-(\mu_k + \mu_\ell)T}}{\mu_k + \mu_\ell}.
\]
The cross-term between modes $k$ and $\ell$ saturates at
$w_k w_\ell / (\mu_k + \mu_\ell)$, which is large only when
\emph{both} eigenvalues are small. For example, if $\mu_2 = 0.001$
and we compare $\mu_3 = 0.002$ (two slow barriers) against
$\mu_3 = 1.0$ (one slow, one fast), the $(k,\ell) = (2,3)$
cross-term contributes $w_2 w_3 / 0.003$ in the first case versus
$w_2 w_3 / 1.001$ in the second, a ratio of over $300\times$.
The squared formulation thus rewards distributions whose metastable
structure involves multiple slow barriers rather than a single
dominant one.

\subsection{Spectral Kernel: Monotonicity and Saturation}
\label{app:spectral_kernel}

\begin{proposition}[Monotonicity and Saturation]
\label{prop:monotonicity}
For any fixed $T > 0$, $h(x) = (1 - e^{-xT})/x$ is strictly decreasing for $x > 0$, with $\lim_{x \to 0^+} h(x) = T$.
\end{proposition}
\begin{proof}
L'H\^opital gives the limit. For monotonicity, let $z = xT > 0$. Then $h'(x)$ has the sign of $e^{-z}(z+1) - 1 < 0$ since $e^z > 1+z$ for $z > 0$.
\end{proof}

\begin{corollary}[Metastable Variance Penalty]
\label{cor:metastable_variance}
Fix $T > 0$ and weights $\{w_k\}$. Partition the modes $k \geq 2$ into slow and fast, $\mathcal{K}_{\text{slow}}$ and $\mathcal{K}_{\text{fast}}$. As $\mu_k \to 0$ for all $k \in \mathcal{K}_{\text{slow}}$, with $\mu_k$ for $k \in \mathcal{K}_{\text{fast}}$ held fixed,
\[
S_T(f) \;=\; T\Big(\textstyle\sum_{k \in \mathcal{K}_{\text{slow}}} w_k\Big)^2 \;+\; O(1).
\]
\end{corollary}
\begin{proof}
Partition the double sum in~\eqref{eq:ST_spectral} into three blocks by whether each index is slow or fast:
\[
S_T(f) = \underbrace{\sum_{k,\ell \in \mathcal{K}_{\text{slow}}}\!\!\! w_k w_\ell\, h(\mu_k+\mu_\ell)}_{A}
     + \underbrace{2\!\!\sum_{\substack{k \in \mathcal{K}_{\text{slow}} \\ \ell \in \mathcal{K}_{\text{fast}}}}\!\!\! w_k w_\ell\, h(\mu_k+\mu_\ell)}_{B}
     + \underbrace{\sum_{k,\ell \in \mathcal{K}_{\text{fast}}}\!\!\! w_k w_\ell\, h(\mu_k+\mu_\ell)}_{C},
\]
where $h(x) = (1-e^{-xT})/x$. For $(k,\ell) \in \mathcal{K}_{\text{slow}}^2$, $\mu_k + \mu_\ell \to 0$ and $h(\mu_k+\mu_\ell) \to T$ by Proposition~\ref{prop:monotonicity}, so $A \to T \sum_{k,\ell \in \mathcal{K}_{\text{slow}}} w_k w_\ell = T\big(\sum_{k \in \mathcal{K}_{\text{slow}}} w_k\big)^2$. For $B$ and $C$, at least one index is fast. Since $h$ is decreasing and $h(x) \leq 1/x$, each term is bounded by $w_k w_\ell / \mu_\ell$ (for $\ell$ fast), giving $B, C = O(1)$. Hence $S_T(f) = T\big(\sum_{k \in \mathcal{K}_{\text{slow}}} w_k\big)^2 + O(1)$ in the stated limit.
\end{proof}

\subsection{Proof of Proposition~\ref{prop:ssa_monotonicity}}
\label{app:monotonicity_proof}
\begin{proof}
For the infinite-horizon score $S(f) = \sum_{i, j \geq 2} w_i w_j / (\mu_i + \mu_j)$, fix $k$. Terms not involving $k$ are independent of $\mu_k$. For off-diagonal terms ($i = k, j \neq k$ or $i \neq k, j = k$),
\[
\frac{\partial}{\partial \mu_k} \frac{w_i w_j}{\mu_i + \mu_j} \;=\; -\frac{w_i w_j}{(\mu_i + \mu_j)^2} \;\leq\; 0.
\]
For the diagonal term $i = j = k$,
\[
\frac{\partial}{\partial \mu_k} \frac{w_k^2}{2\mu_k} \;=\; -\frac{w_k^2}{2\mu_k^2} \;\leq\; 0.
\]
All terms are nonpositive, so the sum is nonincreasing in $\mu_k$, and strictly decreasing whenever $w_k > 0$ (the diagonal term then has strict negative derivative). The same argument applies to the finite-horizon score $S_T(f)$ via its spectral form $S_T(f) = \sum_{k,\ell \geq 2} w_k w_\ell\, h_T(\mu_k + \mu_\ell)$, obtained by integrating Lemma~\ref{lem:spectral} over $[0, T]$, and the strict monotonicity of the kernel $h_T(x) = (1 - e^{-xT})/x$ in $x > 0$ (Proposition~\ref{prop:monotonicity}).
\end{proof}

\subsection{Autocovariance Spectrum and SSA Baseline}
\label{app:ou_autocovariance}

For $f = \mathcal{N}(0,\Sigma)$ with eigenvalues $\lambda_1,\ldots,\lambda_d$ and scalar variance $\sigma_f^2 = \frac{1}{d}\tr(\Sigma)$, the canonical diffusion is Ornstein--Uhlenbeck:
\[
dX_t = -\tfrac{1}{2}\sigma_f^2\,\Sigma^{-1} X_t\,dt + \sigma_f\,dW_t,
\]
with stationary autocovariance $C(\tau) = e^{-\frac{\sigma_f^2}{2}\Sigma^{-1}\tau}\,\Sigma$ and trace-normalized autocorrelation
\[
\rho(t) = \frac{\sum_i \lambda_i\, e^{-\sigma_f^2 t/(2\lambda_i)}}{\sum_i \lambda_i}.
\]

For the isotropic case $\Sigma = \sigma^2 I_d$, the canonical diffusion
reduces to the standard OU process
\[
dX_t = -\tfrac{1}{2} X_t\,dt + \sigma\,dW_t.
\]
The drift $-\tfrac{1}{2} X_t$ is independent of $\sigma$ because the two
factors of $\sigma^2$ cancel: the score is $\nabla \log f(x) = -x/\sigma^2$,
and multiplying by the diffusion coefficient $\tfrac{1}{2}\sigma^2$ leaves
just $-\tfrac{1}{2} x$. The noise term $\sigma\,dW_t$ retains the $\sigma$
dependence, so rescaling the data rescales the trajectories without
changing the shape of the dynamics, as expected from similarity invariance
(Proposition~\ref{prop:similarity_invariance}).

The conditional distribution at lag $\tau$ is:
\[
X_\tau \mid X_0 \sim \mathcal{N}\!\left(e^{-\tau/2}X_0,\;\sigma^2(1 - e^{-\tau})I_d\right).
\]
The stationary autocovariance is therefore:
\[
C(\tau) = \E_f[X_\tau X_0^\top] = e^{-\tau/2}\,\E_f[X_0 X_0^\top] = e^{-\tau/2}\,\sigma^2 I_d.
\]
That is, $C(\tau)$ is a scalar multiple of the identity at every lag, with the scalar decaying as $e^{-\tau/2}$. All $d$ eigenvalues of $C(\tau)$ are equal to $e^{-\tau/2}\sigma^2$ and decay uniformly to zero. There are no preferred directions at any lag: the isotropic Gaussian has no metastable structure, and DA applied to the population autocovariance would find no distinguished direction.

Crucially, the decay rate $e^{-\tau/2}$ is independent of $\sigma^2$. This is a direct consequence of the scalar variance normalization: rescaling the data changes both the drift and diffusion strengths proportionally, leaving the autocorrelation structure invariant. The trace-normalized autocorrelation is $\rho(\tau) = e^{-\tau/2}$ and $\rho(\tau)^2 = e^{-\tau}$, giving:
\[
S_T(f) = \int_0^T e^{-\tau}\,d\tau = 1 - e^{-T} \;\to\; 1 \quad\text{as } T \to \infty.
\]
This confirms that SSA introduces no spurious slow timescales for unimodal targets. In contrast, multimodal targets have generator
eigenvalues $\mu_k$ exponentially small in the barrier height. By the
spectral formula of Corollary~\ref{thm:ssa_spectral},
$S(f) = \sum_{k,\ell \geq 2} w_k w_\ell / (\mu_k + \mu_\ell)$, so each
slow-mode pair contributes $w_k w_\ell / (\mu_k + \mu_\ell)$, which
diverges as the corresponding barriers grow (i.e., as
$\mu_k, \mu_\ell \to 0$). The contrast with the isotropic baseline
$S = 1$ is therefore unbounded: SSA grows without limit as barriers
deepen, while the unimodal value stays fixed at unity.

\subsection{Variance-Scale Stopping Rule for Lag-$\tau$ Coupling}
\label{app:rho_clt}

The stopping rule of Section~\ref{sec:ssa} truncates the lag grid at the
first $\tau$ for which $\hat\rho(\tau)^2$ falls below a distribution-free upper bound on the asymptotic null variance of $\hat\rho(\tau)$ under $H_0(\tau): X_0 \perp X_\tau$ (no lag-$\tau$ coupling). This is a one-standard-deviation-scale signal threshold rather than a formal rejection test at a specified $\alpha$ level; for isotropic high-$d$ targets it is conservative, while for targets concentrated in a single direction it corresponds to roughly one null standard deviation. The derivation below requires only finite second and fourth moments of $f$; no Gaussian or isotropy assumption is needed, and the bound is computed under independent-pair sampling (the trajectory-averaged estimator pools overlapping pairs, giving a slightly conservative threshold).

\paragraph{Setup.} Write $Y_n = X_0^{(n)} - \bar x_f$ and
$Y_n' = X_\tau^{(n)} - \bar x_f$. Under $H_0$, $Y_n$ and $Y_n'$ are
independent draws from the centered version of $f$; in particular,
$\E[Y_n] = \E[Y_n'] = 0$ and
$\E[Y_n Y_n^\top] = \E[Y_n' Y_n'^\top] = \Sigma_f := \Cov_f(X)$. The
trace-normalized autocorrelation is
\[
\hat\rho(\tau)
= \frac{\tr \hat C(\tau)}{\tr \hat C(0)}
= \frac{A_N}{B_N}, \qquad
A_N := \tfrac{1}{N}\sum_{n=1}^N Y_n^\top Y_n', \qquad
B_N := \tr \hat C(0).
\]

\paragraph{Mean and variance of the numerator.}
Each summand $Z_n := Y_n^\top Y_n' = \tr(Y_n' Y_n^\top)$ satisfies, under $H_0$,
$\E[Z_n] = \tr(\E[Y_n' Y_n^\top]) = \tr(\E[Y_n']\E[Y_n]^\top) = 0$ and
\[
\Var[Z_n]
\;=\; \E[(Y_n^\top Y_n')^2]
\;=\; \E[\tr(Y_n Y_n^\top Y_n' Y_n'^\top)]
\;=\; \tr(\E[Y_n Y_n^\top]\,\E[Y_n' Y_n'^\top])
\;=\; \tr(\Sigma_f^2),
\]
using independence of $Y_n, Y_n'$ in the third equality. Hence
$\E[A_N] = 0$ and $\Var[A_N] = \tr(\Sigma_f^2)/N$.

\paragraph{Mean and variance of the denominator.}
The denominator $B_N = \tr \hat C(0) = \tfrac{1}{N}\sum_n \|Y_n\|^2$ is a
sample mean of i.i.d.\ scalars. Each summand $W_n := \|Y_n\|^2 = \tr(Y_n Y_n^\top)$
satisfies
\[
\E[W_n] = \tr(\E[Y_n Y_n^\top]) = \tr\Sigma_f,
\]
so $\E[B_N] = \tr\Sigma_f$. Expanding in coordinates,
\[
\E[\|Y_n\|^4] = \E\!\left[\Bigl(\sum_i Y_{n,i}^2\Bigr)^2\right] = \sum_{i,j} \E[Y_{n,i}^2 Y_{n,j}^2].
\]
For any $f$ with finite fourth moments, the second-moment relations
$\E[Y_{n,i}^2 Y_{n,j}^2] = \Sigma_{f,ii}\Sigma_{f,jj} + 2\Sigma_{f,ij}^2 + \kappa_{ij}$,
where $\kappa_{ij}$ is the fourth-order cumulant correction at indices
$(i,j)$ (zero for Gaussian $f$), yield
\[
\E[\|Y_n\|^4] \;=\; \sum_{i,j} \Sigma_{f,ii}\Sigma_{f,jj} + 2\sum_{i,j}\Sigma_{f,ij}^2 + \sum_{i,j}\kappa_{ij}
\;=\; (\tr\Sigma_f)^2 + 2\tr(\Sigma_f^2) + \kappa_f,
\]
with $\kappa_f := \sum_{i,j}\kappa_{ij}$, where
$\kappa_f$ collects the fourth-order cumulant contributions of $f$
(vanishing in the Gaussian case). Assuming $f$ has finite fourth
moments, $\kappa_f$ is finite, so
$\Var[W_n] = 2\tr(\Sigma_f^2) + \kappa_f$ is a finite quantity
depending on $\Sigma_f$ and the fourth-order structure of $f$.
Averaging over $N$ i.i.d.\ copies, $\Var[B_N] = \Var[W_n]/N$, so $B_N$
concentrates at $\tr\Sigma_f$ at the $1/\sqrt{N}$ rate.

\paragraph{Asymptotic null distribution of $\hat\rho(\tau)$.}
The denominator $B_N$ has mean $\tr\Sigma_f$ and variance $O(1/N)$
(previous paragraph), so by the multivariate CLT and the delta
method~\cite{vanderVaart1998} applied to the smooth map $g(a, b) = a/b$ at
$(0, \tr\Sigma_f)$ (where $\partial_a g = 1/\tr\Sigma_f$ and $\partial_b g = 0$
because $a = 0$):
\begin{equation}
\sqrt{N}\,\hat\rho(\tau) \;\xrightarrow{d}\; \mathcal{N}\!\left(0,\; \frac{1}{d_{\mathrm{eff}}}\right),
\qquad
d_{\mathrm{eff}} := \frac{(\tr\Sigma_f)^2}{\tr(\Sigma_f^2)}.
\label{eq:rho_clt}
\end{equation}
The asymptotic variance is $\nabla g^\top \Sigma_T \nabla g$, where
$\Sigma_T$ is the covariance of $(Z_n, W_n)$. The diagonal entries are
$\Var[Z_n] = \tr(\Sigma_f^2)$ (shown above) and $\Var[W_n] < \infty$.
The off-diagonal vanishes under $H_0$: conditioning on $Y_n$,
$\E[Z_n W_n] = \E[W_n \cdot \E[Z_n \mid Y_n]] = 0$ since
$\E[Z_n \mid Y_n] = Y_n^\top \E[Y_n'] = 0$. With
$\nabla g = (1/\tr\Sigma_f, 0)$ at $(0, \tr\Sigma_f)$, only the $Z_n$
variance contributes: $\nabla g^\top \Sigma_T \nabla g = \tr(\Sigma_f^2)/(\tr\Sigma_f)^2 = 1/d_{\mathrm{eff}}$.
Equivalently, $\Var[\hat\rho(\tau)] \approx 1/(N d_{\mathrm{eff}})$ to
leading order, with denominator fluctuations contributing only at
$o(1/N)$. The quantity $d_{\mathrm{eff}} \in [1, d]$ is the participation
ratio of the eigenvalues of $\Sigma_f$: it equals $d$ when $\Sigma_f$ is
isotropic and drops toward $1$ as the spectrum concentrates on a single
direction.

\paragraph{Variance-scale threshold.}
The asymptotic distribution \eqref{eq:rho_clt} provides a null-variance scale for $\hat\rho(\tau)$. Since $d_{\mathrm{eff}} \geq 1$ for any $\Sigma_f$, the universal upper bound $\Var[\hat\rho(\tau)] \leq 1/N$ holds without structural assumptions on $f$, giving the distribution-free signal-above-variance rule
\begin{equation}
T_{\max} \;=\; \min\!\left\{\tau \,:\, \hat\rho(\tau)^2 \;\le\; \frac{1}{N}\right\}.
\label{eq:rho_clt_threshold}
\end{equation}
This is an operational threshold (squared autocorrelation exceeds one upper-bound null-variance scale), not a formal $\alpha$-level test; for an $\alpha$-level rule one would replace $1/N$ by $z_{1-\alpha/2}^2/(N\,d_{\mathrm{eff}})$ using the plug-in estimate of $d_{\mathrm{eff}}$. Under $H_0$, retaining a lag $\tau$ with $\hat\rho(\tau)^2 > 1/N$ corresponds to a z-score of at least $\sqrt{d_{\mathrm{eff}}}$, conservative on isotropic targets and at the one-standard-deviation scale at $d_{\mathrm{eff}} = 1$. A sharper data-adaptive threshold replaces $1/N$ with the
plug-in estimate $\tr(\hat C(0)^2)/(N\,\tr(\hat C(0))^2)$ from
\eqref{eq:rho_clt}, computable from the same $\hat C(0)$ already needed
for normalization.

\paragraph{Scope and relation to the DA noise floor.}
The CLT \eqref{eq:rho_clt} holds for any fixed $d$ as $N \to \infty$
and assumes only finite second moments of $f$; it applies specifically to
the trace functional $\tr \hat C(\tau)$ that SSA integrates. This is a
distinct regime from the DA noise floor of
Theorem~\ref{thm:limit-spectrum}, which is derived in the proportional
limit $d/N \to \gamma \in (0, \infty)$ and characterizes the top eigenvalue of
$\hat C(\tau)$ rather than its trace. The two thresholds are not
interchangeable: \eqref{eq:rho_clt_threshold} averages over $d$
directions and is the right object for SSA's scalar readout, while
$\lambda_+(\tau)$ takes the maximum over directions and is the right
object for DA's eigendecomposition.

\section{DA and Limit Spectrum}
\label{app:da_limit_spectrum}

We prove Theorem~\ref{thm:limit-spectrum}, the limit spectrum of $\hat C^{\mathrm{sym}}(\tau)$ under the iso-Gaussian null, via Wishart-difference decomposition and free-probability calculation. We also prove Proposition~\ref{prop:da_convergence} on convergence of the leading DA direction, establish the bimodal heuristic, recall the Marchenko--Pastur limit for the sample covariance, and detail the lag-selection criterion for finite-$\tau$ DA.

\subsection{Proof of Proposition~\ref{prop:da_convergence}}
\label{app:da_convergence_proof}
\begin{proof}
Starting from Lemma~\ref{lem:spectral}, factor out the dominant exponential:
\[
C(\tau) = e^{-\mu_2 \tau}\bigg[\underbrace{\alpha_2 \alpha_2^\top}_{M} + \underbrace{\sum_{k \geq 3} e^{-(\mu_k - \mu_2)\tau}\, \alpha_k \alpha_k^\top}_{R(\tau)}\bigg].
\]
Since $\mu_k \geq \mu_3$ for all $k \geq 3$, the remainder satisfies $\|R(\tau)\| \leq e^{-(\mu_3 - \mu_2)\tau} \sum_{k \geq 3} \|\alpha_k\|^2$. The prefactor $e^{-\mu_2 \tau}$ does not affect eigenvectors, so $v_1(\tau)$ is the leading eigenvector of $M + R(\tau)$ where $M = \alpha_2 \alpha_2^\top$ is rank-one with eigenvalue $\|\alpha_2\|^2$ and eigenvector $\alpha_2/\|\alpha_2\|$. For sufficiently large $\tau$, $\|R(\tau)\| < \|\alpha_2\|^2$. Since $R(\tau) \succeq 0$ and $M$ has rank one, Weyl's inequality gives $\lambda_1(M + R(\tau)) \ge \|\alpha_2\|^2$ and $\lambda_2(M + R(\tau)) \le \|R(\tau)\|$, so $v_1(\tau)$ is unique and the eigenvalue $\|\alpha_2\|^2$ of $M$ is separated from the rest of the spectrum of $M + R(\tau)$ by at least $\|\alpha_2\|^2 - \|R(\tau)\|$. By the Davis--Kahan $\sin\Theta$ theorem, the unit-norm eigenvector $v_1(\tau)$ satisfies
\[
1 - \frac{\langle v_1(\tau),\, \alpha_2\rangle^2}{\|\alpha_2\|^2} \;\leq\; \left(\frac{\|R(\tau)\|}{\|\alpha_2\|^2 - \|R(\tau)\|}\right)^2.
\]
As $\tau \to \infty$, $\|R(\tau)\| \to 0$ and the denominator approaches the constant $\|\alpha_2\|^2$, so the bound is $O(e^{-2(\mu_3 - \mu_2)\tau})$.

\paragraph{Subspace convergence for $m \geq 2$.}
The same factorization extends to subspace level. Fix $m \geq 2$ and assume $\mu_{m+1} < \mu_{m+2}$ and that $\alpha_2, \ldots, \alpha_{m+1}$ are linearly independent. Linear independence at fixed $m$ implies a positive lower bound on the $m$-th eigenvalue of the slow block,
$\lambda_m\bigl(A_m(\tau)\bigr) \geq c_m\,e^{-(\mu_{m+1} - \mu_2)\tau}$ for some $c_m > 0$, where
$A_m(\tau) = \sum_{k=2}^{m+1} e^{-(\mu_k - \mu_2)\tau}\,\alpha_k\alpha_k^\top$.
Splitting the bracket into this slow block and a tail
$R_m(\tau) = \sum_{k \geq m+2} e^{-(\mu_k - \mu_2)\tau}\,\alpha_k\alpha_k^\top$
with $\|R_m(\tau)\| \leq e^{-(\mu_{m+2} - \mu_2)\tau}\sum_{k \geq m+2}\|\alpha_k\|^2$.
Applying the Davis--Kahan $\sin\Theta$ theorem to the gap between the top-$m$
block of $A_m(\tau)$ and the tail gives
\[
\|V_m(\tau)V_m(\tau)^\top - \Pi_m\|_F^2 \;=\; O\bigl(e^{-2(\mu_{m+2} - \mu_{m+1})\tau}\bigr),
\]
where $V_m(\tau) \in \R^{d \times m}$ collects the top $m$ unit-norm
eigenvectors of $C(\tau)$ and $\Pi_m$ is the orthogonal projector onto
$\mathrm{span}(\alpha_2, \ldots, \alpha_{m+1})$. The bound controls the
\emph{subspace} only; identifying the individual eigenvectors
$v_2(\tau), \ldots, v_m(\tau)$ with specific $\alpha_k$ requires
additional gap and non-collinearity conditions on $\{\alpha_k\}$ that
do not hold in general (the $\{\alpha_k\}$ are projection coefficients
of coordinate functions onto the eigenfunctions $\{\psi_k\}$ and are
not orthogonal). The bimodal case (Appendix~\ref{app:bimodal_heuristic})
is a special case where only $\alpha_2$ is non-negligible, so DA1
identifies cleanly. We validate per-direction recovery in the
multi-direction regime empirically (Section~\ref{sec:exp_molecular},
Appendix~\ref{app:validation}).
\end{proof}

\subsection{Bimodal Heuristic: $\alpha_2 \propto \mu_A - \mu_B$}
\label{app:bimodal_heuristic}

We give the derivation behind the bimodal example in
Section~\ref{sec:da}: for a strongly metastable bimodal density, the
linear summary $\alpha_2$ of the slowest eigenfunction $\psi_2$ is
proportional to the direction connecting mode centers.

\paragraph{Setup.} Let $f = \pi_A f_A + \pi_B f_B$ where $f_A, f_B$ are
localized around $\mu_A, \mu_B$ with negligible overlap, and $\pi_A, \pi_B$
are the basin masses. Write $A$ and $B$ for the two basins (the regions of
support of $f_A$ and $f_B$, respectively).

\paragraph{Piecewise-constant approximation of $\psi_2$.} In the strongly
metastable limit, the slowest eigenfunction of the generator $\mathcal{L}_f$
is approximately piecewise constant on the two basins~\cite{Pavliotis2014}:
\[
\psi_2(x) \approx
\begin{cases} c_A & x \in A, \\ c_B & x \in B, \end{cases}
\]
with a thin transition region across the barrier. The intuition is that
$\psi_2$ must be the slowest-relaxing observable under the dynamics; any
function varying within a basin relaxes quickly (intra-basin mixing is
fast), so the slowest observable is the one that distinguishes the two
basins and is otherwise constant.

\paragraph{The two levels are fixed by orthogonality.} Since
$\psi_1 \equiv 1$ and $\langle \psi_1, \psi_2 \rangle_f = 0$,
\[
\int \psi_2(x)\,f(x)\,dx = \pi_A c_A + \pi_B c_B = 0,
\]
so $c_B = -(\pi_A/\pi_B)\,c_A$. For equal weights ($\pi_A = \pi_B = 1/2$),
this reduces to $c_B = -c_A$; write $c := c_A$.

\paragraph{Computation of $\alpha_2$.} By definition,
$(\alpha_2)_j = \int (x_j - (\bar x_f)_j)\,\psi_2(x)\,f(x)\,dx$. Splitting
the integral over the two basins and using the piecewise-constant
approximation,
\[
\alpha_2 \approx c \int_A (x - \bar x_f)\,f(x)\,dx - c \int_B (x - \bar x_f)\,f(x)\,dx
= c\bigl[\pi_A(\mu_A - \bar x_f) - \pi_B(\mu_B - \bar x_f)\bigr],
\]
where $\mu_A := \E_{f_A}[X]$ and analogously for $\mu_B$. For equal weights,
$\bar x_f = (\mu_A + \mu_B)/2$, so $\mu_A - \bar x_f = \tfrac{1}{2}(\mu_A - \mu_B)$
and $\mu_B - \bar x_f = -\tfrac{1}{2}(\mu_A - \mu_B)$, giving
\[
\alpha_2 \approx \tfrac{c}{2}(\mu_A - \mu_B).
\]
That is, $\alpha_2$ is aligned with the vector connecting the two mode
centers. For unequal weights, the same computation gives
$\alpha_2 \propto \mu_A - \mu_B$ with a prefactor depending on $\pi_A, \pi_B$.

\paragraph{Consequence for DA.} Combined with
Proposition~\ref{prop:da_convergence}, this shows that for a bimodal target
the leading DA direction at large lag converges to the direction connecting
mode centers, independently of how variance is distributed within each mode.
The parallel ellipses experiment
(Section~\ref{sec:exp_parallel_ellipses}, Figure~\ref{fig:structure}c)
exhibits this directly: DA1 points between the two modes while PC1 points
along the within-mode spread.

\subsection{Marchenko--Pastur Law for the Sample Covariance}
\label{app:mp_law}

For reference, the Marchenko--Pastur (MP) law~\cite{marchenko1967} governs the limit eigenvalue distribution of the sample covariance $\hat{\Sigma} = \frac{1}{N}\sum_n x_n x_n^\top$ formed from $N$ i.i.d.\ samples $x_n \in \R^d$ with $\Cov(x) = \sigma^2 I_d$ in the proportional limit $d,N \to \infty$, $d/N \to \gamma > 0$, with density
\[
f_{\mathrm{MP}}(\lambda) = \frac{1}{2\pi\gamma\sigma^2}\,\frac{\sqrt{(\lambda_+ - \lambda)(\lambda - \lambda_-)}}{\lambda}, \quad \lambda_{\pm} = \sigma^2(1 \pm \sqrt{\gamma})^2,
\]
plus a $(1 - 1/\gamma)\delta_0$ atom for $\gamma > 1$. At $\tau = 0$, $\hat C(0)$ is a sample covariance and its eigenvalues follow MP.

\subsection{Proof of Theorem~\ref{thm:limit-spectrum}: exact null spectrum of $\hat C^{\mathrm{sym}}(\tau)$}
\label{app:limit-spectrum}\label{app:noise_floor}

This appendix proves Theorem~\ref{thm:limit-spectrum} in three steps:
(i) a finite-$N$ decomposition of $\hat C^{\mathrm{sym}}(\tau)$ as a
weighted difference of two independent Wishart matrices
(Section~\ref{app:wishart-decomp});
(ii) a short primer on the free-probability tools we use
(Section~\ref{app:fp-primer});
(iii) the derivation of the Stieltjes cubic and edge quartic from
the free-probability calculation
(Section~\ref{app:polynomial-derivation}).
Boundary cases ($\tau = 0, \tau = \infty$) and the small-$\gamma$
semicircle approximation are collected in
Section~\ref{app:consequences}; scope and empirical verification in
Section~\ref{app:scope}.

The proof analyses the simple-pair estimator
$(1/N)\sum_n X_0^{(n)}(X_\tau^{(n)})^\top$, with $\gamma = d/N$ in the
proportional limit. The trajectory-averaged form~\eqref{eq:c_hat_def}
used in Algorithm~\ref{alg:ssa_da} pools $N(T+1-\tau)$ overlapping
pairs that are dependent across $t$; we use the analytic edge as a
reference floor for that estimator and rely on matched Monte-Carlo
calibration when the pipeline departs from the exact simple-pair null.

\paragraph{Related random-matrix results.}
Theorem~\ref{thm:limit-spectrum} sits near two recent random-matrix
threads. Bhattacharjee, Bose, and Dey~\cite{bhattacharjee2023} study
sample cross-covariance matrices $C = n^{-1}XY^\ast$ with correlated
pairs and derive joint asymptotic results implying limiting laws for
symmetric polynomials such as $C + C^\ast$. Kumar and
Charan~\cite{kumar2020differencewishart} analyze weighted differences of
Wishart matrices directly. Our result is more specialized, restricted to the OU-lag Gaussian
null relevant to DA, but in exchange yields closed-form expressions
not given by the more general works above: the lagged Gaussian pair
structure admits an exact finite-$N$ orthogonal rotation to
$c_+W_+ - c_-W_-$, from which we derive the explicit Stieltjes cubic,
edge quartic, atom mass, and the operational noise floor
$\lambda_+(\tau)$ used throughout the paper.

\subsubsection{Step 1: exact Wishart-difference decomposition}
\label{app:wishart-decomp}

The first step is a finite-$N$ rewriting of $\hat C^{\mathrm{sym}}(\tau))$ that converts it
into a sum of functions of two \emph{independent} Gaussian ensembles.
Contrast with the naive decomposition
$\hat C(\tau) = e^{-\tau/2}\hat C(0) + \sigma\sqrt{1-e^{-\tau}}\,\hat R$
with $\hat R = N^{-1}\sum_n \xi^{(n)}(X_0^{(n)})^\top$: the two terms on
the right share $X_0$ and are therefore \emph{not} independent, which
is what prevents a clean application of free probability. Symmetrizing
and rotating in a well-chosen $2{\times}2$ basis produces an
independent-ensemble representation.

\begin{lemma}[Exact decomposition of $\hat C^{\mathrm{sym}}(\tau)$]
\label{lem:decomp}
With $c_\pm(\tau) = (\sigma^2/2)(1 \pm e^{-\tau/2})$, there exist
$d\times N$ random matrices $U, V$ with i.i.d.\ $\mathcal N(0,1)$
entries, independent of each other, such that
\begin{equation}
\hat C^{\mathrm{sym}}(\tau) \;=\; c_+(\tau)\,\frac{UU^\top}{N} \;-\; c_-(\tau)\,\frac{VV^\top}{N}.
\label{eq:wishart-decomp}
\end{equation}
Equivalently, in distribution, $\hat C^{\mathrm{sym}}(\tau) \stackrel{d}{=} c_+(\tau)\,W_+ - c_-(\tau)\,W_-$,
with $W_+,W_-$ i.i.d.\ $d\times d$ standard Wishart matrices at aspect
ratio $\gamma = d/N$. Moreover, $(U, V)$ is obtained from $(\tilde X, Z)$
(standardized samples and noise, defined in the proof) by a specific
orthogonal transformation applied column-wise, with angle a function of
$\tau$ alone.
\end{lemma}

\begin{proof}
Standardize $X_0^{(n)} = \sigma\tilde X^{(n)}$ with $\tilde X^{(n)} \sim \mathcal N(0, I_d)$.
Let $\xi^{(n)} \sim \mathcal N(0, I_d)$ be the OU noise column driving
$X_\tau^{(n)} = \alpha\sigma \tilde X^{(n)} + \beta\, \xi^{(n)}$, where
$\alpha = e^{-\tau/2}$ and $\beta = \sigma\sqrt{1 - e^{-\tau}}$; we write
$\tilde X, Z$ for the $d \times N$ matrices whose $n$-th columns are
$\tilde X^{(n)}, \xi^{(n)}$ (respectively). Then
\[
2N\,\hat C^{\mathrm{sym}}(\tau) \;=\; 2\alpha\sigma^2\,\tilde X\tilde X^\top \;+\; \beta\sigma\bigl(Z\tilde X^\top + \tilde X Z^\top\bigr).
\]

\emph{Intermediate change of basis.}
Set $A = (Z + \tilde X)/\sqrt 2$ and $B = (Z - \tilde X)/\sqrt 2$. The
$2\times 2$ transformation $(\tilde X, Z) \mapsto (A, B)$ is orthogonal,
so $(A, B)$ has the same joint i.i.d.\ $\mathcal N(0, I_d)$ distribution
as $(\tilde X, Z)$, with columns of $A$ and $B$ independent of one
another. Then $\tilde X = (A - B)/\sqrt 2$ and $Z = (A + B)/\sqrt 2$, so direct expansion gives
$\tilde X\tilde X^\top = \tfrac12(AA^\top + BB^\top - AB^\top - BA^\top)$
and $Z\tilde X^\top + \tilde X Z^\top = AA^\top - BB^\top$, so
\[
2N\,\hat C^{\mathrm{sym}}(\tau) \;=\; \underbrace{(\alpha\sigma^2 + \beta\sigma)}_{P}\,AA^\top \;+\; \underbrace{(\alpha\sigma^2 - \beta\sigma)}_{Q}\,BB^\top \;+\; \underbrace{(-\alpha\sigma^2)}_{R}\bigl(AB^\top + BA^\top\bigr).
\]

\emph{Diagonalization.}
The $2{\times}2$ coefficient matrix
$K = \bigl(\begin{smallmatrix}P&R\\R&Q\end{smallmatrix}\bigr)$ has
eigenvalues $\tfrac{P+Q}{2} \pm \sqrt{\bigl(\tfrac{P-Q}{2}\bigr)^2 + R^2} = \alpha\sigma^2 \pm \sigma\sqrt{\alpha^2\sigma^2 + \beta^2} = \sigma^2(\alpha \pm 1)$
(using $\alpha^2\sigma^2 + \beta^2 = \sigma^2$). Let $O(\tau)$ be the orthogonal
$2\times 2$ matrix that diagonalizes $K$; its angle $\theta(\tau)$
satisfies $\tan 2\theta = 2R/(P-Q) = -\alpha\sigma/\beta = -e^{-\tau/2}/\sqrt{1 - e^{-\tau}}$,
a function of $\tau$ alone. Applying $O(\tau)$ column-wise to the pair
$(A_n, B_n)$ produces new $d \times N$ random matrices $(U, V)$. Since each column-pair $(A_n, B_n)$ is jointly i.i.d.\ $\mathcal N(0, I_{2d})$ and orthogonal transformations preserve the i.i.d.\ Gaussian distribution, $(U, V)$ are still jointly i.i.d.\ $\mathcal N(0, I_d)$ with columns of $U$ and $V$ independent of one another, and
\[
2N\,\hat C^{\mathrm{sym}}(\tau) \;=\; \sigma^2(1 + \alpha)\,UU^\top \;-\; \sigma^2(1 - \alpha)\,VV^\top.
\]
Dividing by $2N$ gives \eqref{eq:wishart-decomp}. Composing the two
rotations, $(U, V)$ is obtained from $(\tilde X, Z)$ by a single
orthogonal transformation with angle $\pi/4 + \theta(\tau)$, a
deterministic function of $\tau$.
\end{proof}

\subsubsection{Step 2: primer on free additive convolution and the
$R$-transform}
\label{app:fp-primer}

A short summary of the three tools used below; readers familiar with
Voiculescu--Dykema--Nica~\cite{voiculescu1992} or
Nica--Speicher~\cite{nica2006} may skip.

\paragraph{Stieltjes transform.}
For a compactly supported probability measure $\mu$ on $\R$,
$G_\mu(z) = \int (z-x)^{-1}\,d\mu(x)$ for $z \in \mathbb{C}\setminus\mathrm{supp}(\mu)$.
Density: $\rho_\mu(\lambda) = -\pi^{-1}\lim_{\epsilon\downarrow 0}\mathrm{Im}\,G_\mu(\lambda+i\epsilon)$.
Large-$|z|$ expansion: $G_\mu(z) = \sum_{k\ge 0} m_k/z^{k+1}$.

\paragraph{Free additive convolution.}
Let $A, B$ be $d\times d$ self-adjoint random matrices that are
independent, with at least one of them orthogonally invariant in
distribution ($A \stackrel{d}{=} O A O^\top$ for every orthogonal
$O$). For a self-adjoint $d \times d$ matrix $X$, write
$\mathrm{ESD}(X) := \frac{1}{d}\sum_{i=1}^d \delta_{\lambda_i(X)}$
for its empirical spectral distribution and
$\mu_X := \lim_{d \to \infty}\mathrm{ESD}(X)$ for the limit (when it
exists). Voiculescu's asymptotic-freeness
theorem~\cite{voiculescu1992,nica2006} guarantees that
$\mathrm{ESD}(A+B)$ converges weakly almost surely to a measure on
$\R$, the \emph{free additive convolution} of $\mu_A$ and $\mu_B$:
\[
\mu_A \boxplus \mu_B \;:=\; \lim_{d\to\infty}\,\mathrm{ESD}(A+B).
\]
It is the non-commutative analogue of classical convolution (mean and
variance add under both, but higher moments differ). The limit is not
computable in closed form from this definition; the next paragraph
introduces the $R$-transform, which linearizes it.

\paragraph{The $R$-transform.}
Let $K_\mu(z) = G_\mu^{-1}(z)$ be the compositional inverse of $G_\mu$
near $z=0$. The $R$-transform is $R_\mu(z) := K_\mu(z) - 1/z$. It is
\emph{additive} under free convolution, $R_{A\boxplus B}(z) = R_A(z) + R_B(z)$,
and transforms simply under scaling and reflection:
$R_{cX}(z) = c\,R_X(cz)$, $R_{-X}(z) = -R_X(-z)$. The Marchenko--Pastur
law has $R_{\mathrm{MP}(\gamma)}(z) = 1/(1 - \gamma z)$~\cite{marchenko1967, bai2010, nica2006}.

\paragraph{Recipe: from $R$-transforms back to a density.}
The three identities above give a procedure for deriving the spectral
distribution $\mu_A \boxplus \mu_B$ when $R_A, R_B$ are known:
\begin{enumerate}[label=(\roman*),leftmargin=2.5em,nosep]
\item By additivity, $R_\mu = R_A + R_B$.
\item Set $K_\mu(z) := R_\mu(z) + 1/z$. Then $G_\mu$, the Stieltjes
transform of $\mu_A \boxplus \mu_B$, is the compositional inverse,
defined implicitly by $K_\mu(G_\mu(z)) = z$. This is a polynomial
equation in $G$; in our setting it is the \emph{Stieltjes cubic}
(Section~\ref{app:polynomial-derivation}, Eq.~\eqref{eq:stieltjes-cubic}),
so $G_\mu(z)$ is an algebraic function of $z$ given by Cardano's
formula on the upper-half-plane branch.
\item Recover $\mu_A \boxplus \mu_B$ from $G_\mu$ by Stieltjes
inversion (paragraph above): its density is
$\rho(\lambda) = -\pi^{-1}\lim_{\epsilon\downarrow 0}\mathrm{Im}\,G_\mu(\lambda+i\epsilon)$.
The support boundary is the locus where $G$ acquires its imaginary
part, determined by the discriminant of the polynomial in (ii);
in our setting this is the \emph{edge quartic}
(Section~\ref{app:polynomial-derivation}, Eq.~\eqref{eq:edge-quartic}).
\end{enumerate}
The remainder of this section carries out item~(i) for our case (Asymptotic-freeness paragraph below, giving $R_{\mu_\tau^\gamma}$ explicitly in Eq.~\eqref{eq:R-explicit-appendix}) plus the atom contribution via the free-conv atom rule. Items~(ii) and~(iii) are executed in Step 3 (Section~\ref{app:polynomial-derivation}), which derives the Stieltjes cubic from $K(G)=z$ and the edge quartic from its discriminant.

\paragraph{Asymptotic freeness of $W_+$ and $-W_-$.}
$W_+, W_-$ in Lemma~\ref{lem:decomp} are independent standard Wisharts with i.i.d.\ $\mathcal{N}(0, 1)$ entries, each orthogonally invariant in distribution. By Voiculescu's asymptotic-freeness theorem for Gaussian ensembles~\cite{voiculescu1992,nica2006}, they become
asymptotically free as $d\to\infty$ with $d/N \to \gamma \in (0, \infty)$.
Consequently $\mathrm{ESD}(\hat C^{\mathrm{sym}}(\tau))$ converges weakly almost surely to
$\mu_\tau^\gamma = (c_+\,\mathrm{MP}(\gamma)) \boxplus (-c_-\,\mathrm{MP}(\gamma))$,
and
\begin{equation}
R_{\mu_\tau^\gamma}(z) \;=\; \frac{c_+}{1 - c_+\gamma z} \;-\; \frac{c_-}{1 + c_-\gamma z}.
\label{eq:R-explicit-appendix}
\end{equation}

\paragraph{Atom at zero for $\gamma > 2$.}
The standard Marchenko--Pastur measure $\mathrm{MP}(\gamma)$ has a
point mass $\max(0, 1-1/\gamma)$ at the origin, arising from the
rank deficiency of the Wishart matrix when $d > N$~\cite{marchenko1967, bai2010}.
Free additive convolution combines atoms via the rule
\cite{nica2006}: if $\mu$ has atom of mass $p$ at $a$ and $\nu$ has
atom of mass $q$ at $b$, then $\mu \boxplus \nu$ has atom of mass
$\max(0, p + q - 1)$ at $a + b$. Applied to
$(c_+\,\mathrm{MP}(\gamma)) \boxplus (-c_-\,\mathrm{MP}(\gamma))$ —
both atoms at $0$ with mass $1 - 1/\gamma$ for $\gamma > 1$ — the
result has an atom at $0$ of mass
$\max(0, 2(1 - 1/\gamma) - 1) = \max(0, 1 - 2/\gamma)$. The continuous
part has the complementary mass $\min(1, 2/\gamma)$ and remains
characterized by the Stieltjes cubic and edge quartic derived below;
the $R$-transform identity~\eqref{eq:R-explicit-appendix} continues to
hold for all $\gamma > 0$. The formal identity $R_{\mathrm{MP}(\gamma)}(z) = 1/(1-\gamma z)$ is valid as a generating function regardless of $\gamma$: for $\gamma \le 1$ it inverts to the absolutely continuous MP density on $[(1-\sqrt\gamma)^2, (1+\sqrt\gamma)^2]$; for $\gamma > 1$ the same formal $R$ inverts via the Stieltjes transform $G(z) = m/z + h(z)$ to a measure with an atom of mass $m = 1 - 1/\gamma$ at $0$, with the atom emerging because $G_\mu$ acquires a residue at $z = 0$. This recovers the decomposition
\eqref{eq:mu_tau_decomp} of Theorem~\ref{thm:limit-spectrum}. The
rank-deficiency argument has a direct finite-$N$ counterpart:
$\hat C^{\mathrm{sym}}(\tau)$ has rank at most $2N$ (it is a sum of $N$ symmetric
outer products on $\mathrm{span}(X_t, X_{t+\tau})$), so when $d > 2N$
exactly $d - 2N$ eigenvalues are deterministically zero, and the
empirical atom mass $(d - 2N)/d \to 1 - 2/\gamma$ matches the free-conv
prediction. The same atom mass can be obtained directly from the cubic Stieltjes equation~\eqref{eq:stieltjes-cubic} by a dominant-balance argument at $z = 0$ (not shown).

\subsubsection{Step 3: polynomial equations via the $R$-transform}
\label{app:polynomial-derivation}

\paragraph{Derivation of the Stieltjes cubic.}
The compositional-inverse identity $K(G(\lambda)) = \lambda$ with
$K = R + 1/z$ and \eqref{eq:R-explicit-appendix} yields, writing $G = G(\lambda)$,
\[
\frac{c_+}{1 - c_+\gamma G} \;-\; \frac{c_-}{1 + c_-\gamma G} \;+\; \frac{1}{G} \;=\; \lambda.
\]
Multiplying through by $G(1 - c_+\gamma G)(1 + c_-\gamma G)$ and
collecting terms polynomial in $G$ gives the cubic
\eqref{eq:stieltjes-cubic} of Theorem~\ref{thm:limit-spectrum}.
The leading $G^3$ term comes from the product of the two denominators
times $\lambda G$; the $G^2$ coefficient collects cross terms and
simplifies to $\gamma[c_+c_-(2-\gamma) + \lambda(c_+ - c_-)]$; the $G^1$
coefficient reduces to $(c_+ - c_-)(1-\gamma) - \lambda$; and the
constant term is $1$. Density on the interior of the support follows
from Sokhotski--Plemelj~\cite{bai2010}: $\mathrm{Im}\,G(\lambda + i0^+) = -\pi\rho(\lambda)$.

\paragraph{Derivation of the edge quartic.}
Edges are points at which the complex-conjugate pair of roots of the
cubic \eqref{eq:stieltjes-cubic} coalesces onto the real line as a
double root. A cubic of the form $a G^3 + b G^2 + c G + 1$ (constant term $1$, leading coefficient $a$) has a double root iff its discriminant vanishes:
$18abc - 4b^3 + b^2 c^2 - 4ac^3 - 27a^2 = 0$, which with $(a,b,c)$
from \eqref{eq:stieltjes-cubic} is precisely the quartic
$\Delta_\tau(\lambda) = 0$ of Theorem~\ref{thm:limit-spectrum}(i).
Expanding as a polynomial in $\lambda$: each of $a, b, c$ is at most
linear in $\lambda$, so the five summands have $\lambda$-degrees
$(3, 3, 4, 4, 2)$. Writing $[\lambda^k] f$ for the coefficient of
$\lambda^k$ in a polynomial $f$ (in $\lambda$), the $\lambda^4$
contribution is
\[
[\lambda^4]\,\Delta_\tau \;=\; [\lambda^2]\,b^2 \cdot [\lambda^2]\,c^2 \;+\; (-4)\,[\lambda^1]\,a \cdot [\lambda^3]\,c^3 \;=\; \gamma^2(c_+\!-\!c_-)^2 + 4 c_+c_-\gamma^2 \;=\; \gamma^2(c_++c_-)^2 \;=\; \gamma^2\sigma^4,
\]
using $(c_+-c_-)^2 + 4 c_+c_- = (c_++c_-)^2$ and $c_++c_- = \sigma^2$.
Depending on $(\gamma, \tau)$, the quartic $\Delta_\tau$ may have two or four real roots. The real roots are precisely the candidate boundary points of the continuous support: in the one-cut regime only the smallest and largest real roots occur as support edges; in the supercritical regime $\gamma > 2$, four real roots occur and the support splits into two intervals around the atom at zero (Section~\ref{app:consequences}). The operational threshold $\lambda_+(\tau)$ is the largest real support edge.

\subsubsection{Boundary cases and special regimes}
\label{app:consequences}

\begin{remark}[Two regimes for $\tau > 0$]
\label{rem:two_regimes}
For fixed $\tau > 0$, the limit measure $\mu_\tau^\gamma$
in~\eqref{eq:mu_tau_decomp} takes one of two qualitatively different
forms:
\begin{itemize}[leftmargin=*,nosep]
\item \emph{Sub-critical regime, $\gamma \in (0, 2]$:}
$\mu_\tau^\gamma = \rho_\tau^\gamma$ is a continuous probability
measure on $[\lambda_-(\tau), \lambda_+(\tau)]$, no atom. Even when
each MP factor in the free-additive decomposition has its own atom
(for $\gamma > 1$), the rule $\max(0, p+q-1)$ with
$p = q = 1 - 1/\gamma$ wipes both out exactly when $\gamma \le 2$.
\item \emph{Super-critical regime, $\gamma > 2$:}
$\mu_\tau^\gamma = (1 - 2/\gamma)\,\delta_0 + (2/\gamma)\,\rho_\tau^\gamma$,
the continuous part $\rho_\tau^\gamma$ on the same support, given by
the same cubic.
\end{itemize}
The cubic equation \eqref{eq:stieltjes-cubic} and edge quartic
are identical across both regimes; only the mass distribution between
atom and continuous part changes. So crossing $\gamma = 1$ is a
non-event for the $\tau > 0$ spectrum; the natural atom threshold is
$\gamma = 2$.

\paragraph{Geometric origin: rank-doubling at $\tau = 0^+$.}
At $\tau = 0$ each trajectory contributes one column $X_t$ to the
column space of $\hat C(0)$, giving rank $\le N$ and forced zero
eigenvalues for $d > N$. At any $\tau > 0$ the symmetrization picks
up both $X_t$ and $X_{t+\tau}$ as distinct vectors, doubling the rank
bound to $2N$ and unsticking the previously-atomic mass into a sharp
continuous peak near zero that smooths out as $\tau$ grows. For
$\gamma \in (1, 2)$ the discontinuity is most striking: a substantial
atom of mass $1 - 1/\gamma$ at $\tau = 0$ vanishes the moment $\tau$
becomes positive, even infinitesimally; the continuous bulk absorbs
the atomic mass via free-conv cancellation.
\end{remark}

\paragraph{$\tau = 0$: classical Marchenko--Pastur.} $c_-=0$, so
$a\equiv 0$ and the Stieltjes cubic collapses to the quadratic
$\gamma\sigma^2\lambda\,G^2 + (\sigma^2(1-\gamma) - \lambda)G + 1 = 0$,
the classical Stieltjes equation of $\sigma^2\mathrm{MP}(\gamma)$. Its
discriminant gives edges $\lambda_\pm(0) = \sigma^2(1 \pm \sqrt\gamma)^2$
and density
$\rho_0^\gamma(\lambda) = (2\pi\gamma\sigma^2\lambda)^{-1}\sqrt{(\lambda_+-\lambda)(\lambda-\lambda_-)}$
on $[\lambda_-, \lambda_+]$.

\paragraph{$\tau = \infty$: symmetric free difference of MPs.}
$c_+ = c_- = c := \sigma^2/2$. The edge quartic becomes a biquadratic
in $\lambda$: $\lambda^4 + c^2(1 - 10\gamma - 2\gamma^2)\lambda^2 - \gamma c^4(2-\gamma)^3 = 0$,
with discriminant $(1+4\gamma)^3$. Setting $s = \sqrt{1+4\gamma}$ and
simplifying via $2-\gamma = (3-s)(3+s)/4$ and
$s^4 + 8s^3 + 18s^2 - 27 = (s-1)(s+3)^3$,
\begin{equation}
\boxed{\;\lambda_+(\infty) \;=\; \frac{\sigma^2}{8}(s+3)^{3/2}\sqrt{s-1}, \qquad s = \sqrt{1+4\gamma},\;}
\label{eq:edge-tauinf-appendix}
\end{equation}
and $\lambda_-(\infty) = -\lambda_+(\infty)$ by symmetry. Small-$\gamma$
expansion: $\lambda_+(\infty) = \sigma^2\sqrt{2\gamma}(1 + O(\gamma))$.
Numerical values: at $\gamma = 0.1$, $\lambda_+(\infty) \approx 0.458\,\sigma^2$;
at $\gamma = 0.5$, $\lambda_+(\infty) \approx 1.101\,\sigma^2$;
at $\gamma = 1.0$, $\lambda_+(\infty) \approx 1.665\,\sigma^2$. The
empirical eigenvalue maxima at $\tau = 10$ sit $1\text{--}3\%$ above these closed-form predictions across all $\gamma$ in Figures~\ref{fig:free_conv_overlay_g0p25}--\ref{fig:free_conv_overlay_g5p0}, the expected Tracy--Widom edge concentration.

\paragraph{Small-$\gamma$ semicircle for arbitrary $\tau$.}
For $\gamma \ll 1$ the $R$-transform \eqref{eq:R-explicit-appendix} is
near-linear in $z$, and the bulk of $\mu_\tau^\gamma$ is approximately a
Wigner semicircle of radius $r(\tau) = 2\sqrt{(c_+^2 + c_-^2)\gamma}$
centered at the mean $c_+ - c_- = \sigma^2 e^{-\tau/2}$. Consequently
\begin{equation}
\lambda_+(\tau) \;\approx\; (c_+ - c_-) \;+\; 2\sqrt{(c_+^2 + c_-^2)\,\gamma} \;+\; O(\gamma),
\label{eq:edge-small-gamma-appendix}
\end{equation}
which serves as a closed-form approximation to $\lambda_+(\tau)$ that
avoids numerical root-finding on the quartic $\Delta_\tau$ in the
small-$\gamma$ regime. At $\tau=0$ it recovers
$\sigma^2(1+2\sqrt\gamma) \approx \sigma^2(1+\sqrt\gamma)^2$; at
$\tau=\infty$ it recovers $\sigma^2\sqrt{2\gamma}$.

\subsubsection{Scope and empirical verification}
\label{app:scope}

The upper edge $\lambda_+(\tau)$ is a sharp threshold separating signal eigenvalues from null noise, and the bulk density $\rho_\tau^\gamma$ matches simulated eigenvalue histograms tightly at every $\tau$ (Figures~\ref{fig:free_conv_overlay_g0p25}--\ref{fig:free_conv_overlay_g5p0} below).

\begin{remark}[Score-oracle scope]
\label{rem:score_oracle}
Theorem~\ref{thm:limit-spectrum} calibrates the spectrum under an
isotropic-Gaussian null and so requires the score oracle to be
approximately correct \emph{on iso-Gaussian inputs}, a local
condition strictly weaker than global fidelity to the true $f$. We
distinguish:
\begin{itemize}[leftmargin=*,nosep]
\item \emph{Null fidelity:} the score on iso-Gaussian inputs equals
$-x/\sigma^2$. Required for Theorem~\ref{thm:limit-spectrum} to apply
and for $\lambda_+(\tau)$ to bound the null spectrum. The iso-Gaussian
score is linear and falls within the high-noise, near-Gaussian regime
that Tweedie's identity exploits, so this local condition is
plausible for pretrained denoisers and holds asymptotically for KDE
with shrinking bandwidth.
\item \emph{Signal fidelity:} the score on the data distribution
$f$ equals $\nabla \log f$. Required for the eigenvectors of
$\hat C(\tau)$ to reflect the true metastable structure of $f$
rather than the score model's $\hat f$. This is a stronger,
global condition. Approximation bias (finite NN capacity) and
estimation error (finite training data) both contribute, and
empirical validation against ground truth, as in the alanine
experiment of Section~\ref{sec:exp_molecular}, is the appropriate
check.
\end{itemize}
The two requirements decouple: noise-floor calibration depends only
on null fidelity, while interpretation of the recovered DA directions
depends on signal fidelity.
\end{remark}

\paragraph{Scope.}
Lemma~\ref{lem:decomp} requires that the process whose autocovariance
is eigendecomposed is itself the canonical diffusion. When DA is
applied to a nonlinear transformation of diffusion trajectories (for
example, when latent-space trajectories are embedded through a
pretrained encoder before computing $\hat C(\tau)$), the
decomposition no longer holds exactly, and Theorem~\ref{thm:limit-spectrum}
applies only approximately. Even in the native space, the theorem is
exact for the independent-pair lag estimator analyzed in
Appendix~\ref{app:limit-spectrum}; the trajectory-averaged estimator
\eqref{eq:c_hat_def} used in Algorithm~\ref{alg:ssa_da} averages
overlapping lagged pairs along each trajectory, so those summands are
dependent. We therefore use $\lambda_+(\tau)$ as an analytic
reference floor for the trajectory-averaged estimator, and
in settings with additional nonlinear embeddings we additionally verify
by a matched Monte-Carlo null: simulate the canonical diffusion for an
isotropic Gaussian with variance matched to the embedding-space
covariance, push the trajectories through the same embedding, and take
the top eigenvalue of the resulting $\hat C(\tau)$ as an empirical
floor (applied in Appendix~\ref{app:sdxl_elderly}).

\paragraph{Empirical verification under the Gaussian null.}
Figures~\ref{fig:free_conv_overlay_g0p25}--\ref{fig:free_conv_overlay_g5p0} compare the analytic density $\rho_\tau^\gamma$ from Theorem~\ref{thm:limit-spectrum} to simulated eigenvalue histograms of $\hat C^{\mathrm{sym}}(\tau)$ across six aspect ratios $\gamma \in \{0.25, 0.5, 0.75, 1, 2, 5\}$ at four lags $\tau \in \{0, 0.5, 2, 10\}$ each, with $N = 500$ fixed and eigenvalues pooled over three seeds per panel. At $\tau = 0$ the free-convolution density coincides with the classical Marchenko--Pastur law (Section~\ref{app:consequences}), with a $\delta_0$ atom of mass $\max(0, 1-1/\gamma)$ visible for $\gamma \ge 1$. At every $\tau > 0$ the free-convolution curve tracks the empirical bulk tightly, the analytic upper edge $\lambda_+(\tau)$ contains the bulk, and the atom (when present, $\gamma > 2$) sits at $w_0 = 1 - 2/\gamma$ as predicted. The closed-form $\tau \to \infty$ edge $\lambda_+(\infty)$ from Equation~\eqref{eq:edge-tauinf-appendix} matches the empirical $\tau = 10$ bulk edge across all $\gamma$, confirming that $\lambda_+(\tau)$ is sharp and supporting its use as the noise floor for DA's lag-selection rule (Section~\ref{sec:da}, Appendix~\ref{app:lag_selection}).

\begin{figure}[p]
\centering
\vspace*{\fill}
\begin{minipage}{\textwidth}
\centering
\includegraphics[width=0.62\textwidth]{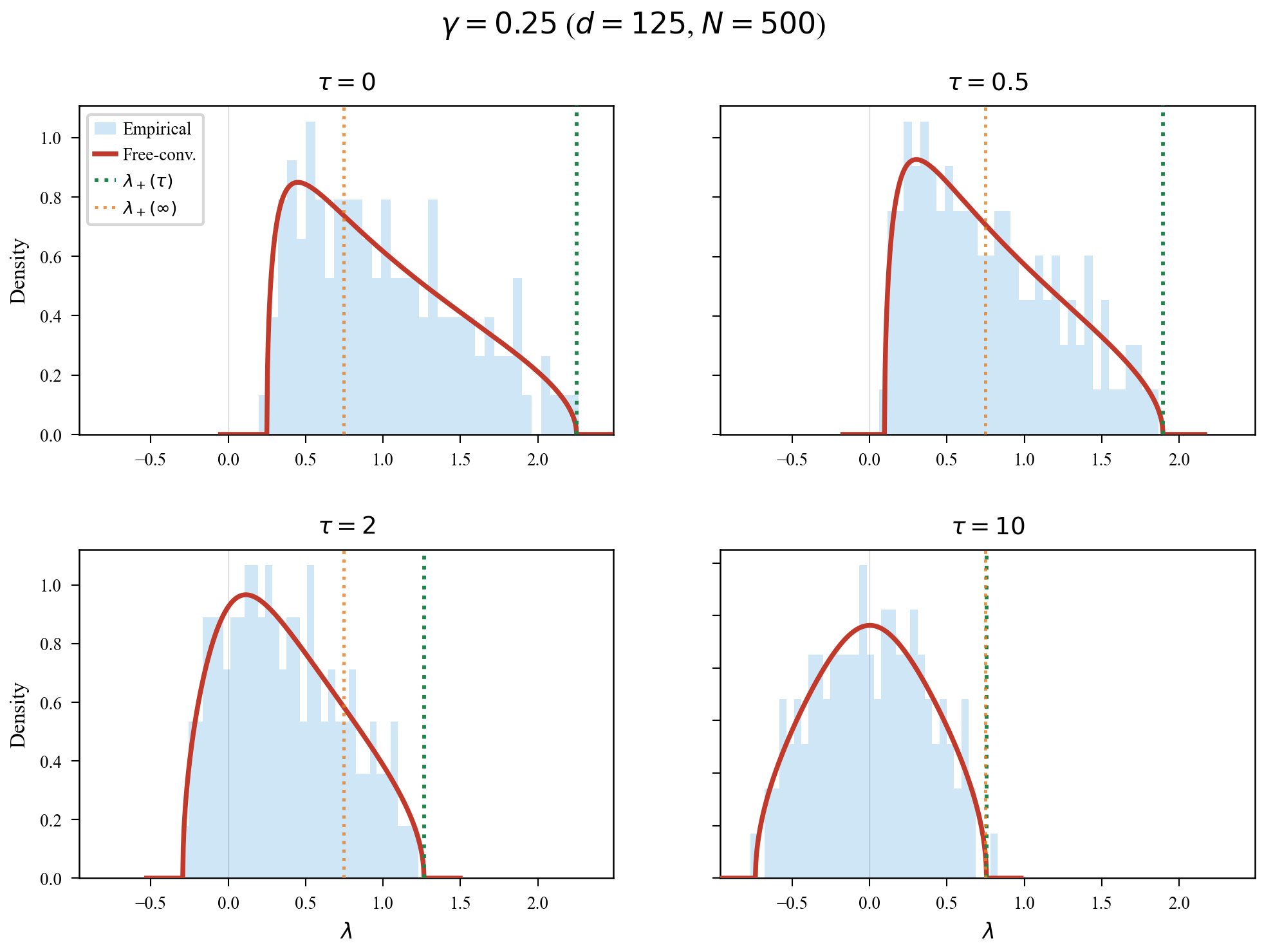}
\captionof{figure}{Free-convolution validation at $\gamma = 0.25$. Light-blue histogram: empirical eigenvalue density of $\hat C^{\mathrm{sym}}(\tau)$ under the isotropic Gaussian null. Red curve: analytic free-convolution density. Red stem at $\lambda = 0$ (when present): $\delta_0$ atom with mass $w_0$. Dark-green dotted: $\lambda_+(\tau)$; orange dotted: $\lambda_+(\infty)$.}
\label{fig:free_conv_overlay_g0p25}
\end{minipage}
\vfill
\begin{minipage}{\textwidth}
\centering
\includegraphics[width=0.62\textwidth]{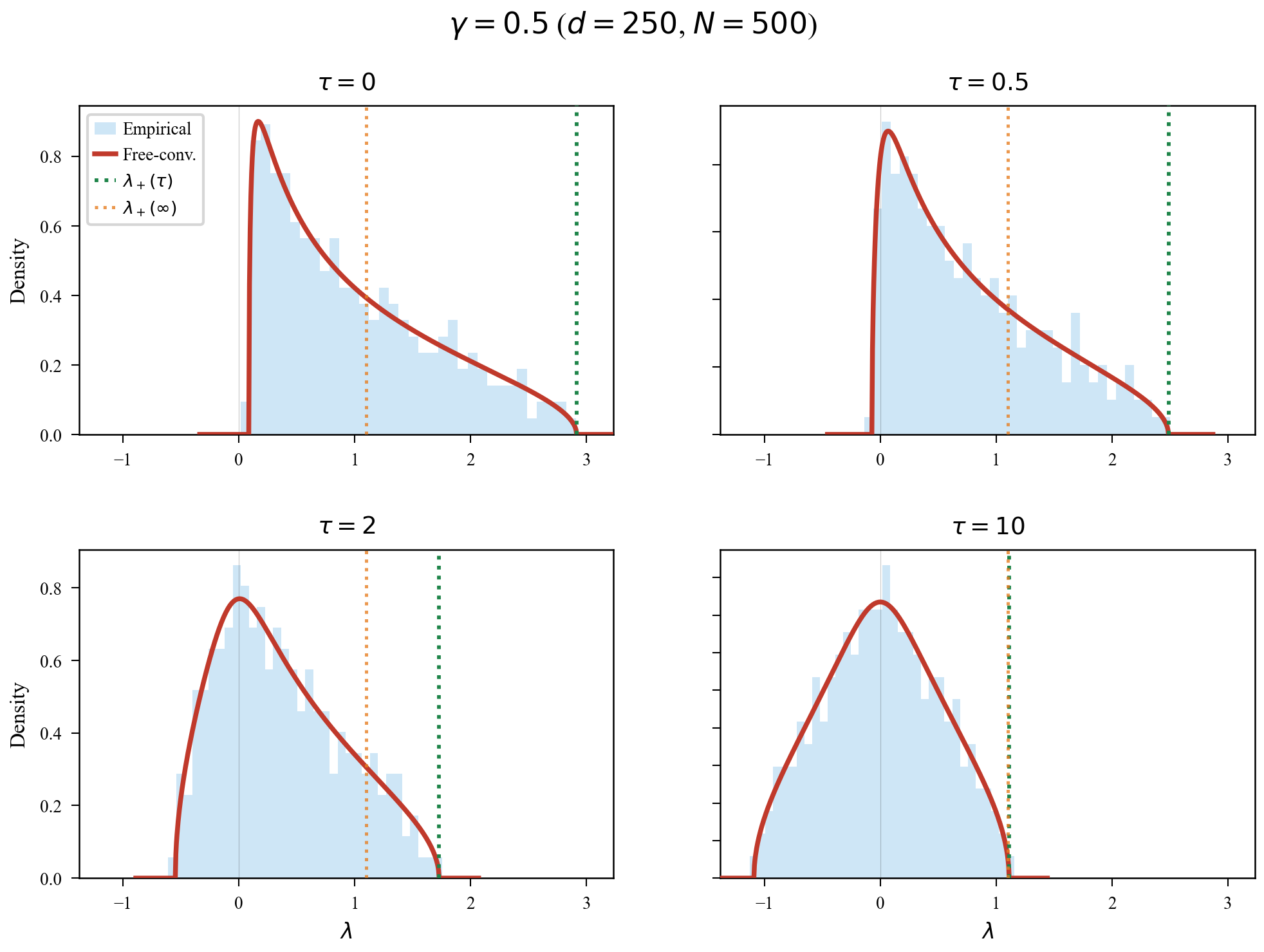}
\captionof{figure}{Free-convolution validation at $\gamma = 0.5$; layout as in Figure~\ref{fig:free_conv_overlay_g0p25}.}
\label{fig:free_conv_overlay_g0p5}
\end{minipage}
\vfill
\begin{minipage}{\textwidth}
\centering
\includegraphics[width=0.62\textwidth]{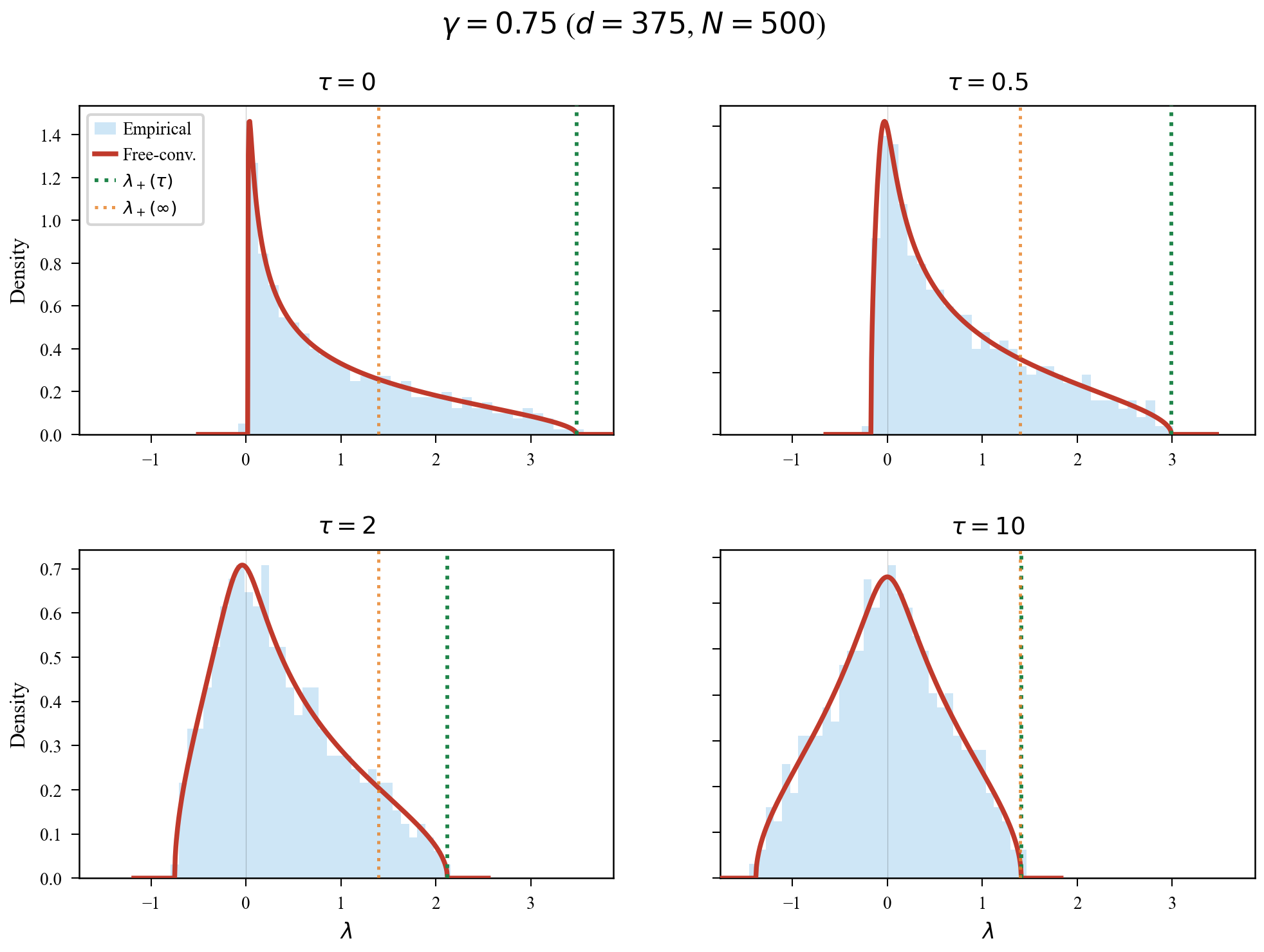}
\captionof{figure}{Free-convolution validation at $\gamma = 0.75$; layout as in Figure~\ref{fig:free_conv_overlay_g0p25}.}
\label{fig:free_conv_overlay_g0p75}
\end{minipage}
\vspace*{\fill}
\end{figure}

\begin{figure}[p]
\centering
\vspace*{\fill}
\begin{minipage}{\textwidth}
\centering
\includegraphics[width=0.62\textwidth]{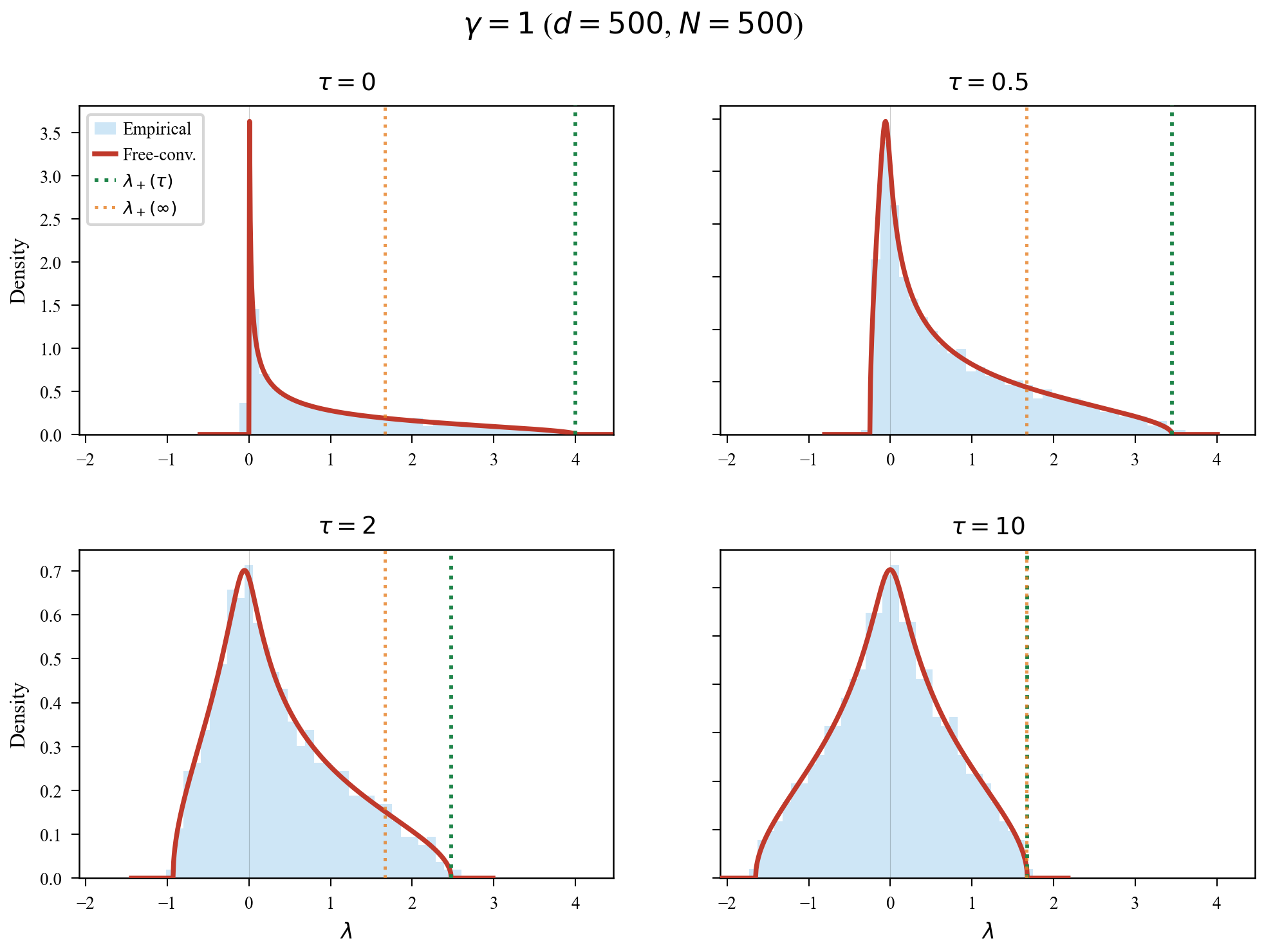}
\captionof{figure}{Free-convolution validation at $\gamma = 1$; the $\tau = 0$ panel shows the expected MP $\delta_0$ atom of mass $1 - 1/\gamma = 0$ at the $\gamma = 1$ boundary (no atom); $\tau > 0$ remains atom-free as $\gamma \le 2$. Layout as in Figure~\ref{fig:free_conv_overlay_g0p25}.}
\label{fig:free_conv_overlay_g1p0}
\end{minipage}
\vfill
\begin{minipage}{\textwidth}
\centering
\includegraphics[width=0.62\textwidth]{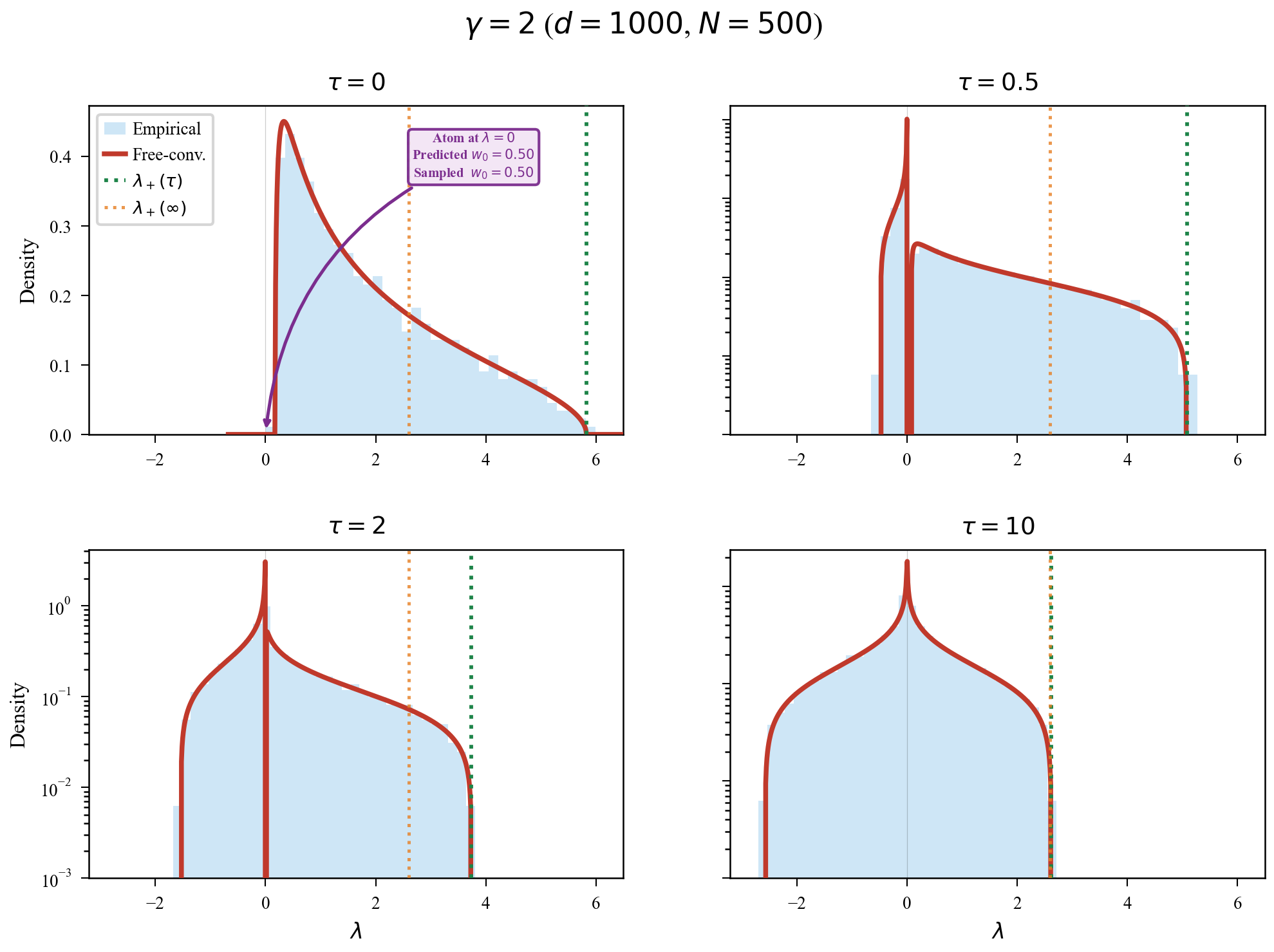}
\captionof{figure}{Free-convolution validation at $\gamma = 2$. The $\tau = 0$ panel has an MP atom of mass $0.5$; the $\tau > 0$ panels are non-atomic (the sub-critical $\gamma = 2$ boundary), with a sharp continuous peak near zero shown on a log-$y$ axis. Layout as in Figure~\ref{fig:free_conv_overlay_g0p25}.}
\label{fig:free_conv_overlay_g2p0}
\end{minipage}
\vfill
\begin{minipage}{\textwidth}
\centering
\includegraphics[width=0.62\textwidth]{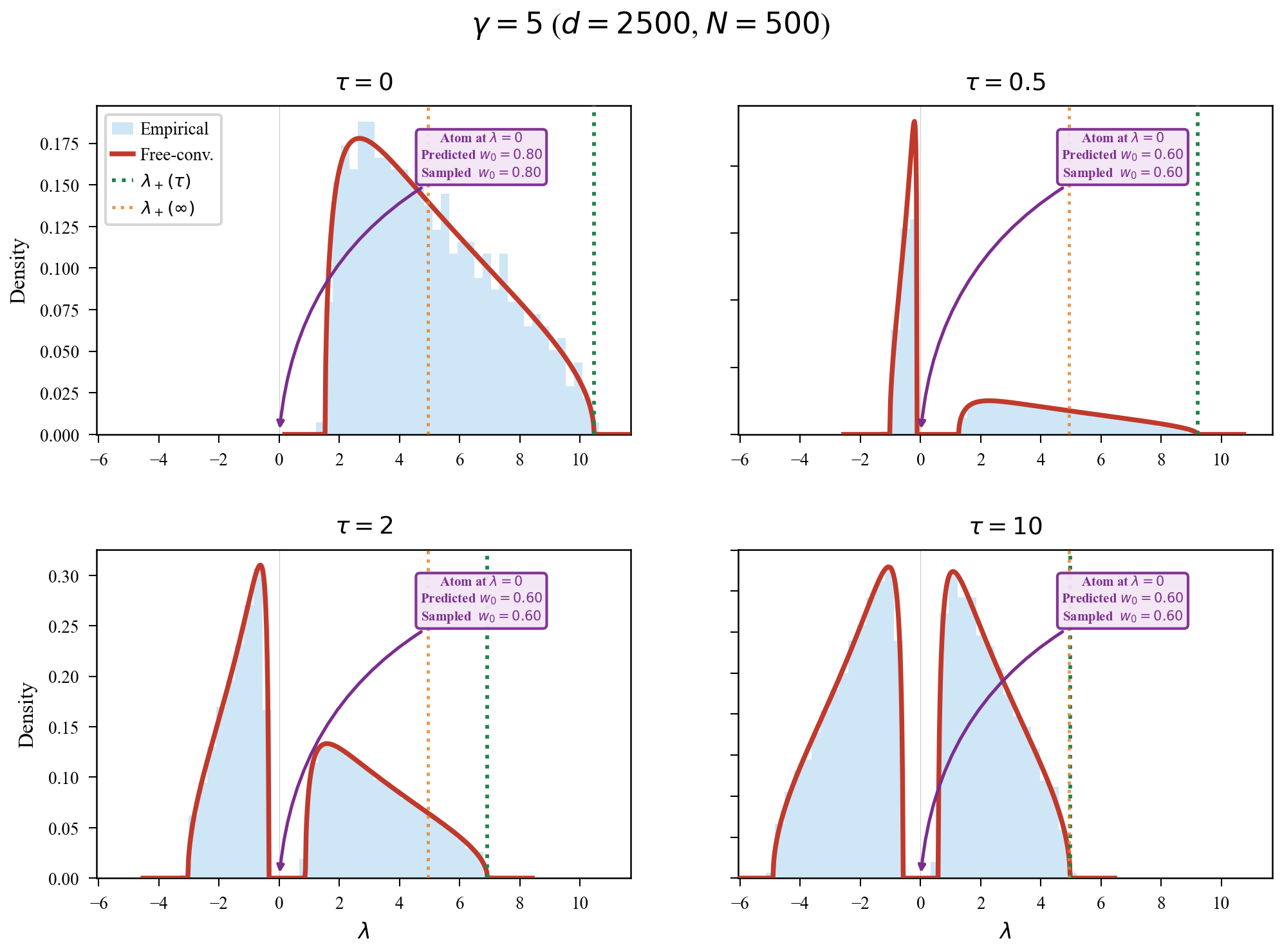}
\captionof{figure}{Free-convolution validation at $\gamma = 5$. The atom mass drops from $1 - 1/\gamma = 0.80$ at $\tau = 0$ (MP) to $1 - 2/\gamma = 0.60$ at every $\tau > 0$ (free-conv). Layout as in Figure~\ref{fig:free_conv_overlay_g0p25}.}
\label{fig:free_conv_overlay_g5p0}
\end{minipage}
\vspace*{\fill}
\end{figure}

\paragraph{Validity on real, non-Gaussian targets.}
Theorem~\ref{thm:limit-spectrum} is proved under an i.i.d.\ isotropic-Gaussian
null. We expect the bulk-and-spike picture it predicts to extend to
broader distributions via Marchenko--Pastur universality
\cite{bai2010,marchenko1967}: classical universality results show that
the MP law governs the bulk of sample-covariance spectra under finite
fourth-moment assumptions far beyond the Gaussian case, and the
free-additive-convolution construction underlying
Theorem~\ref{thm:limit-spectrum} inherits this robustness. We invoke
this universality to apply the analytic bulk and edge as approximations
in the alanine and SDXL experiments of Section~\ref{sec:experiments} (Appendices~\ref{app:alanine_full},~\ref{app:sdxl_details}); the empirical evidence there is that the non-signal eigenvalues track the analytic bulk and the metastable spikes sit outside it, matching the prediction. A formal universality
statement for the symmetrized lag-$\tau$ autocovariance is left to
future work.

\subsubsection{A foundation for further theory}
\label{app:foundation_for_further_theory}

Theorem~\ref{thm:limit-spectrum} establishes the limit spectrum of $\hat C^{\mathrm{sym}}(\tau)$ under the iso-Gaussian null: the bulk density, the upper edge $\lambda_+(\tau)$, and the atom mass for $\gamma > 2$. It does not address edge fluctuations or signal-detection thresholds for embedded slow modes. Two natural extensions remain for future work: (i) a \emph{spiked-model analysis} via BBP-type techniques~\cite{baik2005phase, benaych2011eigenvalues} applied to the free-convolution representation, giving an explicit signal-detection threshold for DA and a minimum $N$ below which slow modes of a given strength are undetectable; (ii) $N^{2/3}$-scale \emph{Tracy--Widom edge fluctuations} via edge-universality results for free additive convolutions~\cite{bao2016, lee2016}, giving a principled data-driven lag-selection rule with explicit false-positive rate.

\subsection{Lag Selection for DA}\label{app:lag_selection}

By Proposition~\ref{prop:da_convergence}, the leading DA direction
$v_1(\tau)$ converges to $\alpha_2/\|\alpha_2\|$ as $\tau \to \infty$,
with rate controlled by the spectral gap $\mu_3 - \mu_2$. For $m \geq 2$
the top-$m$ eigenspace converges to $\mathrm{span}(\alpha_2, \ldots, \alpha_{m+1})$
(Section~\ref{sec:da}). In the population limit the optimal lag is
$\tau \to \infty$; finite $\tau^*$ selection is a concession to limited data.
For recovering $m$ DA directions, the valid window is
\[
\frac{1}{\mu_{m+2} - \mu_{m+1}} \;\ll\; \tau \;\ll\; \frac{\log N}{2\mu_{m+1}},
\]
since the $m$-th direction (governed by $\mu_{m+1}$) is the first to fall below the noise floor as $\tau$ grows. The window is wide precisely in the metastable regime ($\mu_2, \ldots, \mu_{m+1}$ exponentially small in barrier heights, $\mu_{m+2}$ of order one).

\paragraph{Selection rule.}
For $m \geq 1$ directions, one needs all $m$ eigenvectors from a single
eigendecomposition at a common lag so that the directions are mutually
orthogonal. Each of the top $m$ eigenvalues at the chosen lag must
reflect metastable structure rather than finite-sample noise. Theorem~\ref{thm:limit-spectrum} gives the exact null spectral
distribution of $\hat C^{\mathrm{sym}}(\tau)$ under an isotropic Gaussian, with upper edge
$\lambda_+(\tau)$ as its sharp threshold. The selection rule is: choose $\tau^*(m)$ as the largest lag at
which the top $m$ eigenvalues of $\hat{C}(\tau)$ all exceed
$\lambda_+(\tau)$ and $\rho(\tau) < 1/e$. The autocorrelation
requirement excludes short lags at which
$\hat{C}(\tau) \approx \hat{C}(0)$ and any high-variance direction
trivially clears the noise floor. The threshold is taken from the
OU convention~\cite{Pavliotis2014}: for the OU process
$dX_t = -\theta X_t\,dt + \sigma\,dW_t$ the autocorrelation is
$\rho(\tau) = e^{-\theta \tau}$ and the relaxation time
$\tau_c = 1/\theta$ is defined as the unique time at which
$\rho(\tau_c) = e^{-1}$, so $\rho(\tau) < 1/e$ corresponds to ``at
least one relaxation time out,'' filtering directions that have not
yet decayed. The number of metastable directions recovered is then
the maximal $m$ for which $\tau^*(m)$ exists. This balances spectral
filtering (larger $\tau$ improves eigenvector convergence by
Proposition~\ref{prop:da_convergence}) against statistical
detectability (larger $\tau$ pushes weaker eigenvalues toward the
noise floor). The alanine dipeptide experiment
(Section~\ref{sec:exp_molecular}) uses this criterion to extract DA1
and DA2 jointly.

\paragraph{Insufficient lag budget.}
The selection rule succeeds when the window
$1/(\mu_{m+2} - \mu_{m+1}) \ll \tau \ll \log N / (2\mu_{m+1})$ is non-empty, so that the slow eigenvalues remain above the noise floor.
When $N$ is too small relative to the target's spectral gap, this
window can be narrow or absent: the noise floor crosses the signal
eigenvalues before the spectral filter has fully suppressed fast
modes, so the largest feasible $\tau^*$ still sits close to $\tau = 0$.
Since $\hat C(0)$ is (up to trace normalization) the sample covariance,
DA at small $\tau$ is close to PCA; if DA and PCA agree at every
feasible lag, the estimate is ambiguous between (i) DA = PCA being the
true answer for this density, and (ii) an under-resolved estimate that
would separate from PCA at larger $\tau$ if more trajectories were
available. The noise floor scales as $\sigma^2\sqrt{d/N}$
(Appendix~\ref{app:noise_floor}), so increasing $N$ lowers the floor
and extends the feasible lag window, at linear compute cost. Without
ground truth to break the ambiguity, the only principled check is to
verify that DA and PCA have actually separated at the chosen $\tau^*$;
if they have not, one should either raise $N$ or report the result as
inconclusive rather than claiming DA recovers a direction distinct from
PCA. In the experiments reported here (Section~\ref{sec:experiments}; Appendices~\ref{app:gmm_experiments},~\ref{app:sdxl_details},~\ref{app:alanine_full}), DA separates cleanly from PCA at $\tau^*$ in every case.

\subsection{Validation: Anisotropic Null and Dip Test}
\label{app:validation}

The lag-selection criterion in Section~\ref{sec:da} combines two checks on each recovered DA direction: an anisotropic-Gaussian null test for the leading eigenvalue, and Hartigan's dip test on the 1D projection.

\paragraph{Anisotropic-Gaussian null test.}
For each experiment, we draw $B$ replicates from $\mathcal{N}(0, \hat\Sigma)$ with $\hat\Sigma$ matched to the data's empirical covariance, run the full pipeline (Algorithm~\ref{alg:operational_floor} below), and threshold each observed $\lambda_i^{\mathrm{obs}}$ against the empirical null distribution.

\begin{algorithm}[H]
\caption{Operational Noise Floor via Matched-Pipeline Monte Carlo}
\label{alg:operational_floor}
\begin{algorithmic}[1]
 
\State \textbf{Input:} data samples $\{x_n\}_{n=1}^N$, lag $\tau^*$, replicates $B$, pipeline $\mathcal{P}$ (score estimator, canonical diffusion simulator, optional embedding)
\State \textbf{Output:} empirical null distribution $\{\lambda_i^{(b)}\}_{b=1}^B$ for each $i$
\State $\hat\Sigma \gets \tfrac{1}{N}\sum_n (x_n - \bar x)(x_n - \bar x)^\top$ \Comment{covariance match}
\For{$b = 1, \ldots, B$}
    \State Draw $\{x_n^{(b)}\}_{n=1}^N \sim \mathcal{N}(0, \hat\Sigma)$ \Comment{null samples}
    \State Refit score estimator on $\{x_n^{(b)}\}$ if applicable (e.g., KDE)
    \State Run pipeline $\mathcal{P}$ to obtain $\hat C^{(b)}(\tau^*)$ (same diffusion, same embedding)
    \State $\{\lambda_i^{(b)}\}_i \gets \mathrm{eig}(\hat C^{\mathrm{sym},(b)}(\tau^*))$
\EndFor
\State \Return $\{\lambda_i^{(b)}\}$, threshold each observed $\lambda_i^{\mathrm{obs}}$ against the empirical $(1-\alpha)$-quantile of $\{\lambda_i^{(b)}\}_b$
\end{algorithmic}
\end{algorithm}

\paragraph{Dip test.}
Hartigan's dip statistic $D$~\cite{hartigan1985} measures the maximum departure from unimodality of an empirical CDF. We apply it to the 1D projection of the data onto each recovered direction; rejection of $H_0:\ \mathrm{unimodal}$ certifies that the direction carries genuine multimodal structure rather than non-Gaussian unimodal spread.

\paragraph{Hartigan vs. Silverman.}
Hartigan's dip is most sensitive to symmetric bimodality and can
fail to reject on asymmetric mixtures with unequal-mass modes. The
elderly DA1 direction (Appendix~\ref{app:sdxl_elderly}) is one such
case: dip gives $D$ small with $p = 0.234$ (does not reject), while
Silverman's bandwidth test~\cite{silverman1981} gives $p < 0.002$
(rejects). Silverman's test, based on the smallest KDE bandwidth at
which the smoothed density is unimodal, is sensitive to both
symmetric and asymmetric bimodality. We therefore use Hartigan's
dip on directions with approximately symmetric projections (the
alanine DA1 and DA2 directions both pass dip at $p \approx 0$:
DA1 $D = 0.046$, DA2 $D = 0.018$) and Silverman on directions with
asymmetric projections, reporting both when ambiguous.

\subsubsection{Parallel ellipses}
\label{app:validation_ellipses}

$N = 10{,}000$ samples, $B = 1000$ null replicates. Observed $\lambda_1 = 3.49$ exceeds every null replicate ($p < 10^{-3}$). DA1 dip test: $D = 0.087$, $p < 10^{-3}$. PC1 dip test: $D = 0.002$, $p = 0.99$.

\subsection{A Toy Demonstration on a Planted Slow Subspace}
\label{app:planted_spikes}

To illustrate the behavior of the analytic noise floor of Theorem~\ref{thm:limit-spectrum} in a controlled setting, we construct a synthetic density whose slow subspace is known by design and whose metastable rates are all equal. The first construction (Figure~\ref{fig:planted_spikes}) uses a uniform metastable rate across the $k$ slow axes, isolating cleanly how $\lambda_+(\tau)$ separates the slow block from the fast block. A second variant (Figure~\ref{fig:planted_spikes_varying_rate}) introduces a ${\sim}1000\times$ rate spread within the slow block, showing that the plateau characterization does not require within-block uniformity. The realistic regime, in which the slow block itself is unbounded in extent and slow axes drop out of detection one by one at large $\tau$, is exemplified by the alanine dipeptide spectrum in Appendix~\ref{app:alanine_robustness}.

\paragraph{Score-oracle caveat.}
The simulation uses the \emph{analytic} drift $\nabla \log f = -\nabla U$
from the planted potential $U$, so L1 score-estimation error in the
sense of Remark~\ref{rem:score_oracle} is zero by construction. The
experiment thereby validates only the L2 component of
Theorem~\ref{thm:limit-spectrum} — trajectory-pair fluctuations and
spike detection under an oracle score. L1 robustness is exercised
separately in the alanine experiment
(Section~\ref{sec:exp_molecular}), which uses a kernel-density-estimated
score on real data and absorbs L1 inflation via the peeled-bulk
adjustment described in Section~\ref{sec:da}.

\paragraph{Setup.}
Fix $d = 200$ ambient dimensions and $k = 20$ planted metastable axes. Let $Q \in \mathbb{R}^{d \times d}$ be a uniformly random orthogonal matrix; the planted basis is $Q[:, :k]$. The energy in the rotated frame $y = Qx$ is
\[
U(x) = \sum_{i=1}^{k} V(x_i; c, \sigma_s) \;+\; \frac{1}{2 \sigma_{\mathrm{iso}}^2} \sum_{i=k+1}^{d} x_i^2,
\]
with $V(z; c, \sigma_s) = -\log\!\left[\tfrac{1}{2}\mathcal{N}(z; +c, \sigma_s^2) + \tfrac{1}{2}\mathcal{N}(z; -c, \sigma_s^2)\right]$ a symmetric double-well, and $\sigma_{\mathrm{iso}}^2 = c^2 + \sigma_s^2$ chosen so that every coordinate has the same stationary variance $\sigma_f^2 = c^2 + \sigma_s^2$. Matched per-axis variance makes the population covariance proportional to $I_d$, so PCA at $\tau = 0$ has no preferred direction. We use $c = 3$, $\sigma_s = 0.3$, $N = 2000$ canonical-diffusion trajectories, time step $\Delta t = 0.005$, and a log-spaced lag grid up to $\tau_{\max} = 10$. Initial samples are drawn directly from $f$, so the trajectories are stationary from the start; no equilibration is needed.

\paragraph{Bulk and spike spectra.}
Figure~\ref{fig:planted_spikes}(a) shows the eigenvalue histogram of $\hat C(0)$ overlaid with the analytic Marchenko--Pastur density at $\sigma^2 = c^2 + \sigma_s^2$ and $\gamma = d/N = 0.1$. All $200$ eigenvalues lie in the predicted bulk, with no eigenvalues outside. Figure~\ref{fig:planted_spikes}(b) shows the spectrum of $\hat C^{\mathrm{sym}}(\tau^\star)$ at the auto-selected $\tau^\star = 10$: $20$ eigenvalues separate from the bulk and concentrate near $\lambda \approx 9$, while the remaining $180$ eigenvalues track the free-convolution density of Theorem~\ref{thm:limit-spectrum} centered around zero with edge $\lambda_+(\tau^\star) = 4.23$. The detected spike count equals the planted truth.

\paragraph{Plateau and the valid $\tau$-window.}
Figure~\ref{fig:planted_spikes}(c) plots $\mathrm{count}(\tau) = |\{j: \lambda_j(\tau) > \lambda_+(\tau)\}|$ on a symlog $\tau$ axis. The count is one at $\tau = 0$ (a single eigenvalue exceeds the analytic edge by an amount within the Tracy--Widom-scale fluctuation), rises through $9$, $13$, $17$ as fast modes decay below the bulk edge, and locks onto a plateau at $20$ from $\tau \approx 2$ onward. The auto-rule $\hat m = \max_\tau \mathrm{count}(\tau)$, $\tau^\star = \max\{\tau : \mathrm{count}(\tau) = \hat m\}$ returns $\hat m = 20$ and $\tau^\star = 10$, matching the planted $k$. The plateau onset matches the lower edge of the valid $\tau$-window from Appendix~\ref{app:lag_selection}: with $\mu_{k+2} \approx \sigma_f^2/\sigma_{\mathrm{iso}}^2 \approx 0.5$ (the OU rate of the isotropic background) and $\mu_{k+1} \sim \exp(-c^2/(\sigma_s^2 \sigma_f^2)) \approx 1.7 \times 10^{-5}$, the theoretical bound $1/(\mu_{k+2} - \mu_{k+1}) \approx 2.00$ agrees with the empirical onset $\tau \approx 2.00$ to within $1\%$. The plateau is the regime in which the slow block of $k$ metastable axes is jointly well-separated from the fast block of $d-k$ isotropic axes; outside the upper bound $\log N / (2 \mu_2) \approx 2 \times 10^5$ the slowest axes would peel off the count one by one as their eigenvalues fall through the noise floor (the alanine dipeptide spectrum, Appendix~\ref{app:alanine_robustness}, is in that regime).

\paragraph{Sample-complexity sweep.}
Figure~\ref{fig:planted_spikes}(d) sweeps $N \in \{200, 300, 500, 750, 1000, 1500, 2000, 3000, 5000\}$ at fixed $d$ and $k$ and reports the auto-detected count. Recovery is partial below $N \approx 750$ (which corresponds to $\gamma \gtrsim 0.27$) and exact at every $N$ above, yielding $8$ of $20$ at $N = 200$ ($\gamma = 1.0$), $17$ of $20$ at $N = 500$ ($\gamma = 0.4$), and the full $20$ from $N = 750$ onward.

\paragraph{Recovery vs.\ planted $k$.}
Figure~\ref{fig:planted_spikes}(e) sweeps $k \in \{5, 10, 20, 40, 80\}$ at fixed $d = 200$, $N = 2000$, with a fresh random rotation per $k$. The detected count equals the planted $k$ at every value, including $k = 80$ where $40\%$ of the dimensions are metastable.

\paragraph{Subspace alignment vs.\ PCA.}
Figure~\ref{fig:planted_spikes}(f) compares the alignment of the recovered top-$k$ subspace with the planted basis $Q[:, :k]$, measured by $\bigl\| V_{\text{top-}k}(\tau)^\top Q[:, :k] \bigr\|_F^2 / k \in [0, 1]$. Two estimators are shown: DA, which uses the top-$k$ eigenvectors of $\hat C^{\mathrm{sym}}(\tau)$ and depends on $\tau$, and PCA, which uses the top-$k$ eigenvectors of $\hat C(0)$ and is independent of $\tau$. DA's alignment starts at $0.10 = k/d$ at $\tau = 0$ (where $\hat C^{\mathrm{sym}}(0) = \hat C(0)$ and the spectrum is approximately isotropic), rises monotonically with $\tau$, and reaches $0.948$ at $\tau^\star = 10$. PCA's alignment is $0.101$, statistically indistinguishable from chance $k/d = 0.10$. The matched per-axis variance pins PCA at the random-subspace baseline; the same matched-variance setup leaves DA's recovery near its perfect-alignment ceiling.

\paragraph{Generality across rate spread.}
The plateau characterization above does not require the slow block to share a single rate. Figure~\ref{fig:planted_spikes_varying_rate} repeats the experiment with per-axis well widths $\sigma_{s,i}$ varying geometrically over $[0.3, 0.5]$, giving a ${\sim}1000\times$ spread in $\mu_{\mathrm{meta},i}$ across the $k = 20$ axes (panel b). The plateau still forms at the same theoretical lower bound (empirical onset $\tau \approx 2.00$ vs.\ theory $\approx 2.08$, ratio $0.96$), and DA's subspace alignment with the planted basis reaches $0.948$ at $\tau^\star = 10$ while PCA stays at chance $k/d = 0.10$ (panel c). What the $\tau$-window controls is the slow-vs-fast block separation; within-block uniformity is not required.

\begin{figure}[p]
\centering
\vspace*{\fill}
\begin{minipage}{\textwidth}
\centering
\includegraphics[width=0.92\textwidth]{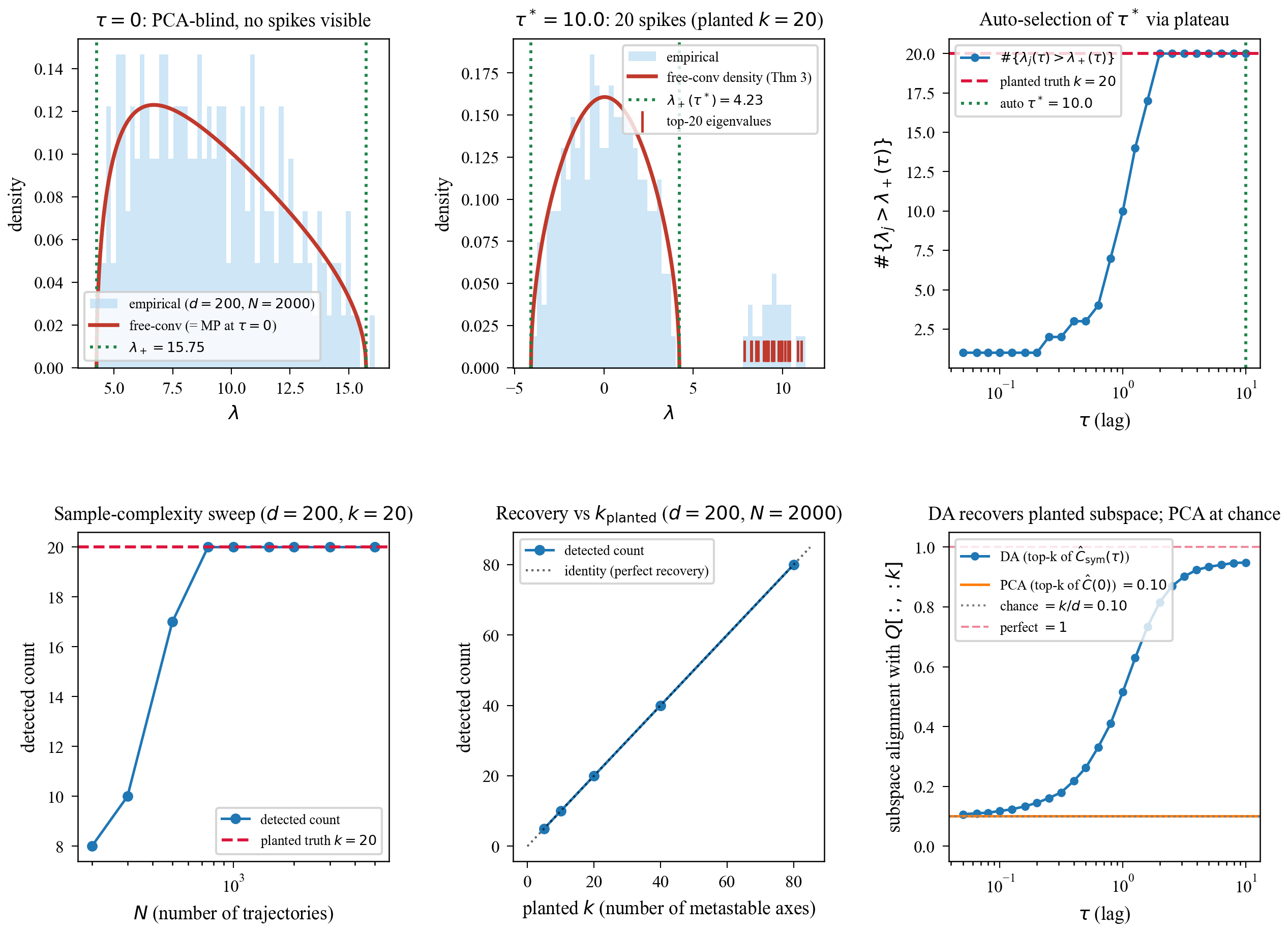}
\captionof{figure}{Validation under planted ground truth, with $d = 200$, $k = 20$, $\gamma = d/N = 0.1$, $c = 3$, $\sigma_s = 0.3$, matched per-axis variance, planted basis $Q[:, :k]$ for a uniformly random orthogonal $Q$.
\textbf{(a)}~$\tau = 0$ eigenvalue histogram of $\hat C(0)$ against the analytic Marchenko--Pastur density; no eigenvalues outside the bulk.
\textbf{(b)}~Spectrum of $\hat C^{\mathrm{sym}}(\tau^\star)$ at the auto-selected $\tau^\star = 10$: $20$ eigenvalues above the free-convolution edge $\lambda_+(\tau^\star) = 4.23$, planted truth $k = 20$.
\textbf{(c)}~Spike count vs.\ $\tau$ on a symlog axis; auto-rule returns $\hat m = 20$, $\tau^\star = 10$.
\textbf{(d)}~Sample-complexity sweep over $N$ at fixed $d$, $k$. Full recovery at $N \geq 750$ ($\gamma \leq 0.27$).
\textbf{(e)}~Recovery vs.\ planted $k$ on the identity line, with a fresh random $Q$ per $k$.
\textbf{(f)}~Subspace alignment of recovered top-$k$ with $Q[:, :k]$. DA reaches $0.948$ at $\tau^\star = 10$; PCA stays at $0.10$ (chance $= k/d$).}
\label{fig:planted_spikes}
\end{minipage}
\vfill
\begin{minipage}{\textwidth}
\centering
\includegraphics[width=0.92\textwidth]{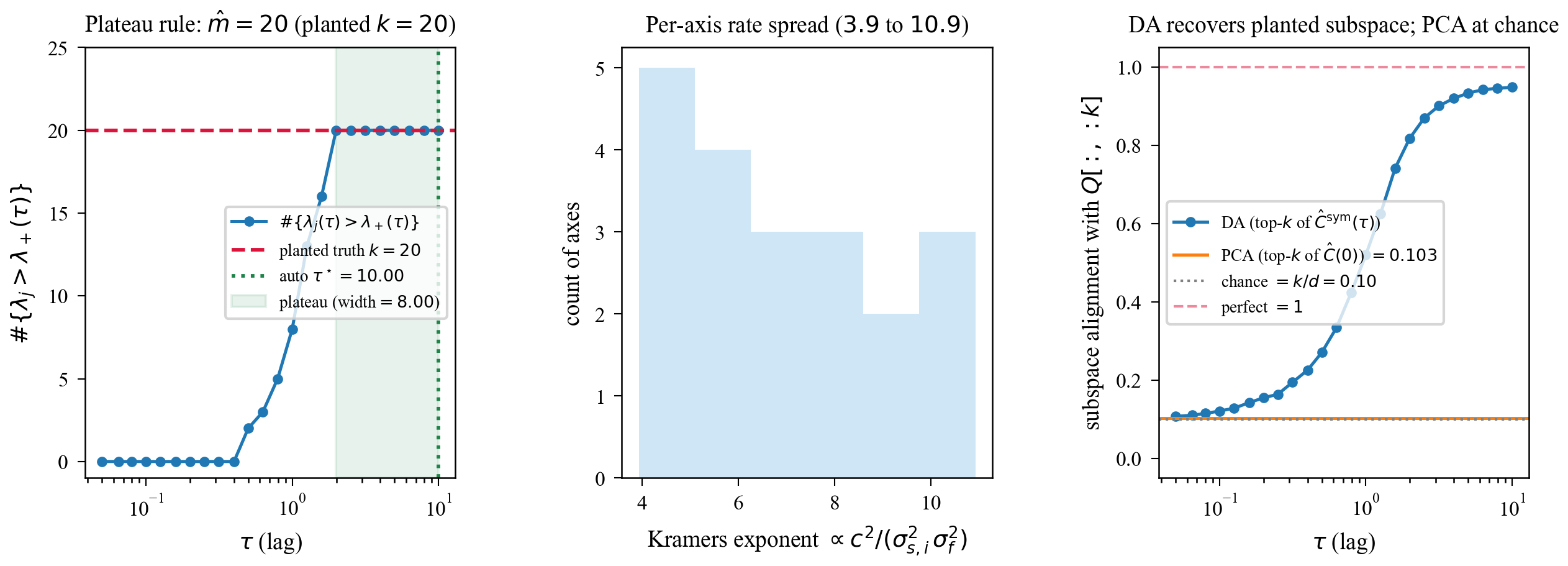}
\captionof{figure}{Same construction as Figure~\ref{fig:planted_spikes}, but with per-axis well widths $\sigma_{s,i}$ varying geometrically over $[0.3, 0.5]$ instead of uniform $\sigma_s = 0.3$.
\textbf{(a)}~Spike count vs.\ $\tau$ on a symlog axis; plateau onset at $\tau \approx 2$ matches the theoretical lower bound $1/(\mu_{k+2} - \mu_{k+1}) \approx 2.08$ to within $1\%$, auto-rule returns $\hat m = 20$, $\tau^\star = 10$.
\textbf{(b)}~Histogram of per-axis Kramers exponents $\propto c^2 / (\sigma_{s,i}^2 \sigma_f^2)$, spanning $[4, 11]$ — a ${\sim}1000\times$ dynamic range in $\mu_{\mathrm{meta}}$.
\textbf{(c)}~Subspace alignment with $Q[:, :k]$ as a function of $\tau$. DA reaches $0.948$ at $\tau^\star = 10$; PCA stays at $0.103 \approx k/d$.}
\label{fig:planted_spikes_varying_rate}
\end{minipage}
\vspace*{\fill}
\end{figure}

\section{Robustness to Process Choice}
\label{app:robustness}

The canonical diffusion of Section~\ref{sec:canonical} is a single element of two natural broader families of processes with stationary density $f$: reversible It\^o diffusions on $\R^d$ with the same stationary $f$ but different constant diffusion matrices $A_0$, and reversible kinetic processes (underdamped Langevin) on phase space $\R^d \times \R^d$ with the same marginal $f$ on positions. SSA and DA values change quantitatively across both families; their \emph{orderings} across densities are preserved in the metastable regime. The first statement is a clean consequence of the variational characterization of self-adjoint operator spectra and is proved here. The second inherits from standard results in stochastic homogenization \cite{pavliotis_stuart2008,Pavliotis2014} and metastability \cite{bovier_dh2015,kramers1940,nelson1967}, and we record its corollary for SSA and DA below.

\subsection{Reversible It\^o Diffusions with Different Constant Diffusion Matrices}
\label{app:robustness_ito}

Fix any constant positive-definite $A_0 \in \R^{d \times d}$. The unique reversible It\^o diffusion with stationary density $f$ and diffusion matrix $A_0$ (Theorem~\ref{thm:uniqueness}) is
\[
dX_t \;=\; \tfrac{1}{2}\,A_0\,\nabla \log f(X_t)\,dt \;+\; A_0^{1/2}\,dW_t.
\]
Its infinitesimal generator
\[
\mathcal{L}_{A_0}\varphi \;=\; \tfrac{1}{2}\nabla \cdot (A_0 \nabla \varphi) \;+\; \tfrac{1}{2}(A_0\nabla \log f)\cdot \nabla \varphi
\]
is self-adjoint on $L^2(f\,dx)$ with discrete non-negative spectrum
$0 = \mu_1(A_0) < \mu_2(A_0) \le \mu_3(A_0) \le \cdots$. We compare the spectra across two choices $A, B \succ 0$ via the standard Rayleigh-quotient comparison.

\begin{proposition}[Eigenvalue ratio bound across diffusion matrices]
\label{prop:robustness_ito}
For any two constant positive-definite diffusion matrices $A, B \in \R^{d\times d}$ and every $k \ge 1$,
\[
\lambda_{\min}\!\bigl(B^{-1/2} A B^{-1/2}\bigr)\,\mu_k(B) \;\le\; \mu_k(A) \;\le\; \lambda_{\max}\!\bigl(B^{-1/2} A B^{-1/2}\bigr)\,\mu_k(B).
\]
\end{proposition}

\begin{proof}
The Dirichlet form associated with $\mathcal{L}_{A_0}$ is
$\mathcal{E}_{A_0}(\varphi) = \tfrac{1}{2}\int (\nabla \varphi)^\top A_0 (\nabla \varphi)\,f\,dx$,
defined on its natural domain $\mathcal{D}(\mathcal{E}_{A_0}) \subset H^1(f\,dx) \cap L^2_0(f\,dx)$ of zero-mean functions. For any test function $\varphi$ in this domain,
\[
\lambda_{\min}\!\bigl(B^{-1/2} A B^{-1/2}\bigr)\,(\nabla \varphi)^\top B (\nabla \varphi) \;\le\; (\nabla \varphi)^\top A (\nabla \varphi) \;\le\; \lambda_{\max}\!\bigl(B^{-1/2} A B^{-1/2}\bigr)\,(\nabla \varphi)^\top B (\nabla \varphi)
\]
pointwise in $x$, by the variational characterization of $\lambda_{\min}, \lambda_{\max}$ of $B^{-1/2} A B^{-1/2}$. Integrating against $f\,dx$,
\[
\lambda_{\min}\!\bigl(B^{-1/2} A B^{-1/2}\bigr)\,\mathcal{E}_B(\varphi) \;\le\; \mathcal{E}_A(\varphi) \;\le\; \lambda_{\max}\!\bigl(B^{-1/2} A B^{-1/2}\bigr)\,\mathcal{E}_B(\varphi).
\]
The min-max characterization of the $k$-th non-trivial eigenvalue,
\[
\mu_k(A_0) \;=\; \min_{\substack{V_k \subset L^2_0(f\,dx) \\ \dim V_k = k}}\;\max_{\substack{\varphi \in V_k \\ \|\varphi\|_{L^2(f)} = 1}} \mathcal{E}_{A_0}(\varphi),
\]
applied with $A_0 = A$ and $A_0 = B$ then gives the stated bound term-by-term.
\end{proof}

\paragraph{Consequences for SSA and DA.}
Three direct consequences:
\begin{itemize}[leftmargin=*,nosep]
\item For $A_0 = c\,\sigma_f^2 I_d$ ($c > 0$), the process is the canonical diffusion run at speed $c$: eigenfunctions, and hence the weights $w_k$, are unchanged, and all eigenvalues scale by $c$ exactly. SSA scales as $1/c$ via the spectral form $S(f) = \sum_{k,\ell \ge 2} w_k w_\ell / (\mu_k + \mu_\ell)$, and DA at lag $\tau$ equals canonical DA at lag $c\tau$. The factor is common to every density, so orderings are preserved. A density-independent $A_0 = c\,I_d$ is different: it rescales eigenvalues by $c/\sigma_f^2$, which varies across densities and can reorder them (Appendix~\ref{app:diffusion_matrix_choice}).
\item For $A_0$ general positive-definite, eigenvalue ratios $\mu_k(A_0)/\mu_\ell(A_0)$ are preserved across choices up to a multiplicative factor of the condition number of $B^{-1/2} A B^{-1/2}$. In the metastable regime, where the slow eigenvalues are exponentially small in the barrier heights of $-\log f$, the $f$-determined Arrhenius exponents dominate the multiplicative factor, so slow-mode orderings $\mu_k(f_A) \lessgtr \mu_k(f_B)$ across densities are preserved asymptotically whenever the barrier-height gap dominates the $A_0$-dependent prefactors.
\item DA directions depend on $A_0$ in general. In the strongly-metastable bimodal limit (Appendix~\ref{app:bimodal_heuristic}), the slowest eigenfunction is approximately piecewise constant on the two basins, a property of $f$ alone. The leading DA direction therefore converges to the inter-mode direction $\mu_A - \mu_B$ for any $A_0$ in this regime.
\end{itemize}

\subsection{Reversible Kinetic (Underdamped Langevin) Dynamics}
\label{app:robustness_kinetic}

Real physical systems (proteins, peptides, biomolecules in solvent) are governed not by the canonical overdamped diffusion but by underdamped Langevin dynamics on phase space $(q, p) \in \R^d \times \R^d$:
\begin{equation}
dq_t \;=\; M^{-1} p_t\,dt, \qquad
dp_t \;=\; -\nabla U(q_t)\,dt \;-\; \eta\,M^{-1} p_t\,dt \;+\; \sqrt{2 \eta \beta^{-1}}\,dW_t,
\label{eq:underdamped}
\end{equation}
with mass matrix $M \succ 0$, friction $\eta > 0$, inverse temperature $\beta = 1/k_B T$, and potential $U(q) = -\beta^{-1}\log f(q)$. The joint stationary distribution on phase space is
$\pi(q, p) = f(q)\,\mathcal{N}\!\bigl(0, \beta^{-1} M\bigr)(p)$,
so the marginal on $q$ is the target density $f$. The generator $\mathcal{L}^{\rm UD}$ acts on functions of $(q, p)$ and satisfies generalized detailed balance under momentum reversal $(q, p) \mapsto (q, -p)$ (it is hypocoercive and not self-adjoint in $L^2(\pi)$; its spectrum need not be real in general). We write $\mu_k^{\rm UD}(\eta, \beta)$ for its slow real relaxation rates in the metastable asymptotic regime considered below, and $\mu_k^{\rm OD}(f)$ for those of the overdamped canonical diffusion attached to $f$.

Two standard limit theorems describe the relationship between $\mu_k^{\rm UD}(\eta, \beta)$ and $\mu_k^{\rm OD}(f)$.

\begin{proposition}[Smoluchowski limit; \cite{pavliotis_stuart2008,Pavliotis2014}]
\label{thm:smoluchowski}
As $\eta \to \infty$, the time-rescaled position trajectory $q_{\eta t}$ of the underdamped Langevin equation~\eqref{eq:underdamped} converges weakly to an overdamped Langevin diffusion with stationary marginal $f$ and a constant diffusion matrix determined by $M$ and $\beta$. Its slow eigenvalues are the overdamped rates of that limiting diffusion; by Appendix~\ref{app:robustness} (Proposition~\ref{prop:robustness_ito}), they share the same metastable barrier exponents as the canonical scalar-diffusion rates $\mu_k^{\rm OD}(f)$.
\end{proposition}

\paragraph{Eyring--Kramers asymptotics (informal).}
Standard Eyring--Kramers theory~\cite{kramers1940,bovier_dh2015,nelson1967} for a metastable density $f$ with a finite collection of well-separated basins (well-defined saddle structure, non-degenerate quadratic minima and saddles) gives slow rates of the underdamped Langevin generator at fixed positive $\eta$ and $M$ of the form
\[
\mu_k^{\rm UD}(\eta, \beta;\, f) \;\sim\; \kappa_k(\eta, M;\, f)\,\exp\!\bigl(-\beta\,\Delta U_k(f)\bigr) \quad \text{as } \beta \to \infty,
\]
with $\kappa_k > 0$ depending on $\eta, M$ and local curvatures, and $\Delta U_k(f)$ a barrier height determined by $f$ alone. The same form holds for the overdamped rates $\mu_k^{\rm OD}$ with a different prefactor.

\paragraph{Asymptotic ordering implication.}
\label{cor:kinetic_ordering}
For two densities $f_A, f_B$ with non-degenerate barrier-height ordering at index $k$, the Eyring--Kramers form implies that in the deep-barrier limit $\beta\,|\Delta U_k(f_A) - \Delta U_k(f_B)| \to \infty$ (at fixed $\eta, M$), the sign of $\mu_k^{\rm UD}(f_A) - \mu_k^{\rm UD}(f_B)$ is determined by the barrier-height comparison and matches the corresponding overdamped sign. This is an asymptotic implication of standard metastability results, not a free-standing theorem; we state it as the heuristic anchor for the empirical agreement between DA on the canonical diffusion and TICA on underdamped MD trajectories (Section~\ref{sec:exp_molecular}).

\paragraph{Derivation sketch.}
The Eyring--Kramers form gives $\mu_k^{\rm UD}(\eta, \beta;\, f) \sim \kappa_k(\eta, M;\, f)\exp(-\beta \Delta U_k(f))$ as $\beta \to \infty$. Taking the ratio of the rates for $f_A$ and $f_B$, the prefactor ratio is finite and positive while the exponent diverges (in absolute value) as $\beta \to \infty$, so the sign of $\mu_k^{\rm UD}(f_A) - \mu_k^{\rm UD}(f_B)$ is determined by $\Delta U_k(f_B) - \Delta U_k(f_A)$ in the deep-barrier limit. The same applies to $\mu_k^{\rm OD}$ via the Smoluchowski limit.

\paragraph{Consequence for slow-mode orderings.}
For \emph{leading exponential escape rates}, the deep-barrier asymptotic predicts that the canonical diffusion and any underdamped Langevin process with the same stationary marginal $f$ have the same slow-CV ordering, since both inherit it from the $f$-determined barrier-height profile. SSA values, DA directions, and prefactors can still differ between the two dynamics outside the asymptotic regime; what is preserved asymptotically is the ordering of the leading exponents. The empirical agreement between DA's static-sample recovery on alanine dipeptide (Section~\ref{sec:exp_molecular}) and TICA's trajectory-based recovery is consistent with this prediction in the regime $\beta \Delta U_k \gg 1$ (alanine dihedrals at 300\,K).
\paragraph{Algorithmic relation to TICA.}
DA and TICA differ in the eigenvalue-problem formulation, not merely in the underlying dynamics. DA solves the standard eigenvalue problem $C(\tau)v = \lambda v$, retaining variance weighting and ranking directions by absolute slow-mode contribution to the autocovariance. TICA solves the generalized eigenvalue problem $C(\tau)v = \lambda C(0)v$, normalizing out per-direction variance and ranking directions by relaxation \emph{rate}. Had we adopted the full-covariance convention $A_0 = \Cov_f(X)$ instead of the scalar-variance choice $A_0 = \sigma_f^2 I_d$ (Appendix~\ref{app:diffusion_matrix_choice}), DA and TICA would coincide in that respect. We did not, for two reasons: a conceptual one (the affine-invariant construction has counterintuitive edge cases) and a computational one (DA avoids inverting $\hat C(0)$, which scales better in high dimensions). Systematically exploring the full-covariance variant is a natural direction for future work.

\paragraph{Quantitative comparison at finite $(\eta, \beta)$.}
Without taking limits, comparing $\eta\,\mu_k^{\rm UD}(\eta, \beta)$ to $\mu_k^{\rm OD}(f)$ at finite friction and finite temperature would require a non-asymptotic relative-error statement of the form
\[
\Bigl|\,\eta\,\mu_k^{\rm UD}(\eta, \beta;\, f) \,/\, \mu_k^{\rm OD}(f) \;-\; 1\,\Bigr| \;\le\; C_1/\eta^2 \;+\; C_2\,e^{-\beta\,\Delta U_k(f)}.
\]
A clean version of this bound is plausible by perturbation theory of the underdamped generator on phase space, but is outside the scope of this paper; we leave it to follow-up work. The asymptotic regime matters: outside it, slow-mode orderings between underdamped and overdamped dynamics can flip, since polynomial-in-width escape costs (overdamped diffusion across a wide shallow basin) compete with kinetic-energy-aided ballistic crossings (underdamped through a narrow deep barrier when $\beta\,\Delta U_k \sim 1$). Alanine dipeptide at 300\,K satisfies $\beta\,\Delta U_k \gg 1$ for the dihedral barriers, so the corollary applies.

\section{Score Estimation via Diffusion Models}
\label{app:tweedie}

\subsection{The Tweedie Identity}
A diffusion model defines a forward noising process $q(x_t \mid x_0) = \mathcal{N}(x_t;
\sqrt{\bar\alpha_t} x_0, \sigma_t^2 I)$, where $\sigma_t^2 = 1 - \bar\alpha_t$.
The marginal at time $t$ is
\[
f_t(x) = \int q(x \mid x_0) f(x_0)\, dx_0.
\]
Using Bayes' rule to express the posterior $p(x_0 \mid x_t)$, the score of this marginal satisfies
\[
\nabla \log f_t(x) = -\frac{x - \sqrt{\bar\alpha_t}\,\E[x_0 \mid x_t = x]}{\sigma_t^2}.
\]
One can notice that any sample from this kernel can be written in the reparameterized form
\[
x_t = \sqrt{\bar\alpha_t}\,x_0 + \sigma_t\,\epsilon, \qquad \epsilon \sim \mathcal{N}(0, I),
\]
which isolates the structured signal $\sqrt{\bar\alpha_t}\,x_0$ from the additive
Gaussian noise $\sigma_t\,\epsilon$. Rearranging gives the residual noise explicitly:
\[
\epsilon = \frac{x_t - \sqrt{\bar\alpha_t}\,x_0}{\sigma_t},
\qquad \text{equivalently,} \qquad
x_0 = \frac{x_t - \sigma_t\,\epsilon}{\sqrt{\bar\alpha_t}}.
\]
A diffusion model trains a denoiser $\epsilon_\theta(x_t, t)$ to predict this noise $\epsilon$,
so that $\E[x_0 \mid x_t] \approx (x_t - \sigma_t\,\epsilon_\theta(x_t, t))/\sqrt{\bar\alpha_t}$.
Substituting into the score expression above yields the Tweedie identity:
\[
\nabla \log f_t(x) = -\frac{\epsilon_\theta(x, t)}{\sigma_t}.
\]

\subsection{Recovery of the Data Score as $t \to 0$}
As $t \to 0$, $\sigma_t \to 0$ and $\sqrt{\bar\alpha_t} \to 1$, so the forward kernel
$q(x_t \mid x_0)$ concentrates: $f_t \to f$ in distribution. More precisely, for any
smooth test function $\phi$,
\[
\int \phi(x) f_t(x)\, dx \;\to\; \int \phi(x) f(x)\, dx,
\]
and under mild regularity on $f$ the scores converge pointwise:
$\nabla \log f_t(x) \to \nabla \log f(x)$.
Consequently, evaluating the Tweedie identity at a small but finite $t_{\mathrm{small}}$
gives
\[
\nabla \log f(x) \approx -\frac{\epsilon_\theta(x,\, t_{\mathrm{small}})}{\sigma_{t_{\mathrm{small}}}},
\]
with approximation error controlled by $\sigma_{t_{\mathrm{small}}}$ and the smoothness
of $f$. This provides a tractable score oracle for the data distribution through a
single forward pass of the trained denoiser, with no density estimation step and no
dependence on dimensionality. The division by $\sigma_t$ is numerically benign at the
first discrete timestep of a standard linear schedule ($\sigma_1 \approx 0.01$).

\subsection{Iso-Gaussian Null via Large-$t$ Score Evaluation}
\label{app:tweedie_largeT}
The same trained denoiser provides a self-contained iso-Gaussian score
oracle for the operational noise floor of Algorithm~\ref{alg:operational_floor}.
By construction, the noised marginal $f_t$ at large $t$ is close to
$\mathcal{N}(0, \sigma_t^2 I)$, so the denoiser learns the linear iso-Gaussian
score $\nabla \log f_t(x) \approx -x/\sigma_t^2$ at that level. Evaluating
the canonical diffusion with this same score network at
$t = t_{\mathrm{null}} \gg 0$ therefore runs the iso-Gaussian null through
the entire downstream pipeline (sampling, decoding, embedding) at no extra
training cost. We use this for the elderly DA experiment in
Section~\ref{sec:exp_sdxl} (details in Appendix~\ref{app:sdxl_elderly}).

\section{Synthetic GMM Experiments}
\label{app:gmm_experiments}

This appendix details the GMM simulation setup, the mutual information ground-truth comparator, distribution snapshots from the parameter sweeps, the random-GMM sweep that confirms SSA-MI rank correlation across many configurations, and bootstrap analysis of the mode-count sweep.

\subsection{GMM Simulation Details}
\label{app:gmm_details}

All GMM experiments use 10-dimensional bimodal Gaussian mixtures with
$N = 200$ trajectories of the canonical diffusion, $\Delta t = 0.01$
Euler--Maruyama steps, and physical horizon $T_{\mathrm{phys}} = 200$\,s
(so $T = 20{,}000$ steps). The truncation lag $T_{\max}$ is selected by
the $1/N$ stopping rule of \eqref{eq:tmax_stopping}. The diffusion
coefficient follows the scalar variance convention
$\sigma_f^2 = \frac{1}{d}\tr(\Cov_f(X))$. Stationary moments are
computed analytically:
\[
\bar x_f = \sum_{i=1}^{k} \pi_i \mu_i, \qquad
\Sigma_f = \sum_{i=1}^{k} \pi_i \left[\Sigma_i + (\mu_i - \bar x_f)(\mu_i - \bar x_f)^\top\right].
\]
The four parameter sweeps are configured as follows.
\textbf{Separation:} two equal-weight isotropic-variance components
($\sigma^2 = 1$ in all 10 dimensions for both modes) with the two mode
means separated along dimension~0 only (the other 9 coordinates
identical between the two modes); $30$ linearly spaced separations
from $0.5$ to $8.0$, of which the $22$ points with $\mathrm{sep} \le 6.0$
(all-seed-converged) are displayed in Figure~\ref{fig:gmm_sweeps}.
\textbf{Variance:} two equal-weight components at fixed separation
$4.0$ along dimension~0, with isotropic component variance varying
across $30$ linearly spaced values from $0.1$ to $5.0$; the $26$ points
with $\mathrm{var} \ge 0.78$ (all-seed-converged) are displayed.
\textbf{Weight:} two components at separation $4.0$ along dimension~0
and unit variance, with $\pi_1$ varying across $20$ linearly spaced
values from $0.5$ to $0.95$ ($\pi_2 = 1 - \pi_1$); all $20$ points are
all-seed-converged. \textbf{Mode count:} $K$ equal-weight isotropic
components with variance $0.5$ placed on a line along dimension~0 at
positions $0, 4, 8, 12, 16$ (other 9 coordinates identical across
modes), $K$ varying from $1$ to $5$.

The \textbf{random sweep} (Figure~\ref{fig:random_sweep}) samples 50
bimodal GMMs in $d = 10$ with each instance drawing three independent
parameters: separation $\sim U(1.0, 8.0)$ (along dimension~0 only;
other 9 dimensions have zero mean separation), isotropic variance
$\sigma^2 \sim U(0.5, 3.0)$ (same in all 10 dimensions, identical for
both modes), and mode-1 weight $\pi_1 \sim U(0.5, 0.95)$
($\pi_2 = 1 - \pi_1$). Each point of the four sweeps in
Figure~\ref{fig:gmm_sweeps} is computed over $10$ independent seeds; we
report the mean and $\pm 1$ SD. Seed-to-seed coefficient of variation
of $\hat S$ is $2.4\%$ on average across all parameter settings, with
a maximum of $12\%$ at the highest displayed separation
($\mathrm{sep} = 6.0$).

\subsection{Mutual Information Baseline}
\label{app:mi_baseline}

For a GMM $X \sim \sum_{i=1}^{k} \pi_i \mathcal{N}(\mu_i, \Sigma_i)$ with
latent component indicator $Z$, the mutual information is
$I(X;Z) = H(Z) - H(Z \mid X)$, where
$H(Z) = -\sum_i \pi_i \log \pi_i$ and the conditional entropy is
$H(Z \mid X) = -\int \phi_{\mathrm{mix}}(x) \sum_i q_i(x) \log q_i(x)\,dx$
with posterior $q_i(x) = \pi_i \phi(x;\mu_i,\Sigma_i) / \phi_{\mathrm{mix}}(x)$.
Well-separated modes yield $H(Z \mid X) \approx 0$ and
$I(X;Z) \approx H(Z)$; overlapping modes yield $I(X;Z) \approx 0$. MI
captures static distinguishability of components from single observations;
SSA captures the dynamic difficulty of transitioning between them. We use
MI as a validation baseline: both should increase with mode separation, but
they can diverge when static overlap differs from barrier height.

\subsection{Distribution Snapshots}

\begin{figure}[H]
\centering
\includegraphics[width=0.8\textwidth]{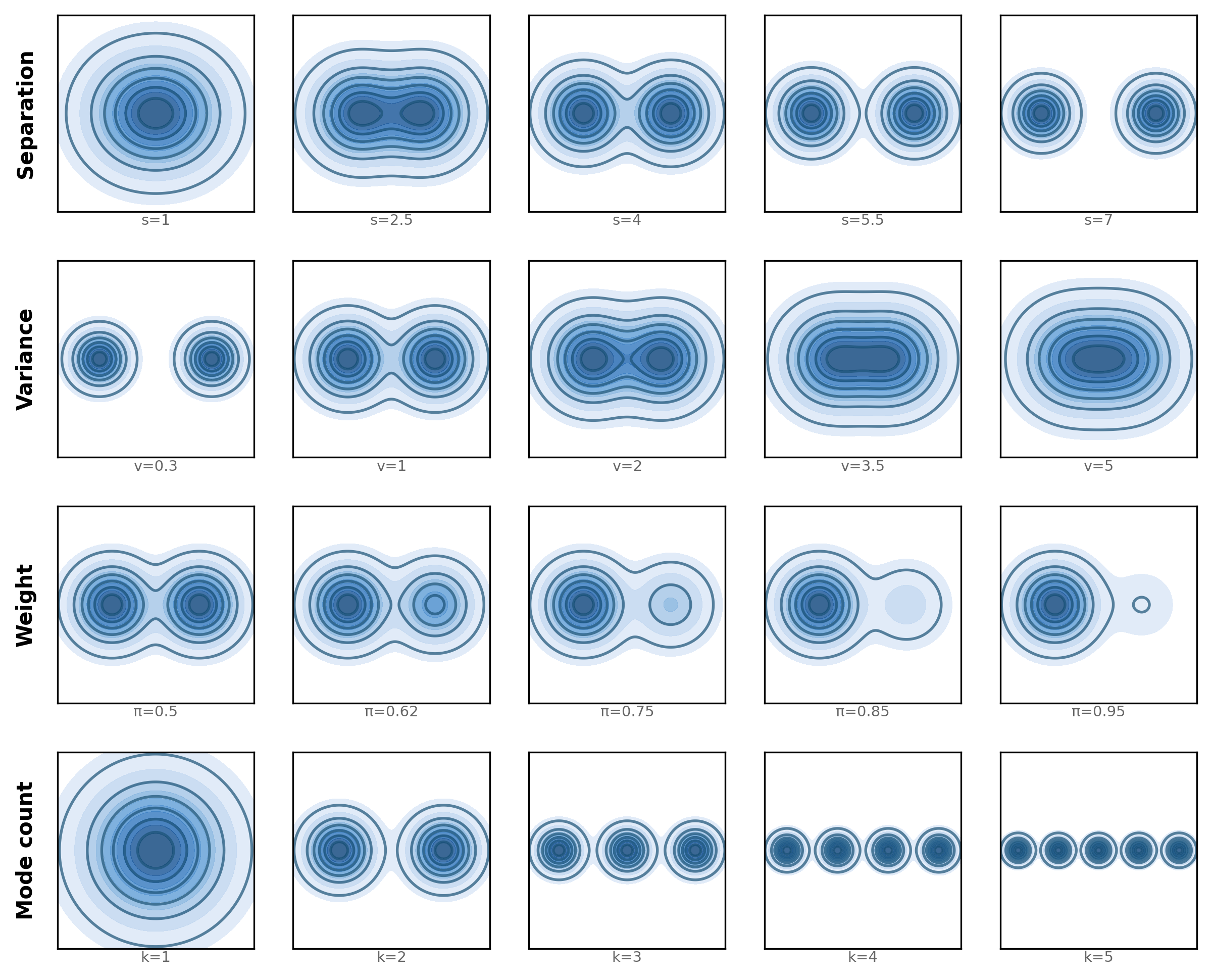}
\caption{  \textbf{2D illustrations of the sweep parameters used in
Figure~\ref{fig:gmm_sweeps}}, shown at five representative points along
each sweep. The actual experiments are run in $d = 10$ (see
Appendix~\ref{app:gmm_details}); the panels here are 2D analogues of the
same parameter configurations, rendered for visual intuition only.
Each row corresponds to one sweep (separation, variance, weight, mode
count); parameter values are labeled below each panel. Density is
rendered as filled contours with square-root scaling to preserve
visibility of low-weight modes.}
\label{fig:gmm_snapshots_appendix}
\end{figure}

\subsection{Random GMM Sweep}
\label{app:gmm_random}

\begin{figure}[H]
\centering
\includegraphics[width=0.35\textwidth]{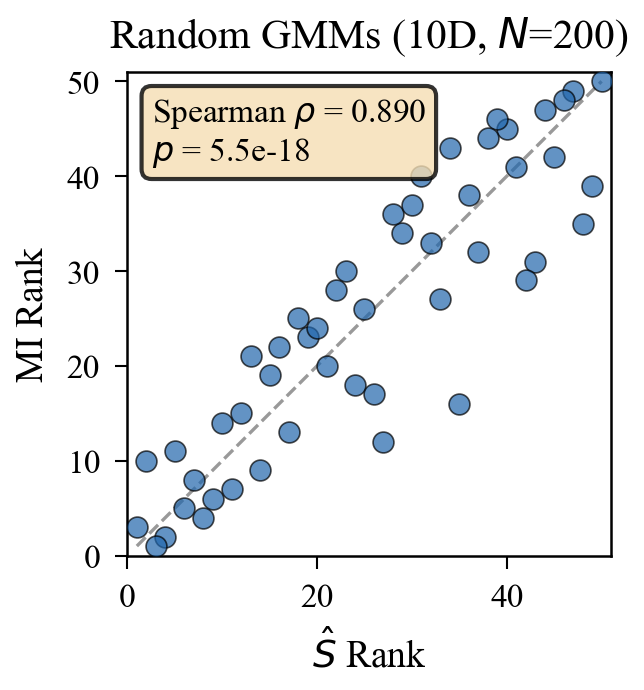}
\caption{  Rank comparison of SSA and MI across 50 random bimodal
GMMs in $d = 10$ ($N = 200$ trajectories per GMM). Spearman
$\rho = 0.890$, $p = 5.5 \times 10^{-18}$. Dashed line indicates
perfect rank agreement.}
\label{fig:random_sweep}
\end{figure}

The rank correlation scatter in Figure~\ref{fig:random_sweep} is
computed from 50 bimodal GMMs in $d = 10$ with separation along
dimension~0 sampled uniformly from $[1, 8]$, isotropic variance
(identical across all 10 dimensions and both modes) from $[0.5, 3]$,
and mode-1 weight from $[0.5, 0.95]$. Each GMM is evaluated with
$N = 200$ trajectories of the canonical diffusion at $\Delta t = 0.01$
and physical horizon $T_{\mathrm{phys}} = 100$\,s, with truncation by
the $1/N$ rule. Spearman rank correlation between SSA and MI is
$\rho = 0.890$ ($p = 5.5 \times 10^{-18}$), confirming agreement on
ordering across diverse configurations.

\subsection{Bootstrap analysis of the mode-count sweep}
\label{app:gmm_bootstrap}

Of the five mode-count points in Figure~\ref{fig:gmm_sweeps}, $K=4$
and $K=5$ did not meet the trace-CLT stopping criterion
(Equation~\ref{eq:tmax_stopping}) within $T_{\mathrm{phys}} = 200$\,s in
all $10$ seeds. We apply the trajectory bootstrap of
Appendix~\ref{app:sdxl_bootstrap}: within each seed, the $N = 200$
trajectories are resampled with replacement and $\hat S$ is recomputed
end-to-end ($B = 10{,}000$ replicates). Trajectory seeds are not
shared across $K$, so each $K$ is bootstrapped independently;
$\Pr[\hat S_K > \hat S_{K-1}]$ is computed from independent draws.
Bootstrap replicates are pooled across the $10$ seeds.

\begin{table}[H]
\centering
\caption{  Mode-count sweep: trajectory bootstrap
($B = 10{,}000$, $10$ seeds). $\dagger$ marks points that did not meet
the $1/N$ stopping criterion at $T_{\mathrm{phys}} = 200$\,s; by
Proposition~\ref{prop:ssa_monotonicity} the reported $\hat S$
lower-bounds the population value.}
\label{tab:gmm_bootstrap}
 
\begin{tabular}{@{}cccc@{}}
\toprule
$K$ & $\hat S$ & SE & $\Pr[\hat S_K > \hat S_{K-1}]$ \\
\midrule
$1$           & $0.989$  & $0.001$ & --      \\
$2$           & $6.088$  & $0.141$ & $1.000$ \\
$3$           & $14.489$ & $0.424$ & $1.000$ \\
$4^{\dagger}$ & $20.831$ & $0.587$ & $1.000$ \\
$5^{\dagger}$ & $24.734$ & $0.642$ & $1.000$ \\
\bottomrule
\end{tabular}
\end{table}

All four pairwise ordering probabilities equal $1.000$, so the
monotonic ranking $K \to \hat S$ is robust despite $K = 4, 5$ not
fully converging within the simulation horizon. The mode-count sweep is a fixed-spectral-profile comparison of the kind covered by Proposition~\ref{prop:ssa_monotonicity}: longer horizons widen the gaps without inverting them; the bootstrap further confirms stability over the lag window.

\section{Alanine Dipeptide}
\label{app:alanine_full}

This appendix details the alanine dipeptide experiment: KDE score-estimation choices, robustness analysis of the DA result, and TICA lag-implied-timescale validation.

\subsection{Alanine Dipeptide Score Estimation Details}
\label{app:alanine_details}

Scores are estimated via isotropic Gaussian KDE with bandwidth
$h = h_{\mathrm{Scott}}/2 = 0.0091$ using 2000 centers subsampled from
the dataset, computed directly in the 45-dimensional pairwise distance
feature space. Halving Scott's rule eliminates a structured bias at
large $\tau$ that we observed under Scott's bandwidth itself: KDE
oversmoothing across the deep $\phi$ and $\psi$ barriers introduces a
small systematic component in the score that accumulates over $\tau$
and produces structured negative outliers in $\hat C^{\mathrm{sym}}(\tau)$
at long lags. The DA1 and DA2 directions are unaffected (their spike
eigenvalues sit two orders of magnitude above the bias scale), but
halving the bandwidth removes the structured tail and extends the lag
range over which the count plateau is stable.

\subsection{Robustness of the Alanine Dipeptide DA Result}
\label{app:alanine_robustness}

\paragraph{Spectrum at $\tau^*$.}
Figure~\ref{fig:alanine_bulk_validation} shows the full eigenvalue
spectrum at the chosen lag $\tau^* = 10{,}000$\,ps. The leading two
eigenvalues $\lambda_1 = 1.21 \times 10^{-2}$ and
$\lambda_2 = 6.10 \times 10^{-4}$ sit far above the analytic upper
edge $\lambda_+(\tau^*) = 1.97 \times 10^{-4}$ predicted by
Theorem~\ref{thm:limit-spectrum} (gaps $61\times$ and $3.1\times$
respectively), while the remaining $43$ eigenvalues track the
analytic free-convolution density. The bulk variance
$\sigma^2_{\mathrm{bulk}}$ is fit after one-step peeling of the two
leading spikes; the loose orange edge derived from the unpeeled
$\sigma^2_f = \tr \hat C(\tau^*)/d$ shows the same generic
spike-bias phenomenon as in the SDXL elderly experiment
(Appendix~\ref{app:sdxl_elderly}, Figure~\ref{fig:elderly_spectrum})

\clearpage

\begin{figure}[H]
\centering
\includegraphics[width=0.85\textwidth]{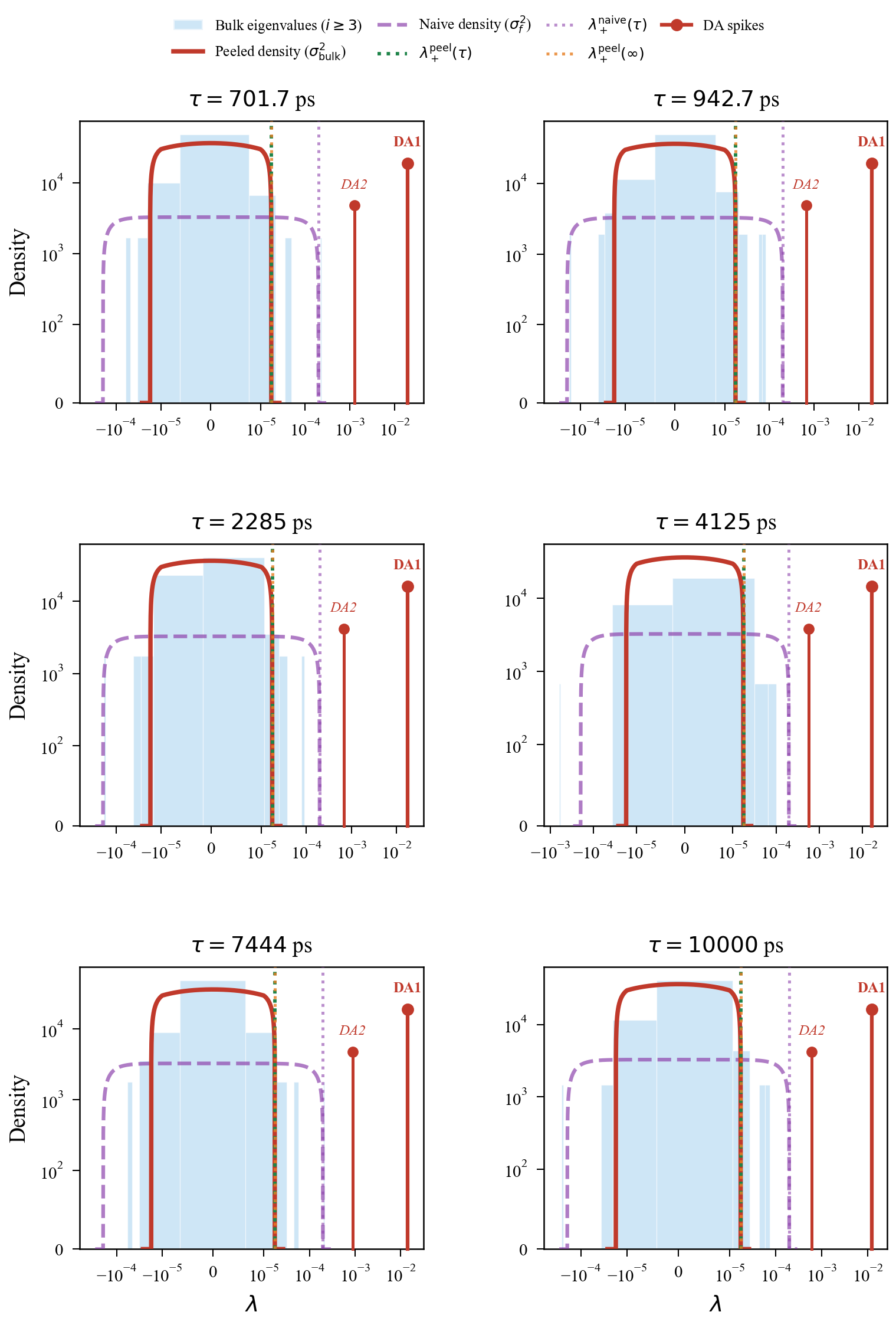}
\caption{  Eigenvalue spectrum of $\hat C^{\mathrm{sym}}(\tau^*)$
at $\tau^* = 10{,}000$\,ps for alanine dipeptide ($d = 45$, $N = 2{,}000$,
KDE bandwidth $h = h_{\mathrm{Scott}}/2 = 0.0091$). Histogram of all
$45$ eigenvalues on a symlog scale. Red curve: analytic free-convolution
density $\rho_{\tau^*}^\gamma$ from Theorem~\ref{thm:limit-spectrum}
with $\gamma = 45/N$ and $\sigma^2_{\mathrm{bulk}}$ fit on the bulk
after peeling the two leading spikes. Dashed vertical lines:
peeled-bulk edges $\lambda_\pm(\tau^*)$; orange dotted lines: full-trace
edges from the unpeeled $\sigma^2_f = \tr\hat C(0)/d$. Two eigenvalues
sit far above the bulk: $\lambda_1 = 1.21 \times 10^{-2}$ (DA1$\to\psi$)
and $\lambda_2 = 6.10 \times 10^{-4}$ (DA2$\to\phi$); both exceed the
analytic upper edge $\lambda_+(\tau^*) = 1.97 \times 10^{-4}$ by
$61\times$ and $3.1\times$ respectively. The remaining $43$ eigenvalues
fall on the analytic bulk. The analytic density is exact only under
the i.i.d.\ Gaussian null and is invoked here as an approximation via
Marchenko--Pastur universality.}
\label{fig:alanine_bulk_validation}
\end{figure}

\paragraph{Lag sensitivity.}
We evaluate DA across the long-lag plateau, sampling
$\tau \in \{700, 1{,}000, 2{,}000, 4{,}000, 7{,}000, 10{,}000\}$\,ps.
Table~\ref{tab:alanine_lag_sensitivity} reports the top two eigenvalues
$\lambda_1, \lambda_2$, the analytic noise-floor edge
$\lambda_+(\tau)$ from Theorem~\ref{thm:limit-spectrum}, the
eigenvalue gap $\lambda_1/\lambda_2$, the DA1 correlation with $\psi$,
the DA2 correlation with $\phi$, the angle between DA1 at each lag
and DA1 at $\tau^* = 10{,}000$\,ps, and whether $\lambda_2 > \lambda_+(\tau)$.
Across the full range, DA1 rotates by at most $2.1^\circ$ and
$|r(\psi)|_{\mathrm{DA1}}$ varies within $[0.670, 0.674]$; the
DA1$\leftrightarrow\psi$, DA2$\leftrightarrow\phi$ ordering holds at
every lag. Both eigenvalues exceed the analytic noise floor at every
lag in this range, with $\lambda_2/\lambda_+(\tau) \in [2.9, 6.4]$
across the plateau. The autocorrelation guard $\hat\rho(\tau) < 1/e$
is only satisfied at $\tau = 10{,}000$\,ps on this grid
(Table~\ref{tab:alanine_lag_sensitivity}); we therefore report DA
at $\tau^* = 10{,}000$\,ps as the unique grid point clearing the full
selection rule, with the remaining rows demonstrating that DA1 is
rotationally stable across the long-lag plateau independent of the
$\hat\rho$ guard.

\begin{table}[H]
\centering
\caption{  Lag sensitivity of DA on alanine dipeptide at h/2
bandwidth across the long-lag plateau. Columns: top-two eigenvalues
of $\hat C^{\mathrm{sym}}(\tau)$, analytic noise floor $\lambda_+(\tau)$
(Theorem~\ref{thm:limit-spectrum}),
DA1 correlation with $\psi$, DA2 correlation with $\phi$,
trace-normalized autocorrelation $\hat\rho(\tau)$, angular deviation
of DA1 from its direction at $\tau^* = 10{,}000$\,ps, and whether
$\lambda_2 > \lambda_+(\tau)$. The $\tau = 2{,}000$\,ps row is at
the nearest grid point ($\tau = 2{,}285.5$\,ps).}
\label{tab:alanine_lag_sensitivity}

\setlength{\tabcolsep}{4pt}
\begin{tabular}{@{}ccccccccc@{}}
\toprule
$\tau$ (ps) & $\lambda_1$ & $\lambda_2$ &
$\lambda_+(\tau)$ &
$|r(\psi)|_{\mathrm{DA1}}$ & $|r(\phi)|_{\mathrm{DA2}}$ &
$\hat\rho(\tau)$ & $\angle$(DA1, ref) & $\lambda_2 > $ floor? \\
\midrule
$701.7$    & 0.0194 & 0.00125  & 0.000197 & 0.670 & 0.854 & 0.505 & $2.1^\circ$ & Yes \\
$942.7$    & 0.0191 & 0.000655 & 0.000197 & 0.670 & 0.881 & 0.478 & $2.1^\circ$ & Yes \\
$2{,}286$  & 0.0181 & 0.000755 & 0.000197 & 0.671 & 0.878 & 0.445 & $1.6^\circ$ & Yes \\
$4{,}125$  & 0.0162 & 0.000564 & 0.000197 & 0.674 & 0.884 & 0.393 & $1.3^\circ$ & Yes \\
$7{,}444$  & 0.0142 & 0.000885 & 0.000197 & 0.670 & 0.858 & 0.369 & $1.4^\circ$ & Yes \\
$10{,}000$ & 0.0121 & 0.000610 & 0.000197 & 0.672 & 0.873 & 0.304 & $0.0^\circ$ & Yes \\
\bottomrule
\end{tabular}
\end{table}

\paragraph{Bootstrap standard errors.}
At $\tau^* = 10{,}000$\,ps, bootstrap resampling over the $750{,}000$ MD
frames ($B = 5{,}000$ replicates, DA directions held fixed) yields
standard errors at or below $0.001$ for all four Table~\ref{tab:alanine}
correlations: $|r(\psi)|_{\mathrm{DA1}} = 0.672 \pm 0.001$,
$|r(\phi)|_{\mathrm{DA1}} = 0.351 \pm 0.001$,
$|r(\psi)|_{\mathrm{DA2}} = 0.069 \pm 0.001$,
$|r(\phi)|_{\mathrm{DA2}} = 0.873 \pm 0.000$. The separation between
DA1's $\psi$ and $\phi$ correlations is $0.321$, hundreds of standard
errors wide, confirming that the DA1$\leftrightarrow\psi$ association
is statistically unambiguous.

\paragraph{The peeled vs.\ unpeeled gap is generic, not framework-specific.}
The two noise-floor lines shown in
Figure~\ref{fig:alanine_bulk_validation}
(peeled $\sigma^2_{\mathrm{bulk}}$ versus full-trace $\sigma^2_f$) differ
because the two leading spikes carry a non-negligible share of
$\tr\hat C(\tau)$, biasing the naive $\sigma^2_f = \tr\hat C(\tau)/d$
upward and producing a loose noise floor. This is a generic feature of
spike-bearing spectra under MP-style fits, not specific to the
free-convolution framework. The same gap appears at $\tau = 0$, where Theorem~\ref{thm:limit-spectrum} reduces to the classical Marchenko--Pastur law (Section~\ref{app:consequences}): the $\hat C(0)$ sample-covariance spectrum carries two leading PCA spikes that bias the naive $\sigma^2_f$ upward, and the peeled bulk fit recovers the correct MP density and edge. The free-convolution generalization at $\tau > 0$ inherits this property unchanged.

\subsection{Significance under the iso-Gaussian null}
\label{app:alanine_bonferroni}

The analytic edge $\lambda_+(\tau^*)$ of
Theorem~\ref{thm:limit-spectrum} provides a closed-form noise floor
under the iso-Gaussian null with an exact (linear) score. As
discussed in Remark~\ref{rem:floor-scope}, an operational refit of
the KDE pipeline on iso-Gaussian draws is not informative for
alanine, since the data-calibrated bandwidth is mismatched to the
ambient geometry of the null. We instead instantiate the operational
null (Algorithm~\ref{alg:operational_floor}) with the analytic
linear score $\nabla \log f(x) = -x/\sigma_f^2$ on iso-Gaussian
samples of matched covariance, running the canonical diffusion
through the full downstream eigendecomposition. This Monte Carlo null
empirically captures the $N^{2/3}$-scale Tracy--Widom edge
fluctuations around $\lambda_+(\tau^*)$ that the asymptotic theorem
does not.

For each candidate $m \in \{1, \ldots, 5\}$ we test whether the top
$m$ DA eigenvalues jointly exceed the null at Bonferroni level
$\alpha/m$ ($\alpha = 0.05$). The recovered metastable dimension is
the maximal $m$ for which DA$_1, \ldots,$ DA$_m$ all clear the
threshold $\alpha/m$. Table~\ref{tab:alanine_bonferroni} reports the
per-direction observed eigenvalue, null-max statistic, empirical
p-value over $B = 100$ replicates, and the Bonferroni threshold each
direction must meet at its own family size. The empirical null
maximum agrees with $\lambda_+(\tau^*)$ to within $9\%$, so the
analytic edge and the Monte Carlo null give identical pass/fail
verdicts here, and the surviving directions also exceed
$\lambda_+(\tau^*)$ directly.

\begin{table}[H]
\centering
\caption{  Iso-Gaussian MC null at $\tau^* = 10{,}000$\,ps for
alanine dipeptide ($B = 100$ replicates, analytic linear score).
$p_{\mathrm{emp}}$ is the empirical p-value for direction $i$;
``Pass at $\alpha/i$?'' tests whether DA$_i$ clears the Bonferroni
threshold required when included in the top-$i$ family.}
\label{tab:alanine_bonferroni}
 
\begin{tabular}{@{}cccccc@{}}
\toprule
$i$ & $\lambda_i^{\mathrm{obs}}$ & null max & $p_{\mathrm{emp}}$ & $\alpha/i$ & Pass? \\
\midrule
$1$ & $1.21 \times 10^{-2}$ & $2.13 \times 10^{-4}$ & $< 0.01$ & $0.0500$ & \checkmark \\
$2$ & $6.10 \times 10^{-4}$ & $1.94 \times 10^{-4}$ & $< 0.01$ & $0.0250$ & \checkmark \\
$3$ & $6.02 \times 10^{-5}$ & $1.79 \times 10^{-4}$ & $1.000$ & $0.0167$ & $\times$ \\
\bottomrule
\end{tabular}
\end{table}

DA1 and DA2 clear their respective Bonferroni thresholds; DA3 sits
below the null median and fails. The recovered metastable dimension
is $\hat m = 2$. Both surviving directions also reject unimodality
under Hartigan's dip test (DA1: $D = 0.046$, $p \approx 0$; DA2:
$D = 0.018$, $p \approx 0$), satisfying the two-stage criterion of
Section~\ref{sec:da}, consistent with the two known dihedral
barriers $\phi$ and $\psi$.

\subsection{TICA Lag Validation}
\label{app:tica_lag_validation}

Figure~\ref{fig:tica_its} shows implied timescale convergence and
direction sensitivity for TICA on the 45-dimensional pairwise distance
featurization. The top two implied timescales have not fully plateaued
at lag $= 10$ frames, indicating that the Markov assumption is not yet
satisfied for timescale estimation at this lag. However, the TICA
\emph{directions} at lag $= 10$ yield the strongest dihedral correlations
(TICA1 $|r(\psi)| = 0.75$; right panel), and the ordering
$\psi$-slowest, $\phi$-second is stable across lags $1$--$10$. We
report TICA at the PyEMMA default lag $= 10$ frames ($10$\,ps).

\begin{figure}[H]
\centering
\includegraphics[width=\textwidth]{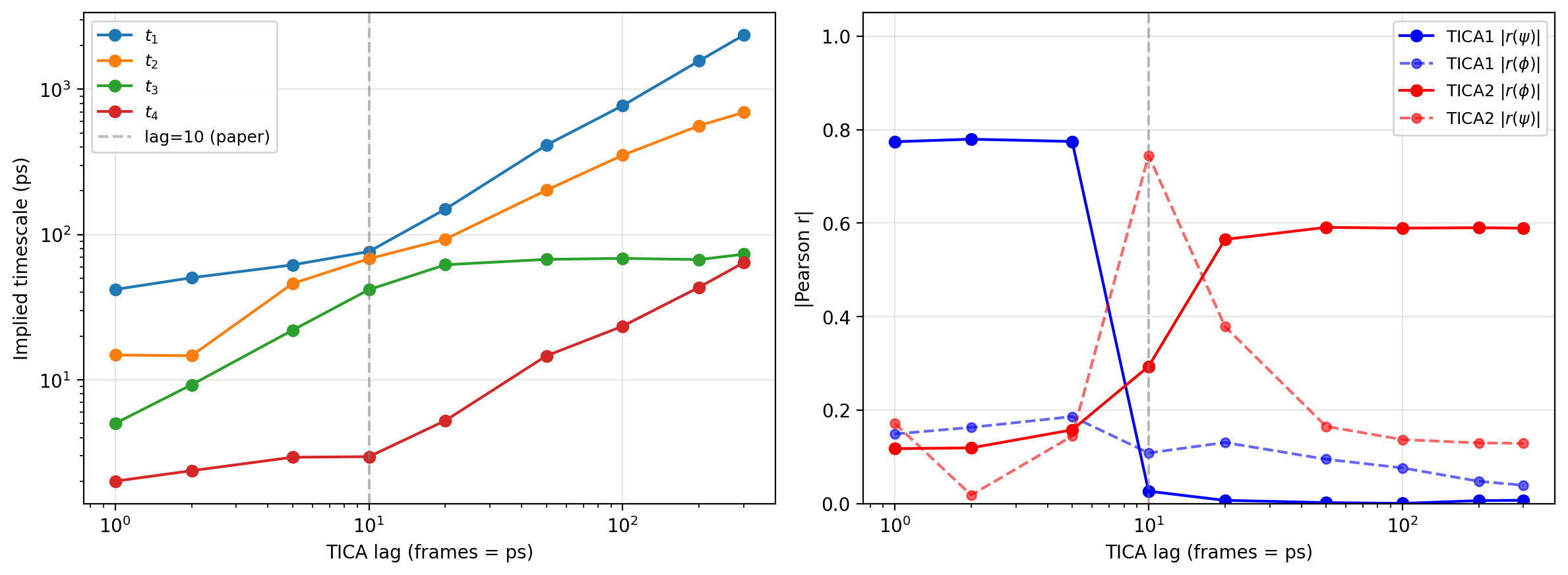}
\caption{  \textbf{Left:} Implied timescales of the top four
TICA components as a function of lag for alanine dipeptide (45D pairwise
distances). Dashed line marks lag $= 10$, used in
Table~\ref{tab:alanine}. \textbf{Right:} Pearson correlation of TICA1
and TICA2 with ground-truth dihedrals $\psi$ and $\phi$ as a function
of lag. TICA1 correlates most strongly with $\psi$ at lags $1$--$10$.}
\label{fig:tica_its}
\end{figure}

\section{SDXL Experiment Details}
\label{app:sdxl_details}

\subsection{Canonical Diffusion and CLIP-Space Autocovariance}
\label{app:sdxl_pipeline}

The canonical diffusion runs in SDXL's latent space ($4 \times 128 \times 128 = 65{,}536$ dimensions at $1024 \times 1024$ resolution). The score oracle uses the Tweedie identity at noise level $t = 1$: $\nabla \log f(x) \approx -\epsilon_\theta(x, 1) / \sigma_1$, queried once per Langevin step with prompt conditioning applied at every step. Generation and Langevin scoring use the same guidance scale $w$.

The latent space itself is not perceptually meaningful: distances and directions there do not correspond to visual attributes, and SSA or DA computed on raw latents would quantify metastability in a space whose geometry is a byproduct of training rather than a reflection of image structure. To recover a perceptually meaningful readout, each trajectory snapshot is decoded through the VAE to pixel space and embedded via CLIP ViT-L/14~\cite{radford2021learning} (\texttt{openai/clip-vit-large-patch14}, $768$ dimensions), a space known to align with human semantic judgments. The autocovariance $\hat{C}(\tau)$ used for both SSA (trace) and DA (eigendecomposition) is computed on these CLIP embeddings, not on the raw latents. Running the canonical diffusion in latent space (rather than in CLIP
space) is what makes the pipeline tractable: the score is available
directly in latent space via Tweedie's identity on the SDXL denoiser,
whereas computing a score in CLIP space would require a separate
density model. Decoding and CLIP-embedding only the saved snapshots
adds a per-snapshot VAE+CLIP forward pass; pulling embeddings into the
dynamics would multiply this cost by every Langevin step. This pipeline applies to both the guidance sweep (Section~\ref{sec:exp_sdxl}) and the DA experiment (Appendix~\ref{app:sdxl_elderly}).

\subsection{Guidance Sweep Details}
\label{app:sdxl_guidance_details}

The guidance sweep uses SDXL Base~1.0 (\texttt{stabilityai/stable-diffusion-xl-base-1.0}) with $50$ DDIM steps at $1024 \times 1024$ resolution. We test fourteen prompts (Table~\ref{tab:sdxl_guidance}) at four guidance scales $w \in \{1, 3, 7.5, 15\}$.

For entropy estimation, $2{,}000$ images per (prompt, $w$) pair are generated and embedded via CLIP ViT-L/14. Entropy is estimated using $k$NN with $k = 5$ after PCA reduction to $d = 50$ dimensions; results are consistent across $k \in \{3, 5, 10\}$ and $d \in \{2, 5, 10, 15, 20, 30, 50\}$.

For SSA, the canonical diffusion runs $N = 200$ trajectories with $T = 2{,}000$ steps, $\Delta t = 0.005$ ($T_{\mathrm{phys}} = 10$\,s), and $111$ snapshots on a combined logarithmic and fixed-interval schedule. The reported SSA values use the fixed-20 variant (100 lags at every 20th snapshot). Total compute for the SDXL guidance sweep was approximately 188 GPU-hours on a single L40S; see the compute checklist for a full breakdown.

\subsection{Convergence Diagnostic for SDXL SSA Estimates}
\label{app:sdxl_convergence}

The stopping rule of \eqref{eq:tmax_stopping}
(Appendix~\ref{app:rho_clt}) truncates the lag integral at the first
lag where $\hat\rho(\tau)^2 \le 1/N$; with $N = 200$ the threshold is
$1/N = 0.005$. All $56$ (prompt, $w$) cells in
Table~\ref{tab:sdxl_full} reach it well within the horizon of
$T = 2{,}000$ steps, at lags between $600$ and $1{,}320$ steps
(physical time $3.0$--$6.6$; median $4.0$), so every reported value
integrates the full signal. Had a cell not reached the floor within
the horizon, its truncated $\hat S$ would only lower-bound the
population value, and comparing it with other cells would require the
compute-limited fallback of Appendix~\ref{app:compute_limited}: a
common truncation horizon, approximately comparable slow-mode weights,
and a bootstrap check of the ordering. Users with slower-mixing targets
or smaller budgets should check convergence per cell before comparing
values.

\subsection{Trajectory bootstrap}
\label{app:sdxl_bootstrap}

For each (prompt, $w$) cell, we resample the $N = 200$ trajectories
with replacement and recompute $\hat S$ end-to-end on each resample
(trajectory-averaged $\hat C(\tau)$ of \eqref{eq:c_hat_def}, the $1/N$
stopping rule of \eqref{eq:tmax_stopping}, and the integral), repeating
$B = 2{,}000$ times. The reported SE is the standard deviation of the
$B$ replicates. Table~\ref{tab:sdxl_guidance_full} reproduces the body
Table~\ref{tab:sdxl_guidance} with the per-cell SE listed beneath each
$\hat S$ value. Because cells are simulated independently, we also
compare guidance scales within each prompt replicate by replicate. The
maximizer lies at an intermediate $w \in \{3, 7.5\}$ in at least
$99.8\%$ of replicates for every one of the ten prompts. The point-estimate
peak scale is recovered in at least $99\%$ of replicates for eight of
them; the exceptions are bat ($76\%$) and palm ($68\%$), whose values at
$w = 3$ and $w = 7.5$ differ by less than one SE.

\begin{table}[H]
\centering
\caption{SDXL guidance sweep with per-cell bootstrap SEs ($B = 2{,}000$); same data as body Table~\ref{tab:sdxl_guidance}. Bold: per-prompt peak.}
\label{tab:sdxl_guidance_full}
\begin{tabular}{@{}lcccccccccc@{}}
\toprule
$w$ & person & elderly & food & car & bat & crane & mouse & seal & jaguar & palm \\
\midrule
$1$   & 0.99 & 0.96 & 0.88 & 0.92 & 0.98 & 1.04 & 0.95 & 0.89 & 0.93 & 1.18 \\
      & (.03) & (.03) & (.02) & (.02) & (.03) & (.03) & (.03) & (.01) & (.03) & (.03) \\
\midrule
$3$   & 1.00 & 0.96 & 1.00 & 1.07 & \textbf{1.30} & 1.41 & \textbf{1.31} & \textbf{1.22} & \textbf{1.12} & \textbf{1.26} \\
      & (.02) & (.02) & (.01) & (.02) & (.02) & (.03) & (.02) & (.02) & (.01) & (.02) \\
\midrule
$7.5$ & \textbf{1.32} & \textbf{1.17} & \textbf{1.19} & \textbf{1.13} & 1.28 & \textbf{1.54} & 1.19 & 1.17 & 0.99 & 1.25 \\
      & (.02) & (.01) & (.01) & (.01) & (.02) & (.05) & (.01) & (.01) & (.01) & (.02) \\
\midrule
$15$  & 1.20 & 0.94 & 0.96 & 0.93 & 1.01 & 1.32 & 0.92 & 0.82 & 0.82 & 0.98 \\
      & (.01) & (.01) & (.01) & (.01) & (.01) & (.04) & (.01) & (.01) & (.01) & (.01) \\
\bottomrule
\end{tabular}
\end{table}

\subsection{Sub-threshold prompts and full entropy table}
\label{app:sdxl_subgaussian}

We swept fourteen prompts at four guidance scales each. Four
prompts (cat, mountains, flowers, street) were excluded from the body
analysis (Section~\ref{sec:exp_sdxl}, Table~\ref{tab:sdxl_guidance})
because their peak SSA across the four scales did not exceed $1.1$, a
margin above the isotropic-Gaussian baseline of $1$ that allows for
anisotropy, which alone raises SSA above $1$
(Appendices~\ref{app:ou_autocovariance},~\ref{app:rebuttal_geometry});
the body analysis is restricted to the ten prompts whose peak SSA does.
All four excluded prompts still peak at an intermediate scale (cat and
flowers at $w = 7.5$, mountains and street at $w = 3$), but their peaks
($\hat S \le 1.08$) are small enough to be compatible with anisotropic
unimodal structure. Since every cell reaches the $1/N$ stopping floor
(Appendix~\ref{app:sdxl_convergence}), a sub-threshold peak is not a
truncation artifact.

Table~\ref{tab:sdxl_full} reports SSA and $k$NN entropy for all
fourteen prompts. Entropy declines monotonically in $w$ for every
prompt, including the four sub-threshold ones, supporting the
fragmentation-vs-dispersion framing in the body: dispersion (entropy)
and fragmentation (SSA) move in unrelated directions across the
sweep, with entropy collapsing as $w$ grows while SSA peaks at
intermediate $w$.

\begin{table}[H]
\centering
\caption{  Full SSA and $k$NN entropy values for all fourteen
prompts at four guidance scales (SDXL Base~1.0, $N = 200$ trajectories
for SSA, $2{,}000$ images for entropy). Cells show $\hat S$ with
bootstrap SE in parentheses ($B = 2{,}000$). The ten body prompts
(Table~\ref{tab:sdxl_guidance}) appear above the rule; the four
excluded prompts (peak $\hat S < 1.1$) appear below. Bold marks the
per-prompt SSA peak.}
\label{tab:sdxl_full}
 
\setlength{\tabcolsep}{3pt}
\begin{tabular}{@{}lcccc|cccc@{}}
\toprule
& \multicolumn{4}{c|}{SSA $\hat S$ (SE)} & \multicolumn{4}{c}{$k$NN entropy (nats)} \\
Prompt & $w{=}1$ & $w{=}3$ & $w{=}7.5$ & $w{=}15$ & $w{=}1$ & $w{=}3$ & $w{=}7.5$ & $w{=}15$ \\
\midrule
person    & 0.99\,(.03) & 1.00\,(.02) & \textbf{1.32}\,(.02) & 1.20\,(.01) & $-69$ & $-65$  & $-74$  & $-82$  \\
elderly   & 0.96\,(.03) & 0.96\,(.02) & \textbf{1.17}\,(.01) & 0.94\,(.01) & $-71$ & $-83$  & $-105$ & $-114$ \\
food      & 0.88\,(.02) & 1.00\,(.01) & \textbf{1.19}\,(.01) & 0.96\,(.01) & $-78$ & $-82$  & $-87$  & $-88$  \\
car       & 0.92\,(.02) & 1.07\,(.02) & \textbf{1.13}\,(.01) & 0.93\,(.01) & $-76$ & $-81$  & $-86$  & $-88$  \\
bat       & 0.98\,(.03) & \textbf{1.30}\,(.02) & 1.28\,(.02) & 1.01\,(.01) & $-74$ & $-95$  & $-107$ & $-112$ \\
crane     & 1.04\,(.03) & 1.41\,(.03) & \textbf{1.54}\,(.05) & 1.32\,(.04) & $-74$ & $-84$  & $-95$  & $-101$ \\
mouse     & 0.95\,(.03) & \textbf{1.31}\,(.02) & 1.19\,(.01) & 0.92\,(.01) & $-77$ & $-106$ & $-118$ & $-121$ \\
seal      & 0.89\,(.01) & \textbf{1.22}\,(.02) & 1.17\,(.01) & 0.82\,(.01) & $-78$ & $-99$  & $-106$ & $-107$ \\
jaguar    & 0.93\,(.03) & \textbf{1.12}\,(.01) & 0.99\,(.01) & 0.82\,(.01) & $-88$ & $-113$ & $-126$ & $-134$ \\
palm      & 1.18\,(.03) & \textbf{1.26}\,(.02) & 1.25\,(.02) & 0.98\,(.01) & $-84$ & $-98$  & $-106$ & $-109$ \\
\midrule
cat       & 0.84\,(.02) & 1.06\,(.01) & \textbf{1.08}\,(.01) & 0.86\,(.01) & $-79$ & $-96$  & $-106$ & $-110$ \\
mountains & 0.96\,(.03) & \textbf{1.05}\,(.02) & 1.03\,(.01) & 0.86\,(.01) & $-81$ & $-86$  & $-95$  & $-99$  \\
flowers   & 0.79\,(.02) & 0.90\,(.01) & \textbf{1.00}\,(.01) & 0.85\,(.01) & $-85$ & $-92$  & $-97$  & $-99$  \\
street    & 1.01\,(.02) & \textbf{1.03}\,(.01) & 0.95\,(.02) & 0.75\,(.01) & $-80$ & $-101$ & $-111$ & $-113$ \\
\bottomrule
\end{tabular}
\end{table}

Full prompts (in the order above): \emph{a superrealistic professional
photograph of a person}; \emph{a portrait of an elderly person};
\emph{a photograph of a plate of food at a restaurant}; \emph{a
photograph of a car parked on a street}; \emph{a photo of a bat};
\emph{a photo of a crane}; \emph{a photo of a mouse}; \emph{a photo of
a seal}; \emph{a photo of a jaguar}; \emph{a photo of a palm}; \emph{a
photograph of a cat sitting on a windowsill}; \emph{a landscape
photograph of mountains at sunset}; \emph{a still life painting of
flowers in a vase}; \emph{a street photograph of a busy city
intersection}.

\subsection{DA on ``a portrait of an elderly person''}
\label{app:sdxl_elderly}

For ``a portrait of an elderly person'' at $w = 7.5$,
eigendecomposing the CLIP-space autocovariance at
$\tau^* = 2{,}000$ ($\Delta t = 0.0025$, $N = 200$ trajectories,
$d = 768$ CLIP dimensions, $\gamma = d/N = 3.84$) yields a leading
direction substantially distinct from PC1
($|\cos(\mathrm{DA1}, \mathrm{PC1})| = 0.514$, an angle of
$59^\circ$): DA1 correlates strongly with gender (Spearman
$|\rho| = 0.90$, 95\% CI $[0.78, 0.94]$ from a bootstrap over the
$N = 200$ trajectories with $B = 2{,}000$ replicates; $|\cos|$ to a
gender axis $= 0.56$), while PC1 mixes gender with realism and
style. Figure~\ref{fig:sdxl_combined} (right, in the body) shows
the top-5 images sorted along each direction; DA1 separates cleanly
by gender while PC1 mixes the two.
Figure~\ref{fig:elderly_da_pc1_projection} shows the joint
projection of all $2{,}000$ snapshots onto the (DA1, PC1) plane,
making the non-collinearity directly visible.

\begin{figure}[H]
\centering
\includegraphics[width=0.7\textwidth]{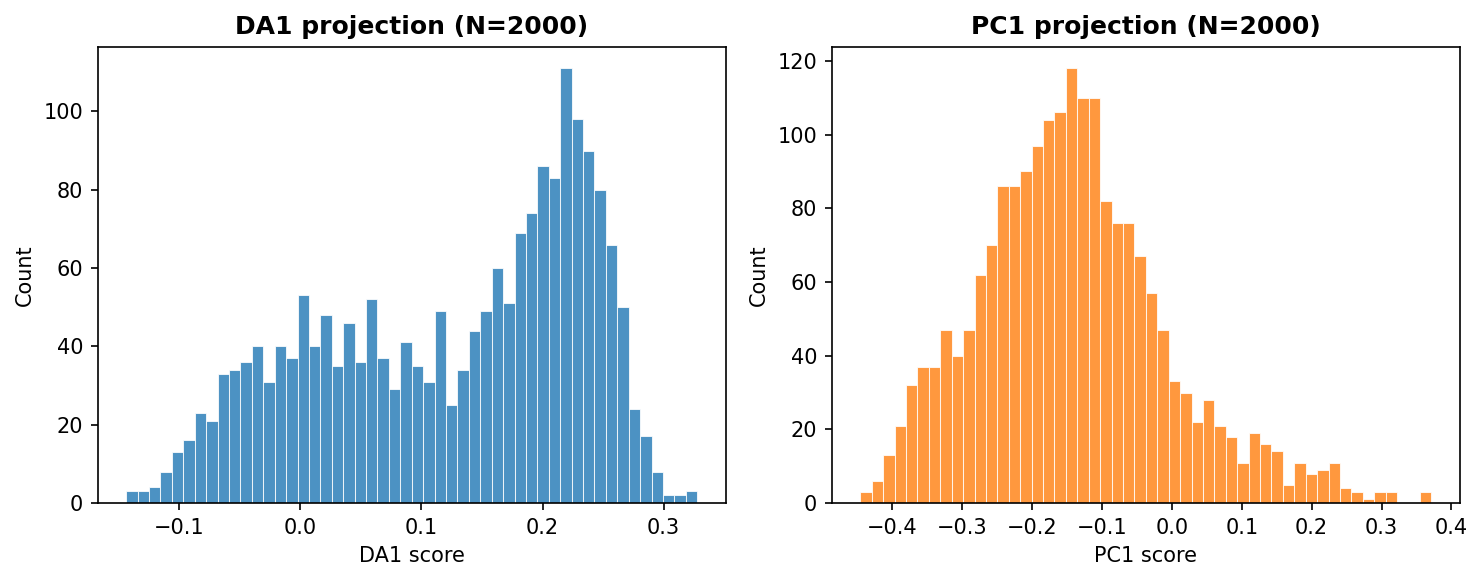}
\caption{  Joint projection of all $N=2{,}000$ trajectory snapshots onto the (DA1, PC1) plane for the ``elderly'' prompt at $w = 7.5$. The two axes are non-collinear ($|\cos(\mathrm{DA1}, \mathrm{PC1})| = 0.514$, an angle of $59^\circ$; Spearman $|\rho|$ with gender $= 0.90$): DA1 separates the distribution along the gender axis while PC1 mixes gender with realism and style.}
\label{fig:elderly_da_pc1_projection}
\end{figure}

The lag $\tau^* = 2{,}000$ is selected as the largest lag at which
the leading eigenvalue $\lambda_1 = 9.96 \times 10^{-3}$ exceeds the
analytic noise floor of Theorem~\ref{thm:limit-spectrum},
following the multi-direction selection rule of
Appendix~\ref{app:lag_selection}. Significance is verified through
two complementary checks shown in
Figure~\ref{fig:elderly_spectrum}. Panel~(a) overlays the analytic
free-convolution density at the naive variance estimate
$\sigma_f^2 = \tr \hat C(\tau^*)/d = 1.12 \times 10^{-4}$, biased
upward by the leading spike, which produces a loose bulk fit and a
loose edge $\lambda_+(\tau^*) \approx 5 \times 10^{-4}$. Panel~(b)
shows the same overlay after one-step peeling of the leading
eigenvalue, yielding $\sigma^2_{\mathrm{bulk}} = 2.29 \times 10^{-5}$
and a tight match between the analytic density and the empirical
bulk; the peeled edge $\lambda_+(\tau^*)$ is roughly two orders of
magnitude below $\lambda_1$. The empirical bulk is centered near
zero with a $\delta_0$ atom of mass $0.479$, the expected mass for
$\gamma = 3.84$ (Theorem~\ref{thm:limit-spectrum}). Panel~(c) shows
an independent Monte-Carlo null matched to CLIP-space covariance
(isotropic Gaussian samples, same $N$, $d$, $\tau^*$, same VAE+CLIP
pipeline): the null distribution's $95$th percentile sits at
$4.11 \times 10^{-3}$, a factor of $2.4\times$ below the observed
$\lambda_1 = 9.96 \times 10^{-3}$. The two checks agree on
significance through different routes, with the Monte-Carlo null
also covering the universality concern that
Theorem~\ref{thm:limit-spectrum} is proved under an i.i.d.\ Gaussian
null while the data passes through a VAE-decoder and CLIP encoder
(Appendix~\ref{app:scope}).
\clearpage

\begin{figure}[H]
\centering
\includegraphics[width=\textwidth]{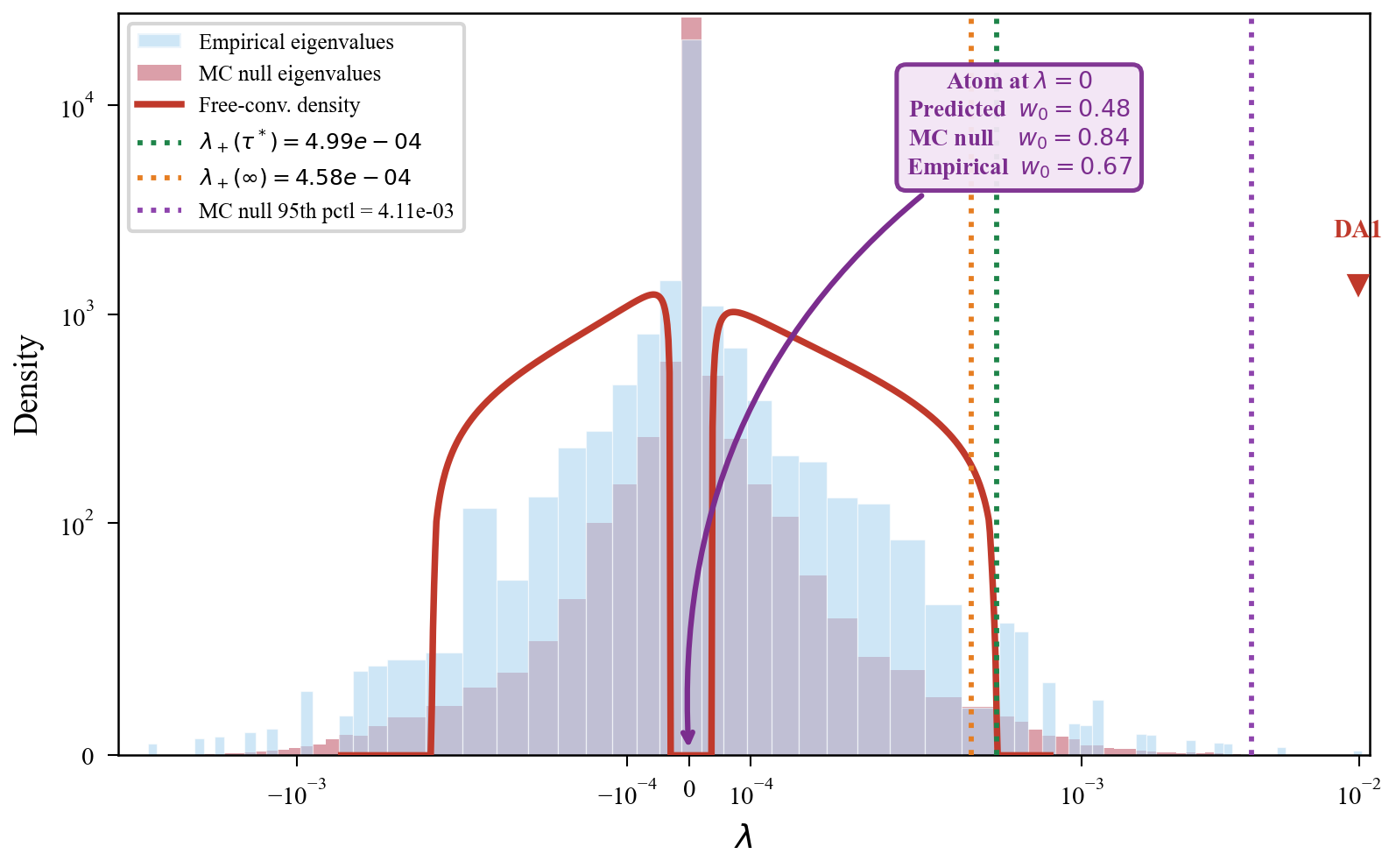}
\caption{  Eigenvalue spectrum of $\hat C^{\mathrm{sym}}(\tau^*)$
for ``a portrait of an elderly person'' at $w = 7.5$, $\tau^* = 2{,}000$
($d = 768$, $N = 200$, $\gamma = 3.84$, $\Delta t = 0.0025$).
\textbf{(a)}~Empirical histogram (light blue) overlaid with the
analytic free-convolution density (orange) at the naive variance
$\sigma_f^2 = 1.12 \times 10^{-4}$; the leading spike at
$\lambda_1 = 9.96 \times 10^{-3}$ (red vertical) sits well above
the loose edge $\lambda_+(\sigma_f^2)$ (orange dashed).
\textbf{(b)}~Same histogram with the analytic density refit at
$\sigma^2_{\mathrm{bulk}} = 2.29 \times 10^{-5}$ after one-step
peeling of the leading eigenvalue; the tight bulk match shows that
peeling recovers the correct noise-floor scale, and $\lambda_1$
sits roughly two orders of magnitude above the peeled edge (red
dashed). The $\delta_0$ atom (mass $0.479$) is the expected
zero-eigenvalue mass at $\gamma = 3.84$.
\textbf{(c)}~Monte-Carlo null distribution of $\lambda_1$
(empirical histogram) generated by simulating the canonical
diffusion on isotropic Gaussian latents with covariance matched to
CLIP space and pushing through the same VAE+CLIP pipeline; the
$95$th percentile (green dashed) at $4.11 \times 10^{-3}$ sits
$2.4\times$ below the observed $\lambda_1$ (red).}
\label{fig:elderly_spectrum}
\end{figure}

\paragraph{Why a finer $\Delta t$ is needed for DA than for SSA.}
The SSA guidance sweep of Section~\ref{sec:exp_sdxl} uses
$\Delta t = 0.005$, chosen so that most (prompt, $w$) cells reach
convergence at $T = 2000$ within compute budget. The DA result
above instead uses $\Delta t = 0.0025$, twice as fine. At
$\Delta t = 0.005$ the integration steps are large enough to
smooth over the intra-distribution barriers that fragment the
elderly prompt, and DA1 collapses toward PC1. SSA is robust to
this because it integrates the trace of the autocovariance, which
under-resolves slow modes at large $\Delta t$ but, for the
fixed-spectral-profile comparisons of Proposition~\ref{prop:ssa_monotonicity}, does not invert relative rankings; DA, by contrast,
identifies a specific direction and is correspondingly more
sensitive to whether the slow mode is resolved at all. We do not
run the full guidance sweep at $\Delta t = 0.0025$ because the
compute cost is prohibitive.

\paragraph{Two-stage significance: noise floor and Silverman bimodality.}
We instantiate the operational null
(Algorithm~\ref{alg:operational_floor}) using the score-oracle
construction of Section~\ref{sec:score_high_dim}: starting from
$x_0 \sim \mathcal{N}(0, \sigma^2 I)$ in latent space, we run the
canonical Langevin using SDXL's own UNet score, but evaluated at
$t_{\mathrm{null}} = 999$ (rather than $t = 1$ for the actual
data). At $t_{\mathrm{null}} = 999$ the training marginal $f_t$ is
essentially $\mathcal{N}(0, \sigma_t^2 I)$ by construction, so the
denoiser learns to approximate the linear iso-Gaussian score
$-x/\sigma_t^2$. The model has no prompt-conditioned metastable
structure at this noise level, so the dynamics stay iso-Gaussian
throughout, while the rest of the pipeline (VAE decode, CLIP
embed, $\hat C(\tau^*)$ eigendecomposition) is identical to the
real run. Across $B_{\mathrm{boot}} = 10{,}000$ replicates with
$N = 200$ trajectories per replicate, every one of the top 10
eigenvalues sits $7.5\text{--}13.5\times$ above the corresponding
null maximum, with empirical p-value at the resolution ceiling
$1 \times 10^{-4}$:

\begin{table}[H]
\centering
\caption{  Iso-Gaussian null via large-$t$ score evaluation
for the elderly prompt at $w = 7.5$, $\tau^* = 2{,}000$ ($d = 768$,
$N = 200$, $B_{\mathrm{boot}} = 10{,}000$). $p_{\mathrm{emp}}$ is
the empirical p-value; Bonferroni threshold is
$\alpha/10 = 5 \times 10^{-3}$.}
\label{tab:elderly_bonferroni}
 
\begin{tabular}{@{}ccccc@{}}
\toprule
$i$ & $\lambda_i^{\mathrm{obs}}$ & null max $|\lambda|$ & $\lambda_i^{\mathrm{obs}}$ / null max & Pass $\alpha/10$? \\
\midrule
$1$  & $9.96 \times 10^{-3}$ & $1.23 \times 10^{-3}$ & $8.08\times$  & \checkmark \\
$2$  & $5.12 \times 10^{-3}$ & $6.83 \times 10^{-4}$ & $7.50\times$  & \checkmark \\
$3$  & $3.47 \times 10^{-3}$ & $4.36 \times 10^{-4}$ & $7.95\times$  & \checkmark \\
$4$  & $3.44 \times 10^{-3}$ & $3.06 \times 10^{-4}$ & $11.23\times$ & \checkmark \\
$5$  & $3.11 \times 10^{-3}$ & $2.45 \times 10^{-4}$ & $12.72\times$ & \checkmark \\
$6$  & $2.45 \times 10^{-3}$ & $2.13 \times 10^{-4}$ & $11.49\times$ & \checkmark \\
$7$  & $2.20 \times 10^{-3}$ & $1.70 \times 10^{-4}$ & $12.96\times$ & \checkmark \\
$8$  & $1.94 \times 10^{-3}$ & $1.48 \times 10^{-4}$ & $13.10\times$ & \checkmark \\
$9$  & $1.78 \times 10^{-3}$ & $1.36 \times 10^{-4}$ & $13.08\times$ & \checkmark \\
$10$ & $1.61 \times 10^{-3}$ & $1.20 \times 10^{-4}$ & $13.48\times$ & \checkmark \\
\bottomrule
\end{tabular}
\end{table}

The noise-floor stage thus passes 10 of the top 10 directions, more
than the bimodality stage will retain. Applying Silverman's
bandwidth test~\cite{silverman1981} to the 1D projections of the
data onto each direction at Bonferroni $\alpha/3 = 0.0167$ gives:
DA1 $p < 0.002$ (reject unimodality), DA2 $p = 0.77$ (do not
reject), PC1 not bimodal. Hartigan's dip test fails to reject on
DA1 ($D$ small, $p = 0.234$), a known limitation on asymmetric
bimodality (Appendix~\ref{app:validation}). Combining the two
stages, only DA1 satisfies both, giving $m = 1$: the gender axis.

\subsection{DA on ``a photo of a crane''}
\label{app:sdxl_crane}

The ``crane'' prompt at $w = 7.5$ has the highest SSA in the sweep
($\hat S = 1.54$, Table~\ref{tab:sdxl_guidance}). The cause is
lexical: SDXL produces both senses of the word, the bird and the
construction equipment, in roughly comparable proportions, so the
output distribution fragments into two well-separated clusters of
comparable variance. Eigendecomposing the CLIP-space autocovariance
at $\tau^* = 2{,}000$ ($\Delta t = 0.0025$, $N = 200$) yields a
leading direction $\mathrm{DA1}$ that resolves this bird-vs-equipment
split: sorting trajectory snapshots along $\mathrm{DA1}$ places all
construction-equipment images at one extreme and all bird images at
the other (Figure~\ref{fig:crane_da_pc1}, top). $\mathrm{PC1}$
recovers the same split (Figure~\ref{fig:crane_da_pc1}, bottom).
Both pass Hartigan's dip test on their 1D projections at
$p < 10^{-3}$. This is the contrast case to the elderly result
(Appendix~\ref{app:sdxl_elderly}): when the dominant axis of
variation is also the dominant fragmentation axis, $\mathrm{DA1}$
and $\mathrm{PC1}$ agree; when fragmentation is orthogonal to the
dominant variance direction, as on the elderly prompt, they diverge.

\clearpage

\begin{figure}[H]
\centering
\includegraphics[width=\textwidth]{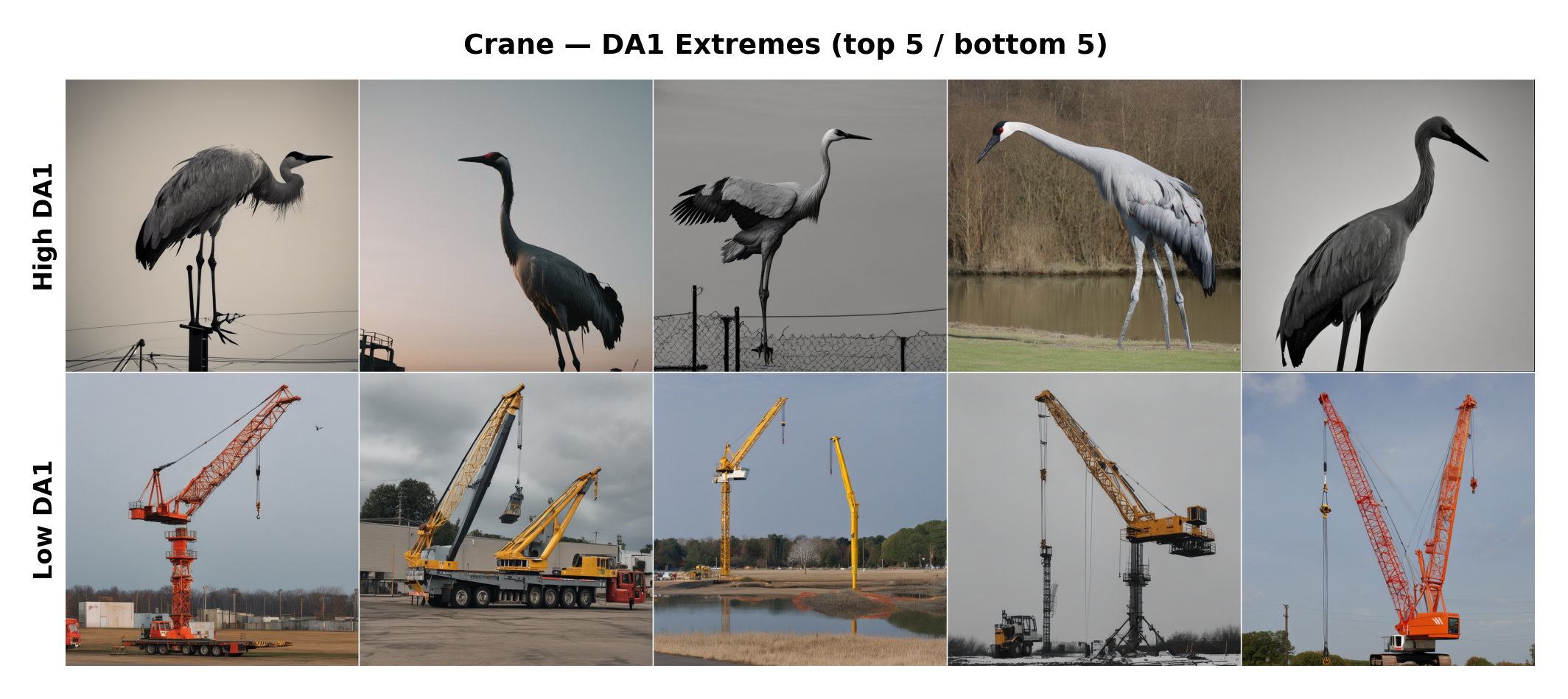}\\[0.5em]
\includegraphics[width=\textwidth]{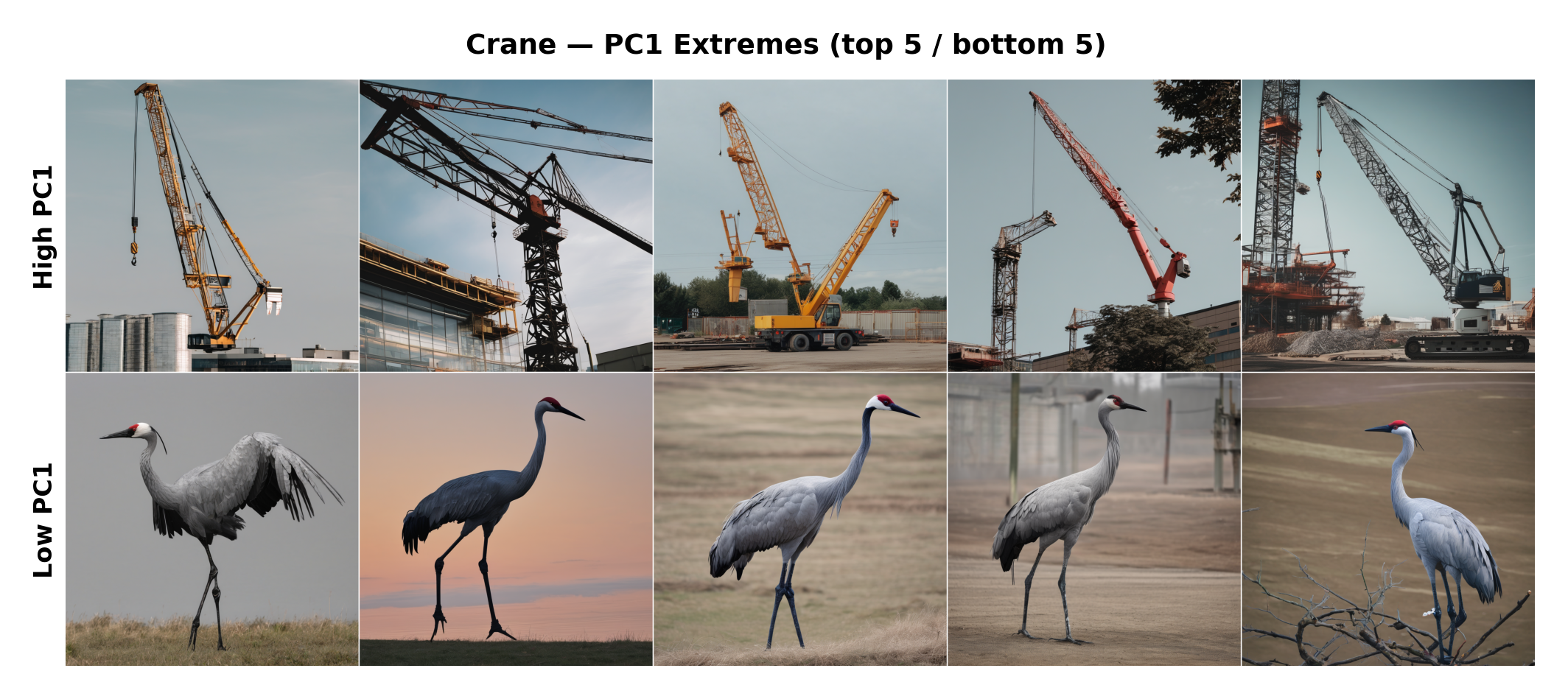}
\caption{  Top-5 images at each end of $\mathrm{DA1}$ (top panel,
top row = high DA1, bottom row = low DA1) and $\mathrm{PC1}$ (bottom
panel, top row = high PC1 = construction cranes, bottom row = low
PC1 = birds) for ``a photo of a crane'' at $w = 7.5$, $\tau^* = 2{,}000$.
Both axes separate the two senses of the prompt: birds at one extreme,
construction equipment at the other. Both pass Hartigan's dip test
($p < 10^{-3}$).}
\label{fig:crane_da_pc1}
\end{figure}

\section{Reproducibility and Hyperparameters}
\label{app:hyperparameters}

SSA and DA share three computational parameters: the time step $\Delta t$, the lag horizon $L$ (equivalently the physical time $T = L \cdot \Delta t$), and the number of trajectories $N$. The discussion below focuses on SSA, where these parameters most directly control the estimator; the DA-specific lag-selection machinery is in Appendix~\ref{app:lag_selection}, and the joint compute-limited regime in Appendix~\ref{app:compute_limited}. Table~\ref{tab:hp_summary} summarizes the role, failure mode, and safeguard for each parameter.

\begin{table}[H]
\centering
\caption{  Summary of SSA hyperparameter sensitivity and safeguards.}
\label{tab:hp_summary}
 
\setlength{\tabcolsep}{4pt}
\begin{tabular}{@{}p{0.05\textwidth}p{0.20\textwidth}p{0.28\textwidth}p{0.37\textwidth}@{}}
\toprule
Param. & Controls & Failure mode & Safeguard \\
\midrule
$\Delta t$ & Discretization fidelity & Biased $\hat\rho(\ell)$, possible ranking distortion &
Richardson ($\Delta t$ vs.\ $\Delta t/2$); verified in Appendix~\ref{app:dt_robustness} \\
$T$ & Dynamic range of score & Reduced discriminatory power &
Proposition~\ref{prop:ssa_monotonicity}: fixed-profile ordering is $T$-stable; otherwise check lag-window stability \\
$N$ & Estimator variance & Ambiguous rankings at small $N$ &
Bootstrap over trajectories (Appendix~\ref{app:sdxl_bootstrap}) \\
\bottomrule
\end{tabular}
\end{table}

The time step $\Delta t$ controls the fidelity of the Euler--Maruyama discretization to the continuous canonical diffusion. If $\Delta t$ is too large, discretization error can corrupt the autocorrelation estimates differently for different distributions, potentially affecting rankings. We recommend a Richardson-type validation: run a short pilot at $\Delta t$ and $\Delta t/2$ and verify that $\hat{\rho}(\ell)$ agrees at matched physical times. We performed this validation for the GMM experiments (Appendix~\ref{app:dt_robustness}). Once this check passes, $\Delta t$ can be fixed for all subsequent comparisons.

 \textbf{T-stability of fixed-profile rankings (population).} For comparisons of the fixed-spectral-profile kind covered by Proposition~\ref{prop:ssa_monotonicity} (one density's spectral profile dominates another's term-by-term), $S_A(T) > S_B(T)$ for every $T > 0$, and increasing $T$ can only widen the gap. For arbitrary pairs of distributions whose spectral weights and eigenvalues differ in different ways, $T$ is a resolution parameter and rankings should be checked for stability over a lag window. At the estimator level, finite-sample variance in $\hat{S}$ is controlled by bootstrap (e.g., Appendix~\ref{app:sdxl_bootstrap}). At small $T$, the spectral factor $(1 - e^{-2\mu_k T})/(2\mu_k) \approx T$ for all $k$, so all eigenvalue contributions collapse and the score loses discriminatory power. As $T$ grows, slow modes separate from fast ones and the ranking becomes easier to resolve empirically. However, at very large lags the estimator $\hat{\rho}(\ell)$ is computed from fewer overlapping windows, increasing its variance. The practical recommendation is to increase $L$ while monitoring the pairwise rankings of the distributions under comparison: once rankings have stabilized over a window of lags, further increases in $T$ improve dynamic range but do not change the outcome.

The trajectory count $N$ controls the statistical precision of $\hat{S}$. Each trajectory is an independent realization of the canonical diffusion, so the variance of the estimator scales as $O(1/N)$. When two distributions have similar true SSA values, the empirical rankings may fluctuate at small $N$ even though the population ordering is fixed. This is a statistical power issue, not a bias issue: the expectation of $\hat{S}$ converges to $S(T)$ for any $N$, but resolving small differences requires large $N$. Bootstrap resampling over trajectories provides confidence intervals on rank differences when needed.

In summary, $\Delta t$ should be validated once and fixed; $T$ should be increased until rankings stabilize over a lag window; and $N$ should be increased if rankings remain ambiguous at stable $T$. For fixed-spectral-profile comparisons (Proposition~\ref{prop:ssa_monotonicity}) the procedure cannot produce false inversions; for arbitrary distributions, stability over a lag window is the operational criterion.

\subsection{Compute-limited and slowly mixing targets}
\label{app:compute_limited}

Both readouts use noise-floor stopping rules: SSA truncates the lag integral at the first $\tau$ for which $\hat\rho(\tau)^2$ falls below the trace-CLT floor (Equation~\ref{eq:tmax_stopping}, derived in Appendix~\ref{app:rho_clt}); DA reads $\tau^*$ off the largest $\tau$ at which the top eigenvalues of $\hat C(\tau)$ exceed the free-convolution edge $\lambda_+(\tau)$ from Theorem~\ref{thm:limit-spectrum} (Appendix~\ref{app:lag_selection}). On slowly mixing targets or under tight compute budgets, the principled stopping lag may not be reachable. The fallback differs between the two readouts.

\paragraph{SSA: fixed-profile rankings are preserved; otherwise verify stability.}
When the principled $T_{\max}$ is unreachable, truncate at the largest feasible horizon, common across the distributions being compared. By Proposition~\ref{prop:ssa_monotonicity}, in fixed-spectral-profile comparisons $S(T_{\max})$ is monotone in every $\mu_k$ at every $T_{\max}$: truncation underestimates absolute values but preserves rankings, and faster-mixing competitors only contribute near-zero past their own stopping lag, so the gap between scores can only widen as $T_{\max}$ grows. For arbitrary pairs of distributions whose spectral weights and eigenvalues differ in different ways, this invariance is no longer automatic and rankings should be checked for stability over a lag window. To certify the ordering at finite samples, bootstrap over trajectories or lags and report $\Pr[\hat S_A > \hat S_B]$, as in the SDXL sweep (Appendix~\ref{app:sdxl_bootstrap}). 

\paragraph{DA: criterion failure means no $\tau^*(m)$.}
DA's fallback is built into the selection rule itself
(Appendix~\ref{app:lag_selection}, Section~\ref{sec:da}): if no
$\tau$ jointly satisfies (i) top-$m$ eigenvalue exceedance of the
iso-Gaussian null threshold, (ii) 1D bimodality (dip or Silverman),
and (iii) $\rho(\tau) < 1/e$, then $\tau^*(m)$ does not exist and
$m$ is not recovered. The $\rho(\tau) < 1/e$ guard rules out short
lags where $\hat C(\tau) \approx \hat C(0)$ and DA collapses
toward PCA. When the criterion fails, raising $N$ lowers the noise
floor as $\sigma^2\sqrt{d/N}$ and may extend the feasible window;
otherwise the recovered $\hat m$ is honest at the available budget.

\paragraph{Common thread.}
In both cases, faster-mixing or smaller-barrier competitors are the easy case: their signal saturates earlier in $\tau$ and any reasonable common budget resolves them. A shared feasible horizon thus cannot manufacture false orderings, only fail to resolve close ones — the genuinely hard case is two slow targets whose spectra differ only at lags beyond reach.

\subsection{$\Delta t$ Convergence: Empirical Verification}
\label{app:dt_robustness}

We verify the Richardson recommendation of
Appendix~\ref{app:hyperparameters} by re-running all four 10D GMM
sweeps from Figure~\ref{fig:gmm_sweeps} at three step sizes
$\Delta t \in \{0.005, 0.01, 0.02\}$ (a $4\times$ range), holding
physical time $T_{\mathrm{phys}} = 100$\,s and all other parameters
fixed. Table~\ref{tab:dt_robustness} reports the Spearman rank
correlation between $\hat S$ profiles at each pair of step sizes.
All correlations exceed $0.97$, the mode-count sweep is rank-preserved
exactly at every pair, and no ranking inversions occur in any sweep.
The same validation procedure can be applied to any new target.

\begin{table}[H]
\centering
\caption{  Spearman rank correlation of $\hat S$ profiles across
step sizes for all four 10D GMM sweeps. All correlations $\ge 0.97$;
mode-count is rank-preserved exactly across the full $4\times$ range
of $\Delta t$.}
\label{tab:dt_robustness}
 
\begin{tabular}{@{}lccc@{}}
\toprule
Sweep & $\rho(0.01, 0.005)$ & $\rho(0.01, 0.02)$ & $\rho(0.005, 0.02)$ \\
\midrule
Separation & $1.000$ & $1.000$ & $0.999$ \\
Variance   & $0.991$ & $0.996$ & $0.996$ \\
Weight     & $0.982$ & $0.990$ & $0.971$ \\
Mode count & $1.000$ & $1.000$ & $1.000$ \\
\bottomrule
\end{tabular}
\end{table}
\clearpage

\section{Compute Resources Breakdown}
\label{app:compute}

Table~\ref{tab:compute_breakdown} reports the compute used for each
experiment in this paper. All jobs are single-GPU (one NVIDIA L40S
48GB or L4 24GB per job) on an internal academic cluster, with no
multi-GPU or multi-node training. No model training was performed;
all GPU compute consists of forward-pass inference through pretrained
models and downstream Langevin simulation and analysis.

\begin{table}[h]
\centering
\small
\caption{Compute per reported experiment. Total: $\sim$1{,}100 GPU-hours.}
\label{tab:compute_breakdown}
\begin{tabular}{lllrrr}
\toprule
Experiment & Artifact & Hardware & Jobs & h/job & GPU-h \\
\midrule
SDXL guidance Langevin     & Fig~3, Table~1     & L40S & 28 & 19  & 532 \\
SDXL guidance generation   & Fig~3, Table~1     & L40S & 28 & 7   & 196 \\
SDXL traj-avg recompute    & Table~1            & L40S & 24 & 8   & 192 \\
SDXL MC null               & Table~4            & L40S & 3  & 19  & 57  \\
SDXL bootstrap CIs         & Table~1            & L40S & 59 & 0.5 & 30  \\
SDXL elderly DA            & Fig~4, A8          & L40S & 1  & 24  & 24  \\
SDXL $C(\tau)$ full        & Fig~A7             & L40S & 1  & 20  & 20  \\
SDXL decode grids          & Image grids        & L40S & 28 & 0.5 & 14  \\
dt robustness              & Fig~A3             & L4   & 1  & 12  & 12  \\
GMM sweeps                 & Fig~1, A1          & L4   & 1  & 8   & 8   \\
Alanine KDE-DA             & Fig~5              & L40S & 1  & 6   & 6   \\
Random sweep               & Fig~A2             & L4   & 1  & 4   & 4   \\
Alanine $C(\tau)$          & Fig~5, A4--A6, T2  & L4   & 1  & 3   & 3   \\
Planted spikes             & Fig~A15            & L4   & 1  & 2   & 2   \\
Alanine MC null            & Table~4            & L4   & 10 & 0.1 & 1   \\
Decode crane               & Crane fig          & L40S & 1  & 0.5 & 0.5 \\
Weights                    & App~H              & L4   & 1  & 0.5 & 0.5 \\
Baselines                  & Table~3            & CPU  & 1  & 0.5 & 0.5 \\
On-the-fly figures         & Fig~2a-b, A0, A9--A14 & CPU & --- & --- & $<1$ \\
\midrule
\textbf{Total reported} & & & & & \textbf{$\sim$1{,}103} \\
\bottomrule
\end{tabular}
\end{table}

\subsection{Licenses}

\begin{table}[H]
\centering
\caption{Licenses for all assets used in this paper.}
\label{tab:licenses}
\begin{threeparttable}
\begin{tabular}{lll}
\toprule
\textbf{Asset} & \textbf{License} & \textbf{Version} \\
\midrule
SDXL-Base\tnote{1}             & CreativeML Open RAIL++-M & 1.0 \\
CLIP ViT-L/14\tnote{2}         & MIT                      & --  \\
HuggingFace Diffusers\tnote{4} & Apache 2.0               & --  \\
PyTorch\tnote{3}               & BSD-3                    & --  \\
WandB\tnote{4}                 & MIT                      & --  \\
MDShare\tnote{5}               & CC BY 4.0                & --  \\
PyEMMA\tnote{6}                & LGPL-3.0                 & 2.x \\
\bottomrule
\end{tabular}
\begin{tablenotes}
\footnotesize
\item[1] \url{https://huggingface.co/stabilityai/stable-diffusion-xl-base-1.0}
\item[2] \url{https://github.com/openai/CLIP} (MIT inferred from repo; no license specified on model page)
\item[3] \url{https://github.com/huggingface/diffusers}
\item[4] \url{https://github.com/pytorch/pytorch}
\item[5] \url{https://github.com/wandb/wandb}
\item[6] \url{https://markovmodel.github.io/mdshare}
\item[7] \url{https://github.com/markovmodel/PyEMMA}
\end{tablenotes}
\end{threeparttable}
\end{table}

\section{Additional Robustness and Scope Controls}
\label{app:rebuttal_robustness}

This appendix consolidates the controls requested during review. All values come
from full machine-readable runs; smoke outputs are not used as evidence.

\subsection{Sensitivity to score-estimation error}
\label{app:rebuttal_score_error}

We injected a frozen, deterministic, state-dependent random-Fourier error field
$\varepsilon(x)$ into the exact score $s(x)$, such that  $\widehat s(x)=s(x)+\alpha\varepsilon(x)$,
normalizing its per-coordinate RMS to the exact-score RMS. The field is fixed
across levels, so only its amplitude $\alpha$ varies.

Across 50 random 10D GMMs and three seeds, the SSA--MI Spearman correlation fell
from $0.877\pm0.006$ at $\alpha=0$ to $0.872\pm0.008$, $0.874\pm0.009$,
$0.865\pm0.007$, $0.832\pm0.008$ and $0.741\pm0.010$ at
$\alpha=0.1,0.25,0.5,1,2$ (Figure~\ref{fig:rebuttal_score_error}a). The median
absolute change in finite-horizon SSA over the same levels was $0.2\%$, $0.6\%$,
$1.7\%$, $5.2\%$ and $12.3\%$ (90th percentiles $2.3\%$, $3.6\%$, $6.9\%$,
$13.2\%$, $29.0\%$). Value error and ranking loss therefore degrade together.

For DA (Figure~\ref{fig:rebuttal_score_error}b) each mixture is compared against
its own unperturbed run, so the deviation at $\alpha=0$ is exactly zero and the
curve isolates the effect of the perturbation. On the five representative mixtures DA$_1$ rotates by a median of
$0.23^\circ$, $0.56^\circ$, $1.06^\circ$, $1.70^\circ$ and $2.80^\circ$
(worst case $3.14^\circ$ at $\alpha=1$). These synthetic perturbations map a
controlled failure envelope; they do not calibrate the unknown error of the
SDXL denoiser.

\begin{figure}[H]
\centering
\includegraphics[width=0.92\textwidth]{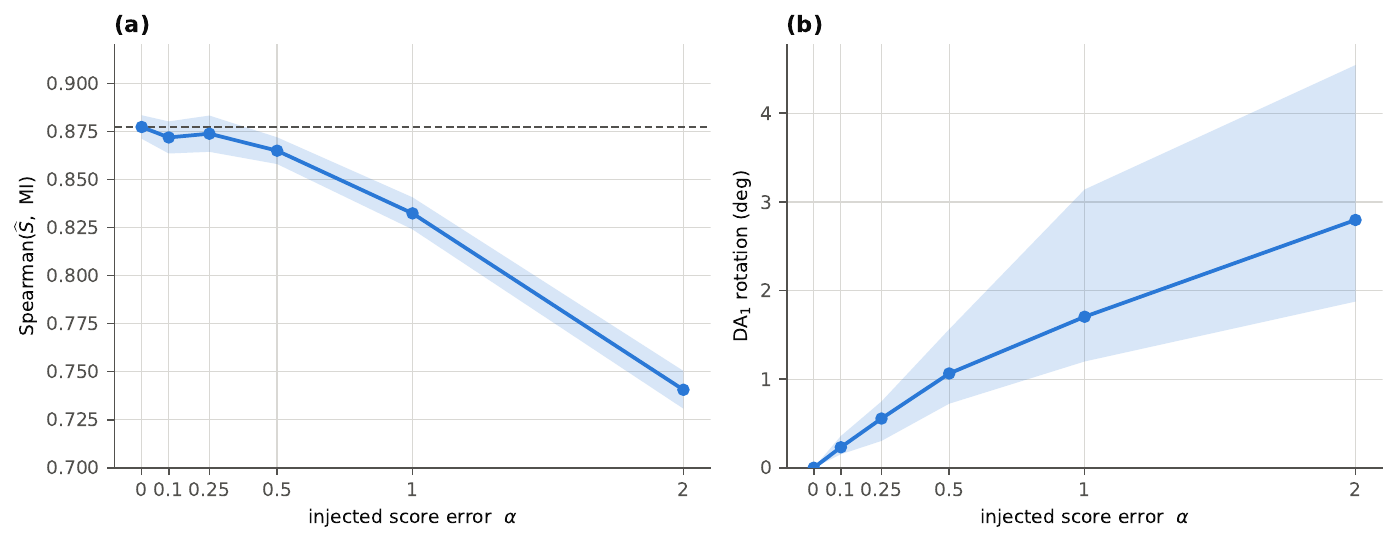}
\caption{\textbf{Effect of score error on SSA and DA.} 50 random 10D GMMs,
$N=200$, three seeds. (a) SSA--MI rank correlation against injected error
$\alpha$ (in units of the baseline per-coordinate score RMS); band is $\pm1$ SD
over seeds, dashed line the exact-score value. (b) DA$_1$ rotation away from
each mixture's own unperturbed direction, median over the five representative
mixtures with a min--max band.}
\label{fig:rebuttal_score_error}
\end{figure}

\subsection{Numerical parameters}
\label{app:rebuttal_numerics}

The numerical parameters play distinct roles: $\Delta t$ controls discretization
bias, $T_{\max}=T\Delta t$ the relaxation horizon, and $N$ the precision, while
also setting the $1/N$ stopping rule. On a seven-rung 10D bimodal ladder
(Figure~\ref{fig:rebuttal_numerics}), every configuration orders all rungs with
separation $\geq 2$ correctly (seed-averaged $\widehat S$ strictly increasing);
below that there is no substantial barrier, every value sits at the unimodal
baseline $\widehat S\approx 1$, and their relative order is noise. Increasing
$N$ from $50$ to $400$ tightens the spread and raises Spearman agreement with MI
from $0.757$ to $0.943$, with the representative SSA staying in
$[1.0583,1.0798]$. A fourfold change of $\Delta t\in\{0.005,0.01,0.02\}$ at
matched physical time moves the representative value only from $1.0676$ to
$1.0598$ to $1.0502$.

Truncation matters more than either: capping the lag window at $15$, $10$ or $5$
instead of $20$ changes values by up to $1.7\%$, $6.6\%$ or $18.6\%$, though the
benchmark ordering survives. On the submitted Gaussian-versus-parallel-ellipses
comparison, every cell of a five-seed grid over three step sizes, four trajectory
counts, four lag windows and both stopping conventions retained
$S_{\rm ellipse}>S_{\rm Gaussian}$, with ratios in $[1.77,1.86]$.

\begin{figure}[H]
\centering
\includegraphics[width=0.48\textwidth]{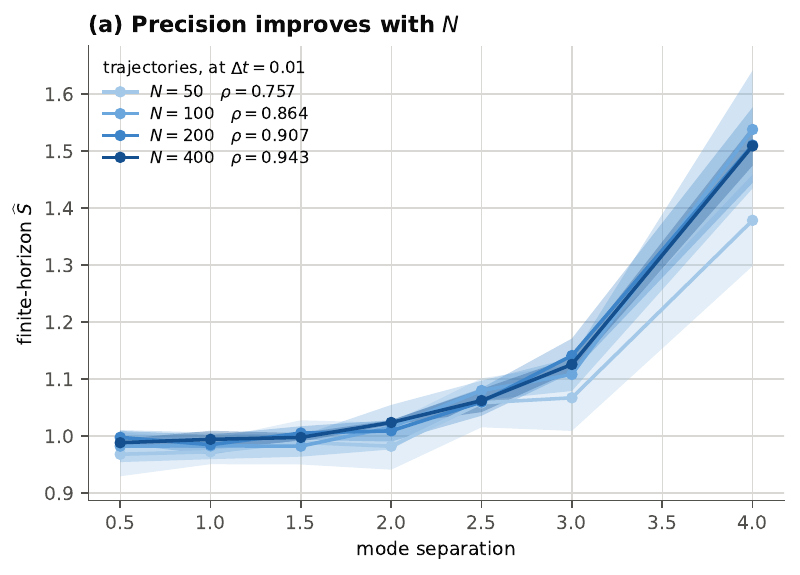}
\hfill
\includegraphics[width=0.48\textwidth]{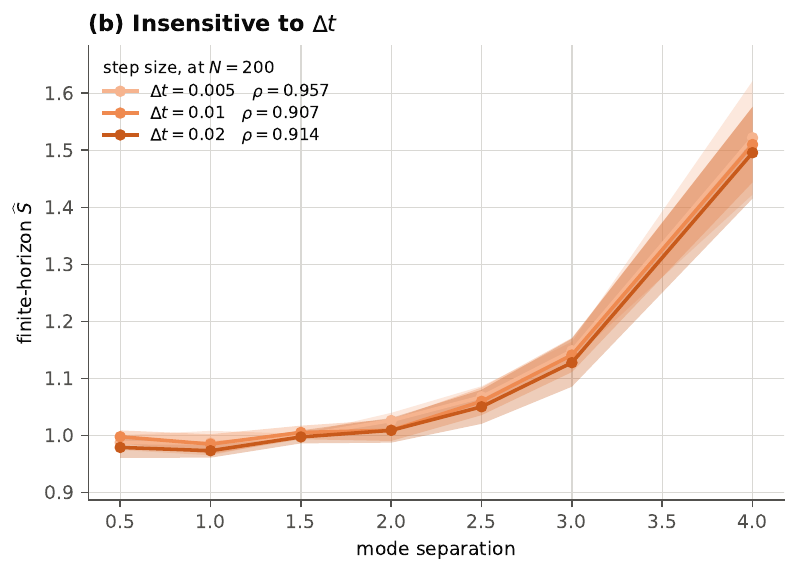}
\caption{\textbf{Numerical robustness of finite-horizon SSA} on a 10D bimodal
separation ladder, five seeds, shared vertical scale. (a) Trajectory count
$N$ at $\Delta t=0.01$; bands are $\pm1$ SD and the legend gives
Spearman$(\widehat S,\mathrm{MI})$ per configuration. (b) Step size $\Delta t$
at $N=200$. The design is a star about the baseline $(N{=}200,\Delta t{=}0.01)$,
so no $N\times\Delta t$ interaction is measured.}
\label{fig:rebuttal_numerics}
\end{figure}

\subsection{SDXL horizon and lag dependence}
\label{app:rebuttal_finite_horizon}

The SDXL sweep uses $N=200$, $\Delta t=0.005$ and $T=2000$ steps. Every cell
reaches the $1/N$ stopping floor within this horizon
(Appendix~\ref{app:sdxl_convergence}), so each $\widehat S$ integrates the
resolved signal rather than a truncated one, and the trajectory bootstrap
places the maximizer at an intermediate guidance scale for every body prompt
(Appendix~\ref{app:sdxl_bootstrap}).

\paragraph{Lag dependence of the elderly DA.} The terminal result is a long-lag
property, not a short-lag PCA effect: under both the analytic reference and a
conservative matched OU null, admissible lags give gender-axis Spearman
correlations in $[0.813,0.917]$. The direction rotates toward its terminal value
as the lag grows (up to $30.3^\circ$ over the admissible range), so the semantic
interpretation is stable but DA is not lag-independent.

\subsection{Scope of the analytic null}
\label{app:rebuttal_null}

Theorem~\ref{thm:limit-spectrum} has three defining assumptions: an
isotropic-Gaussian null, an exact score, and independent lag pairs. Gaussianity
specifies the null rather than the alternative, but score error, overlapping
pairs and nonlinear observation maps are genuine departures.

Both reported DA results are read at the terminal lag, where each trajectory
contributes a single $(X_0,X_T)$ pair, so the independent-pair assumption holds
exactly and the analytic edge applies as stated. At shorter lags the overlapping
windows of the trajectory-averaged $\hat C(\tau)$ are dependent, and an edge
evaluated at the nominal pair count is anti-conservative; where a shorter lag is
needed, matched Monte Carlo reproducing the trajectory geometry is the
appropriate substitute. Characterising the overlapping-pair spectrum
analytically is left to future work.

The exact-score assumption is tested by breaking it: we inject a frozen,
deterministic error field into the score of an isotropic-Gaussian target, again
with one independent pair per trajectory, and measure how often the top
eigenvalue exceeds the analytic edge (Figure~\ref{fig:rebuttal_null_score_error}).
At small perturbations the edge behaves as intended: no exceedances at the
SDXL-like aspect ratio $(d,N)=(768,200)$, and rates of $0.05$ and $0.13$ at the
alanine-like ratio $(40,2000)$ for $\alpha=0.1$ and $0.25$. At larger $\alpha$
the injected drift alters the process itself rather than adding noise to it, so
exceedances there are detections of real structure in a target that is no longer
the null. This finite grid neither extends the theorem nor supplies a universal
one-sided bound.

The SDXL VAE+CLIP observation map is nonlinear, so the analytic CLIP-space edge
remains a reference only. The pipeline-matched anisotropic null and a more
conservative uniform-rate OU surrogate give different margins at the terminal
lag; both reject, but they answer different questions. Significance claims must
therefore name the null rather than quote an unqualified multiple of the
analytic edge.

\begin{figure}[H]
\centering
\includegraphics[width=0.62\textwidth]{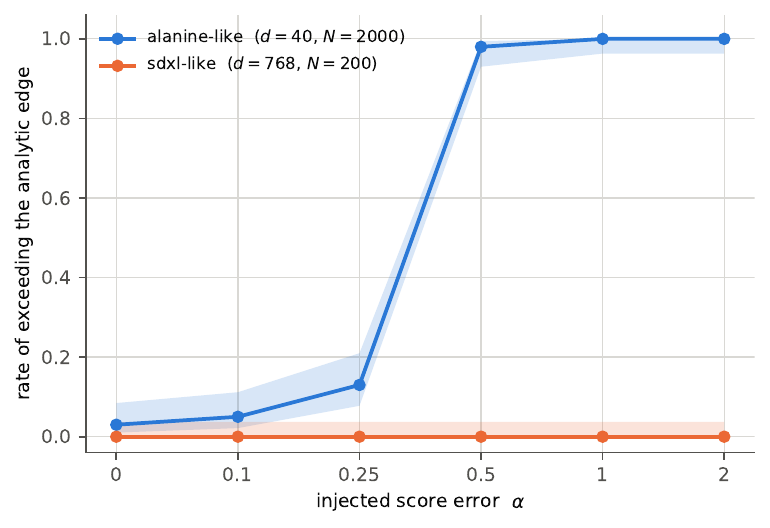}
\caption{\textbf{Exceedance of the analytic edge under injected score error.}
Isotropic-Gaussian null, nonconservative error family, one independent pair per
trajectory, 100 replicates per point; bands are Wilson 95\% intervals. The rate
stays near zero while the perturbation is small, and saturates at the
alanine-like ratio once the injected drift has altered the process itself.}
\label{fig:rebuttal_null_score_error}
\end{figure}

\subsection{Anisotropy, degenerate slow subspaces, and feature visibility}
\label{app:rebuttal_geometry}

Anisotropy produces slow relaxation without multiple modes. For
$f=\mathcal N(0,\Sigma)$ the canonical process is OU with
$\mu_i=\sigma_f^2/(2\lambda_i)$ and $w_i=\lambda_i/\sum_j\lambda_j$, so
Corollary~\ref{thm:ssa_spectral} depends on the covariance eigenvalue profile:
in two dimensions the population SSA is $1.32$ at condition number $5$, $1.79$
at $25$, and tends to $2$ as the condition number diverges. For the
parallel-ellipses example a unimodal Gaussian matched to the mixture covariance
$\operatorname{diag}(25,10)$ has population SSA $1.11$ against the measured
$\widehat S=1.822$, so anisotropy explains part but not most of the excess. A
unimodal Gaussian stays Gaussian in every one-dimensional projection, which is
why interpreting a DA direction requires a separate unimodality test. The
strictly unimodal flat-top density $f(x)\propto\exp[-\sum_i(|x_i|/3)^6]$ makes
the same point directly: $\widehat S=1.094\pm0.003$ in $d=10$ against
$0.992\pm0.013$ for its matched isotropic Gaussian, while failing to reject
unimodality (dip $p=0.98$, Silverman $p=0.37$).

Two further controls bound what DA can identify.
Figure~\ref{fig:rebuttal_geometry_gmm} uses a 10D four-mode mixture whose slow
plane is known by construction. Its finite-horizon SSA is at least
$8.043\pm0.060$ versus $0.996\pm0.002$ for a matched isotropic Gaussian, and the
slow plane is recovered to within $0.31^\circ$ in every seed. Yet
$\lambda_1/\lambda_2=1.03$, so no unique first direction exists and DA$_1$
rotates by up to $89.3^\circ$ across seeds; only the separation from the rest of
the spectrum ($\lambda_2/\lambda_3\approx2022$) makes the two-dimensional
subspace identifiable, via Davis--Kahan. The identifiable object is the slow
subspace, not DA$_1$.

Figure~\ref{fig:rebuttal_geometry_rings} shows the complementary failure: a
rotationally symmetric two-shell density in which the slow coordinate is radial.
In the Cartesian readout every direction is equivalent by construction, DA$_1$
is an arbitrary diameter, and the spectrum is nearly degenerate
($\lambda_1/\lambda_2=1.11$). Expressing the same trajectories in polar
coordinates makes the mode visible: DA$_1$ places a fraction $0.999$ of its
weight on the radial coordinate, with $\lambda_1/\lambda_2=3.87$. Whether DA can
find a slow mode is thus a property of the readout, not only of the density.

\begin{figure}[H]
\centering
\includegraphics[width=0.92\textwidth]{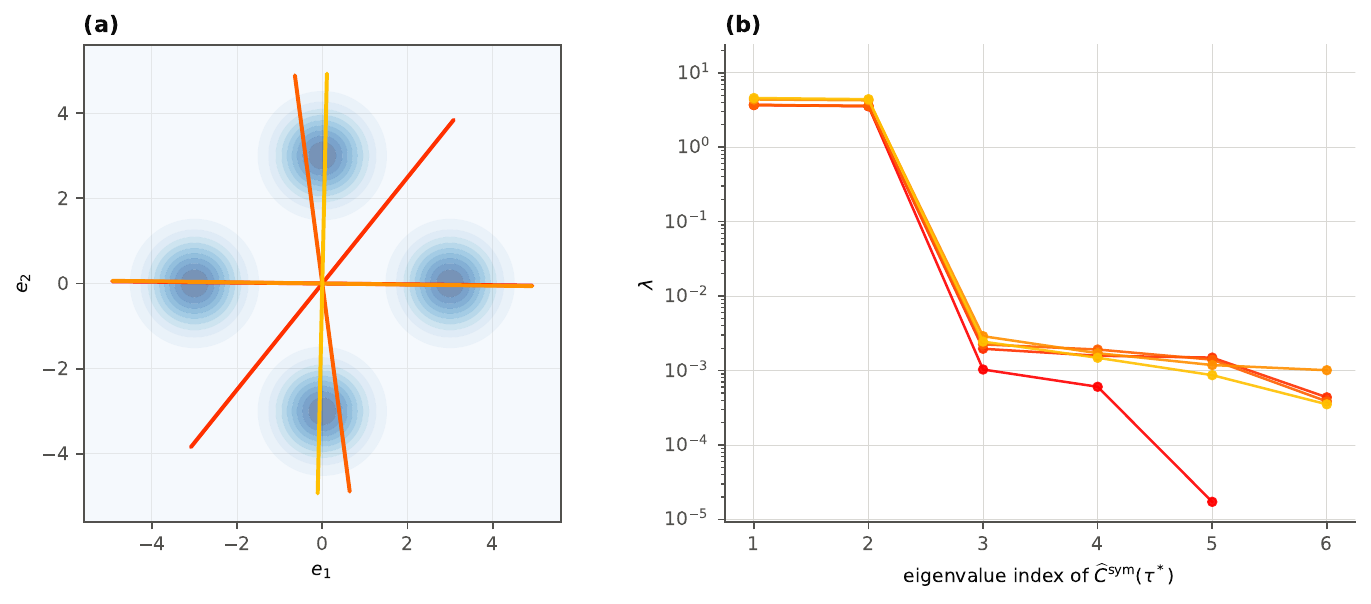}
\caption{\textbf{The identifiable object is the slow subspace, not DA$_1$.}
10D four-mode circular mixture with slow plane $\mathrm{span}(e_1,e_2)$.
(a) The density restricted to that plane, with each of five seeds' DA$_1$ drawn
as a double-headed axis (an eigenvector's sign is arbitrary): the plane is
recovered to $\leq0.31^\circ$ but DA$_1$ within it spans $89.3^\circ$.
(b) The spectrum of $\hat C^{\mathrm{sym}}(\tau^*)$ per seed. The leading pair
is stable in magnitude across seeds and internally degenerate
($\lambda_1/\lambda_2=1.03$), while the remaining eigenvalues are finite-sample
noise --- scattering over two orders of magnitude and not sign-definite --- yet
sit stably far below the pair ($\lambda_2/\lambda_3\geq1472$ in every seed).
Indices beyond $6$ lie at or below that noise level and are omitted.}
\label{fig:rebuttal_geometry_gmm}
\end{figure}

\begin{figure}[H]
\centering
\includegraphics[width=0.92\textwidth]{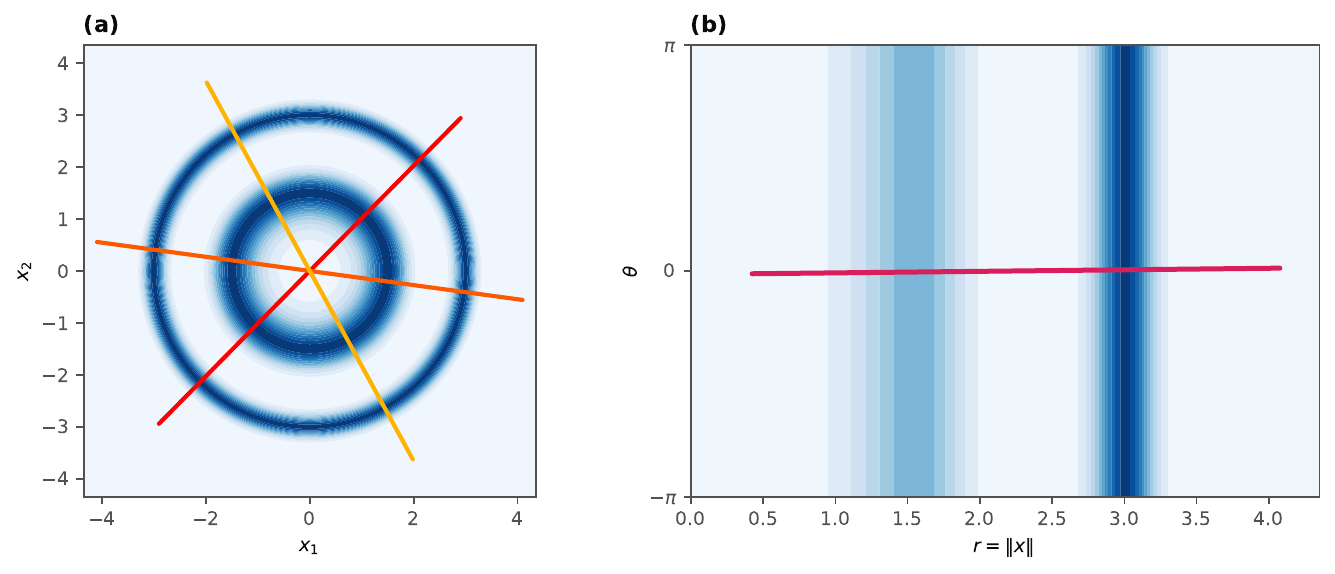}
\caption{\textbf{The readout decides whether DA can find the slow mode.}
Rotationally symmetric two-shell density, shells at $r=1.5$ and $r=3.0$.
(a) Cartesian readout: three seeds' DA$_1$ axes point in unrelated directions,
as they must when all directions are equivalent by construction.
(b) The same trajectories in polar coordinates, where the shells occupy two
values of $r$; DA$_1$ carries radial weight $0.999$ (median over seeds, range
$0.9991$--$0.9996$) and the spectral gap rises to $3.87$ from $1.11$ in the
Cartesian readout. The drawn axis is tilted by the measured radial--angular
weight split. Angle enters the readout as $(\cos\theta,\sin\theta)$, since a
circular coordinate cannot be used directly in an autocovariance.}
\label{fig:rebuttal_geometry_rings}
\end{figure}

\section{Code Availability}
\label{app:code}

Code to reproduce all experiments, along with configuration files and a README documenting environment setup and per-experiment commands, is available at \url{\repomain}.

\newpage
\section*{NeurIPS Paper Checklist}

\begin{enumerate}

\item {\bf Claims}
    \item[] Question: Do the main claims made in the abstract and introduction accurately reflect the paper's contributions and scope?
    \item[] Answer: \answerYes{}
        \item[] Justification: The abstract and introduction state the paper's contributions: (i) a canonical similarity-invariant diffusion uniquely determined by any density (Theorem~\ref{thm:uniqueness}, Proposition~\ref{prop:similarity_invariance}); (ii) SSA and DA as two readouts of its autocovariance (Sections~\ref{sec:ssa},~\ref{sec:da}); (iii) a closed-form null spectrum for the empirical autocovariance generalizing Marchenko--Pastur, with an analytic upper edge $\lambda_+(\tau)$ used for lag selection (Theorem~\ref{thm:limit-spectrum}); (iv) SSA captures structure invisible to entropy on SDXL (Section~\ref{sec:exp_sdxl}, Table~\ref{tab:sdxl_guidance}) and DA recovers slow dihedrals on alanine dipeptide from static samples (Section~\ref{sec:exp_molecular}, Table~\ref{tab:alanine}). Scope and limitations (linear directions only, dependence on score estimate) are stated in Section~\ref{sec:discussion}.
    \item[] Guidelines:
    \begin{itemize}
        \item The answer \answerNA{} means that the abstract and introduction do not include the claims made in the paper.
        \item The abstract and/or introduction should clearly state the claims made, including the contributions made in the paper and important assumptions and limitations. A \answerNo{} or \answerNA{} answer to this question will not be perceived well by the reviewers. 
        \item The claims made should match theoretical and experimental results, and reflect how much the results can be expected to generalize to other settings. 
        \item It is fine to include aspirational goals as motivation as long as it is clear that these goals are not attained by the paper. 
    \end{itemize}

\item {\bf Limitations}
    \item[] Question: Does the paper discuss the limitations of the work performed by the authors?
    \item[] Answer: \answerYes{}
       \item[] Justification: Section~\ref{sec:discussion} contains a dedicated Limitations paragraph covering: (i) DA produces only linear directions and cannot capture nonlinear mode geometry; (ii) the framework's outputs are only as good as the input score, with calibration (null fidelity) and interpretation (signal fidelity) decoupling as discussed in Remark~\ref{rem:score_oracle}; (iii) the dominant cost is $N \times T$ score evaluations, which on large generative models bounds the reachable $(\Delta t, T, N)$ regimes; SSA's compute parameters affect resolution but not fixed-profile rankings (Proposition~\ref{prop:ssa_monotonicity}); arbitrary comparisons require lag-window stability checks (Appendix~\ref{app:hyperparameters}). Section~\ref{sec:discussion} also notes that linearity is partly a feature (it enables the closed-form noise floor of Theorem~\ref{thm:limit-spectrum}) and addresses scope of claims regarding deployment risk.
    \item[] Guidelines:
    \begin{itemize}
        \item The answer \answerNA{} means that the paper has no limitation while the answer \answerNo{} means that the paper has limitations, but those are not discussed in the paper. 
        \item The authors are encouraged to create a separate ``Limitations'' section in their paper.
        \item The paper should point out any strong assumptions and how robust the results are to violations of these assumptions (e.g., independence assumptions, noiseless settings, model well-specification, asymptotic approximations only holding locally). The authors should reflect on how these assumptions might be violated in practice and what the implications would be.
        \item The authors should reflect on the scope of the claims made, e.g., if the approach was only tested on a few datasets or with a few runs. In general, empirical results often depend on implicit assumptions, which should be articulated.
        \item The authors should reflect on the factors that influence the performance of the approach. For example, a facial recognition algorithm may perform poorly when image resolution is low or images are taken in low lighting. Or a speech-to-text system might not be used reliably to provide closed captions for online lectures because it fails to handle technical jargon.
        \item The authors should discuss the computational efficiency of the proposed algorithms and how they scale with dataset size.
        \item If applicable, the authors should discuss possible limitations of their approach to address problems of privacy and fairness.
        \item While the authors might fear that complete honesty about limitations might be used by reviewers as grounds for rejection, a worse outcome might be that reviewers discover limitations that aren't acknowledged in the paper. The authors should use their best judgment and recognize that individual actions in favor of transparency play an important role in developing norms that preserve the integrity of the community. Reviewers will be specifically instructed to not penalize honesty concerning limitations.
    \end{itemize}

\item {\bf Theory assumptions and proofs}
    \item[] Question: For each theoretical result, does the paper provide the full set of assumptions and a complete (and correct) proof?
    \item[] Answer: \answerYes{}
        \item[] Justification: All theoretical results are numbered, cross-referenced, and proved. Theorem~\ref{thm:uniqueness} states regularity assumptions on $f$ (smooth, strictly positive, integrable by parts) and the reversibility plus fixed-$A_0$ conditions; proof in Appendix~\ref{app:uniqueness_proof}. Proposition~\ref{prop:similarity_invariance} is proved in Appendix~\ref{app:similarity_proof}. Lemma~\ref{lem:spectral} and Corollary~\ref{thm:ssa_spectral} are proved in Appendix~\ref{app:spectral}. Proposition~\ref{prop:ssa_monotonicity} is proved in Appendix~\ref{app:monotonicity_proof}, building on Proposition~\ref{prop:monotonicity} and Corollary~\ref{cor:metastable_variance} (Appendix~\ref{app:spectral_kernel}). Proposition~\ref{prop:da_convergence} (DA1 convergence via Davis--Kahan) is proved in Appendix~\ref{app:da_convergence_proof}; the subspace statement for $m \geq 2$ is proved at subspace level only, with per-eigenvector identification stated as an empirically validated claim rather than a theorem (Section~\ref{sec:da}, Appendix~\ref{app:da_convergence_proof}). Theorem~\ref{thm:limit-spectrum} states the i.i.d.\ isotropic-Gaussian-null assumption explicitly and is proved in Appendix~\ref{app:limit-spectrum} via the exact Wishart-difference decomposition (Lemma~\ref{lem:decomp}) and free additive convolution applied to the Marchenko--Pastur law, with three independent derivations of the atom mass for cross-validation. Robustness results (Proposition~\ref{prop:robustness_ito}) are proved in Appendix~\ref{app:robustness}; the kinetic-family statements (Smoluchowski limit, Eyring--Kramers asymptotics, slow-mode ordering implication) are presented as asymptotic implications of standard metastability results cited to the literature.
    \item[] Guidelines:
    \begin{itemize}
        \item The answer \answerNA{} means that the paper does not include theoretical results. 
        \item All the theorems, formulas, and proofs in the paper should be numbered and cross-referenced.
        \item All assumptions should be clearly stated or referenced in the statement of any theorems.
        \item The proofs can either appear in the main paper or the supplemental material, but if they appear in the supplemental material, the authors are encouraged to provide a short proof sketch to provide intuition. 
        \item Inversely, any informal proof provided in the core of the paper should be complemented by formal proofs provided in appendix or supplemental material.
        \item Theorems and Lemmas that the proof relies upon should be properly referenced. 
    \end{itemize}

    \item {\bf Experimental result reproducibility}
    \item[] Question: Does the paper fully disclose all the information needed to reproduce the main experimental results of the paper to the extent that it affects the main claims and/or conclusions of the paper (regardless of whether the code and data are provided or not)?
    \item[] Answer: \answerYes{}
        \item[] Justification: All experimental hyperparameters, data sources, and algorithmic choices needed to reproduce the main results are documented. Algorithm~\ref{alg:ssa_da} consolidates the SSA and DA procedure end-to-end. GMM simulation details (trajectory count, step size, horizon, seeding, four-sweep parameter ranges, random-sweep configuration) are in Appendix~\ref{app:gmm_details}. The SDXL pipeline (model identifier, latent dimensions, Tweedie score oracle at $t=1$, VAE+CLIP embedding, snapshot schedule, guidance scales, $k$NN entropy settings) is in Appendices~\ref{app:sdxl_pipeline} and~\ref{app:sdxl_guidance_details}, with the DA-specific finer step $\Delta t = 0.0025$ and lag $\tau^* = 2{,}000$ documented in Appendix~\ref{app:sdxl_elderly}. Alanine dipeptide details (KDE bandwidth $h_{\mathrm{Scott}}/2 = 0.0091$ with 2000 centers, 45 pairwise heavy-atom distances, MDShare $3 \times 250$\,ns at 300\,K, TICA at PyEMMA default lag $= 10$) are in Appendices~\ref{app:alanine_details} and~\ref{app:tica_lag_validation}. The operational noise-floor procedure used for DA lag selection is given in Algorithm~\ref{alg:operational_floor} (Appendix~\ref{app:validation}). Hyperparameter selection rationale and sensitivity analyses (including $\Delta t$ Richardson validation) are in Appendix~\ref{app:hyperparameters}. All datasets and pretrained models are public benchmarks: MDShare alanine dipeptide, SDXL Base~1.0, CLIP ViT-L/14.
    \item[] Guidelines:
    \begin{itemize}
        \item The answer \answerNA{} means that the paper does not include experiments.
        \item If the paper includes experiments, a \answerNo{} answer to this question will not be perceived well by the reviewers: Making the paper reproducible is important, regardless of whether the code and data are provided or not.
        \item If the contribution is a dataset and\slash or model, the authors should describe the steps taken to make their results reproducible or verifiable. 
        \item Depending on the contribution, reproducibility can be accomplished in various ways. For example, if the contribution is a novel architecture, describing the architecture fully might suffice, or if the contribution is a specific model and empirical evaluation, it may be necessary to either make it possible for others to replicate the model with the same dataset, or provide access to the model. In general. releasing code and data is often one good way to accomplish this, but reproducibility can also be provided via detailed instructions for how to replicate the results, access to a hosted model (e.g., in the case of a large language model), releasing of a model checkpoint, or other means that are appropriate to the research performed.
        \item While NeurIPS does not require releasing code, the conference does require all submissions to provide some reasonable avenue for reproducibility, which may depend on the nature of the contribution. For example
        \begin{enumerate}
            \item If the contribution is primarily a new algorithm, the paper should make it clear how to reproduce that algorithm.
            \item If the contribution is primarily a new model architecture, the paper should describe the architecture clearly and fully.
            \item If the contribution is a new model (e.g., a large language model), then there should either be a way to access this model for reproducing the results or a way to reproduce the model (e.g., with an open-source dataset or instructions for how to construct the dataset).
            \item We recognize that reproducibility may be tricky in some cases, in which case authors are welcome to describe the particular way they provide for reproducibility. In the case of closed-source models, it may be that access to the model is limited in some way (e.g., to registered users), but it should be possible for other researchers to have some path to reproducing or verifying the results.
        \end{enumerate}
    \end{itemize}

\item {\bf Open access to data and code}
    \item[] Question: Does the paper provide open access to the data and code, with sufficient instructions to faithfully reproduce the main experimental results, as described in supplemental material?
    \item[] Answer: \answerYes{}
    \item[] Justification: Code to reproduce all experiments is provided at \url{\repomain}, with a README documenting environment setup and per-experiment reproduction commands. All datasets and pretrained models used are public: MDShare alanine dipeptide trajectories (via the \texttt{mdshare} Python package), SDXL Base~1.0 (\texttt{stabilityai/stable-diffusion-xl-base-1.0} on HuggingFace), and CLIP ViT-L/14 (\texttt{openai/clip-vit-large-patch14}). Algorithms~\ref{alg:ssa_da} and~\ref{alg:operational_floor} together with the hyperparameters in the appendices specify the method in sufficient detail to permit independent reimplementation.
    \item[] Guidelines:
    \begin{itemize}
        \item The answer \answerNA{} means that paper does not include experiments requiring code.
        \item Please see the NeurIPS code and data submission guidelines (\url{https://neurips.cc/public/guides/CodeSubmissionPolicy}) for more details.
        \item While we encourage the release of code and data, we understand that this might not be possible, so \answerNo{} is an acceptable answer. Papers cannot be rejected simply for not including code, unless this is central to the contribution (e.g., for a new open-source benchmark).
        \item The instructions should contain the exact command and environment needed to run to reproduce the results. See the NeurIPS code and data submission guidelines (\url{https://neurips.cc/public/guides/CodeSubmissionPolicy}) for more details.
        \item The authors should provide instructions on data access and preparation, including how to access the raw data, preprocessed data, intermediate data, and generated data, etc.
        \item The authors should provide scripts to reproduce all experimental results for the new proposed method and baselines. If only a subset of experiments are reproducible, they should state which ones are omitted from the script and why.
        \item At submission time, to preserve anonymity, the authors should release anonymized versions (if applicable).
        \item Providing as much information as possible in supplemental material (appended to the paper) is recommended, but including URLs to data and code is permitted.
    \end{itemize}

\item {\bf Experimental setting/details}
    \item[] Question: Does the paper specify all the training and test details (e.g., data splits, hyperparameters, how they were chosen, type of optimizer) necessary to understand the results?
    \item[] Answer: \answerYes{}
\item[] Justification: All experimental details needed to understand the results are specified. The paper does not train new models; experiments use pretrained SDXL Base~1.0 and CLIP ViT-L/14, the public MDShare alanine dipeptide trajectories, and synthetic GMM samples. Numerical-integration details (Euler--Maruyama step $\Delta t$, horizon $T$, trajectory count $N$) are stated per experiment: GMM ($\Delta t=0.01$, $T_{\mathrm{phys}}=200$\,s, $N=200$; Appendix~\ref{app:gmm_details}); SDXL SSA ($\Delta t=0.005$, $T=2{,}000$, $N=200$; Appendix~\ref{app:sdxl_guidance_details}); SDXL DA ($\Delta t=0.0025$, $\tau^*=2{,}000$, $N=200$; Appendix~\ref{app:sdxl_elderly}); alanine ($\tau^*=10{,}000$\,ps, KDE bandwidth $h_{\mathrm{Scott}}/2$, 2000 KDE centers; Appendix~\ref{app:alanine_details}). The rationale for hyperparameter choices and their sensitivity is in Appendix~\ref{app:hyperparameters} (with empirical $\Delta t$ Richardson validation in Appendix~\ref{app:dt_robustness}); the lag-selection rule for DA is in Section~\ref{sec:da} and Appendix~\ref{app:lag_selection}; the SSA truncation rule is derived in Appendix~\ref{app:rho_clt}. Significance-testing protocols (analytic null floor, matched Monte Carlo null, Hartigan/Silverman bimodality tests) are documented in Appendix~\ref{app:validation}.
    
    \item[] Guidelines:
    \begin{itemize}
        \item The answer \answerNA{} means that the paper does not include experiments.
        \item The experimental setting should be presented in the core of the paper to a level of detail that is necessary to appreciate the results and make sense of them.
        \item The full details can be provided either with the code, in appendix, or as supplemental material.
    \end{itemize}

\item {\bf Experiment statistical significance}
    \item[] Question: Does the paper report error bars suitably and correctly defined or other appropriate information about the statistical significance of the experiments?
    \item[] Answer: \answerYes{}
\item[] Justification: Statistical significance is reported for all main experimental claims, with the relevant factor of variability stated in each case. GMM sweeps show mean $\pm 1$\,SD over 10 seeds (Figure~\ref{fig:gmm_sweeps}); the random GMM sweep reports Spearman $\rho=0.890$ with $p=5.5\times10^{-18}$ (Figure~\ref{fig:random_sweep}, Appendix~\ref{app:gmm_random}). SDXL SSA cells include trajectory bootstrap standard errors ($B=2{,}000$; Table~\ref{tab:sdxl_guidance}, Appendix~\ref{app:sdxl_bootstrap}); the GMM mode-count sweep includes paired lag bootstrap with pairwise ordering probabilities ($B=10{,}000$; Appendix~\ref{app:gmm_bootstrap}). Alanine correlations include frame-bootstrap standard errors over $750{,}000$ MD frames ($B=5{,}000$; Appendix~\ref{app:alanine_robustness}). Significance against the iso-Gaussian null uses Bonferroni-corrected Monte Carlo tests with $B=1{,}000$ (alanine, Table~\ref{tab:alanine_bonferroni}) and $B=10{,}000$ (SDXL elderly, Table~\ref{tab:elderly_bonferroni}); 1D bimodality is tested via Hartigan's dip and Silverman's bandwidth tests with reported $p$-values (Appendix~\ref{app:validation}). Free-convolution validation pools eigenvalues over 3 seeds per $(\gamma,\tau)$ panel (Appendix~\ref{app:scope}). Reported error bars are either $\pm 1$\,SD over independent seeds or bootstrap standard errors, identified explicitly in each figure caption and table.
    \item[] Guidelines:
    \begin{itemize}
        \item The answer \answerNA{} means that the paper does not include experiments.
        \item The authors should answer \answerYes{} if the results are accompanied by error bars, confidence intervals, or statistical significance tests, at least for the experiments that support the main claims of the paper.
        \item The factors of variability that the error bars are capturing should be clearly stated (for example, train/test split, initialization, random drawing of some parameter, or overall run with given experimental conditions).
        \item The method for calculating the error bars should be explained (closed form formula, call to a library function, bootstrap, etc.)
        \item The assumptions made should be given (e.g., Normally distributed errors).
        \item It should be clear whether the error bar is the standard deviation or the standard error of the mean.
        \item It is OK to report 1-sigma error bars, but one should state it. The authors should preferably report a 2-sigma error bar than state that they have a 96\% CI, if the hypothesis of Normality of errors is not verified.
        \item For asymmetric distributions, the authors should be careful not to show in tables or figures symmetric error bars that would yield results that are out of range (e.g., negative error rates).
        \item If error bars are reported in tables or plots, the authors should explain in the text how they were calculated and reference the corresponding figures or tables in the text.
    \end{itemize}

\item {\bf Experiments compute resources}
    \item[] Question: For each experiment, does the paper provide sufficient information on the computer resources (type of compute workers, memory, time of execution) needed to reproduce the experiments?
    \item[] Answer: \answerYes{}
    \item[] Justification: All experiments were run on single GPUs (NVIDIA L40S 48\,GB
or L4 24\,GB) on an internal academic SLURM cluster; no experiment required
multi-GPU or multi-node execution. The software environment uses Python 3.12,
PyTorch 2.2.2 (CUDA 12.1), diffusers 0.30.0, and transformers 4.40.0. No model
training was performed: all GPU compute consists of forward-pass inference
through pretrained models (SDXL Base~1.0, CLIP ViT-L/14) and downstream
Langevin simulation and analysis. Total compute to reproduce the experiments
reported in the paper is approximately $1{,}100$ GPU-hours, with a per-experiment
breakdown given in Appendix~\ref{app:compute} (Table~\ref{tab:compute_breakdown}).
The dominant cost is the SDXL guidance sweep: Langevin simulation across 28
prompt$\times$scale configurations ($\sim$532 GPU-h on L40S, each requiring 200
trajectories with Tweedie score evaluations on SDXL Base~1.0 at $1024 \times 1024$
plus CLIP ViT-L/14 embedding of snapshots), companion image generation for
$k$NN entropy estimation ($\sim$196 GPU-h), and trajectory-averaged recomputes
($\sim$192 GPU-h). The SDXL elderly DA analysis, Monte-Carlo nulls, and
bootstrap CIs together account for $\sim$130 GPU-h. Synthetic experiments
(GMM sweeps, random sweep, $\Delta t$ robustness, planted-spikes validation)
and the alanine dipeptide experiment together require under 40 GPU-hours,
with several auxiliary figures generated on CPU. Including exploratory runs,
superseded configurations, and pilot sweeps not reported in the paper, total
project compute was approximately 3$\times$ the reported figure.
    \item[] Guidelines:
    \begin{itemize}
        \item The answer \answerNA{} means that the paper does not include experiments.
        \item The paper should indicate the type of compute workers CPU or GPU, internal cluster, or cloud provider, including relevant memory and storage.
        \item The paper should provide the amount of compute required for each of the individual experimental runs as well as estimate the total compute. 
        \item The paper should disclose whether the full research project required more compute than the experiments reported in the paper (e.g., preliminary or failed experiments that didn't make it into the paper). 
    \end{itemize}
    
\item {\bf Code of ethics}
    \item[] Question: Does the research conducted in the paper conform, in every respect, with the NeurIPS Code of Ethics \url{https://neurips.cc/public/EthicsGuidelines}?
    \item[] Answer: \answerYes{}
    \item[] Justification: The research uses only publicly available pretrained
    models and generates synthetic data. No human subjects are involved. we
    discuss implications in the Broader Impact statement.
    \item[] Guidelines:
    \begin{itemize}
        \item The answer \answerNA{} means that the authors have not reviewed the NeurIPS Code of Ethics.
        \item If the authors answer \answerNo, they should explain the special circumstances that require a deviation from the Code of Ethics.
        \item The authors should make sure to preserve anonymity (e.g., if there is a special consideration due to laws or regulations in their jurisdiction).
    \end{itemize}

\item {\bf Broader impacts}
    \item[] Question: Does the paper discuss both potential positive societal impacts and negative societal impacts of the work performed?
    \item[] Answer: \answerYes{}
\item[] Justification: As discussed in Section~\ref{sec:discussion}, the paper introduces a measurement tool applied to existing generative models rather than a new generative capability, so we see no direct deployment risk beyond that of the underlying models analyzed.
    \item[] Guidelines:
    \begin{itemize}
        \item The answer \answerNA{} means that there is no societal impact of the work performed.
        \item If the authors answer \answerNA{} or \answerNo, they should explain why their work has no societal impact or why the paper does not address societal impact.
        \item Examples of negative societal impacts include potential malicious or unintended uses (e.g., disinformation, generating fake profiles, surveillance), fairness considerations (e.g., deployment of technologies that could make decisions that unfairly impact specific groups), privacy considerations, and security considerations.
        \item The conference expects that many papers will be foundational research and not tied to particular applications, let alone deployments. However, if there is a direct path to any negative applications, the authors should point it out. For example, it is legitimate to point out that an improvement in the quality of generative models could be used to generate Deepfakes for disinformation. On the other hand, it is not needed to point out that a generic algorithm for optimizing neural networks could enable people to train models that generate Deepfakes faster.
        \item The authors should consider possible harms that could arise when the technology is being used as intended and functioning correctly, harms that could arise when the technology is being used as intended but gives incorrect results, and harms following from (intentional or unintentional) misuse of the technology.
        \item If there are negative societal impacts, the authors could also discuss possible mitigation strategies (e.g., gated release of models, providing defenses in addition to attacks, mechanisms for monitoring misuse, mechanisms to monitor how a system learns from feedback over time, improving the efficiency and accessibility of ML).
    \end{itemize}
    
\item {\bf Safeguards}
    \item[] Question: Does the paper describe safeguards that have been put in place for responsible release of data or models that have a high risk for misuse (e.g., pre-trained language models, image generators, or scraped datasets)?
    \item[] Answer: \answerNA{}
    \item[] Justification: We do not release new pretrained generative models or datasets with misuse risk. Released materials are limited to measurement and analysis code for the purpose of methodological reproduction.
    \item[] Guidelines:
    \begin{itemize}
        \item The answer \answerNA{} means that the paper poses no such risks.
        \item Released models that have a high risk for misuse or dual-use should be released with necessary safeguards to allow for controlled use of the model, for example by requiring that users adhere to usage guidelines or restrictions to access the model or implementing safety filters. 
        \item Datasets that have been scraped from the Internet could pose safety risks. The authors should describe how they avoided releasing unsafe images.
        \item We recognize that providing effective safeguards is challenging, and many papers do not require this, but we encourage authors to take this into account and make a best faith effort.
    \end{itemize}

\item {\bf Licenses for existing assets}
    \item[] Question: Are the creators or original owners of assets (e.g., code, data, models), used in the paper, properly credited and are the license and terms of use explicitly mentioned and properly respected?
    \item[] Answer: \answerYes{}
    \item[] Justification: All pretrained components are cited in Section~\ref{sec:experiments} and full license details for all assets are provided in Table~\ref{tab:licenses}
    \item[] Guidelines:
    \begin{itemize}
        \item The answer \answerNA{} means that the paper does not use existing assets.
        \item The authors should cite the original paper that produced the code package or dataset.
        \item The authors should state which version of the asset is used and, if possible, include a URL.
        \item The name of the license (e.g., CC-BY 4.0) should be included for each asset.
        \item For scraped data from a particular source (e.g., website), the copyright and terms of service of that source should be provided.
        \item If assets are released, the license, copyright information, and terms of use in the package should be provided. For popular datasets, \url{paperswithcode.com/datasets} has curated licenses for some datasets. Their licensing guide can help determine the license of a dataset.
        \item For existing datasets that are re-packaged, both the original license and the license of the derived asset (if it has changed) should be provided.
        \item If this information is not available online, the authors are encouraged to reach out to the asset's creators.
    \end{itemize}

\item {\bf New assets}
    \item[] Question: Are new assets introduced in the paper well documented and is the documentation provided alongside the assets?
    \item[] Answer: \answerYes{}
    \item[] Justification: The released code repository (\url{\repomain}) contains a README, runnable reproduction scripts, and configuration files documenting all hyperparameters; full hyperparameter rationale and sensitivity analysis are in Appendix~\ref{app:hyperparameters}.
    \item[] Guidelines:
    \begin{itemize}
        \item The answer \answerNA{} means that the paper does not release new assets.
        \item Researchers should communicate the details of the dataset\slash code\slash model as part of their submissions via structured templates. This includes details about training, license, limitations, etc. 
        \item The paper should discuss whether and how consent was obtained from people whose asset is used.
        \item At submission time, remember to anonymize your assets (if applicable). You can either create an anonymized URL or include an anonymized zip file.
    \end{itemize}

\item {\bf Crowdsourcing and research with human subjects}
    \item[] Question: For crowdsourcing experiments and research with human subjects, does the paper include the full text of instructions given to participants and screenshots, if applicable, as well as details about compensation (if any)? 
    \item[] Answer: \answerNA{}
   \item[] Justification: The paper does not involve crowdsourcing or research with human subjects.
    \item[] Guidelines:
    \begin{itemize}
        \item The answer \answerNA{} means that the paper does not involve crowdsourcing nor research with human subjects.
        \item Including this information in the supplemental material is fine, but if the main contribution of the paper involves human subjects, then as much detail as possible should be included in the main paper. 
        \item According to the NeurIPS Code of Ethics, workers involved in data collection, curation, or other labor should be paid at least the minimum wage in the country of the data collector. 
    \end{itemize}

\item {\bf Institutional review board (IRB) approvals or equivalent for research with human subjects}
    \item[] Question: Does the paper describe potential risks incurred by study participants, whether such risks were disclosed to the subjects, and whether Institutional Review Board (IRB) approvals (or an equivalent approval/review based on the requirements of your country or institution) were obtained?
    \item[] Answer: \answerNA{}
    \item[] Justification: The paper does not involve crowdsourcing or research with human subjects.
    \begin{itemize}
        \item The answer \answerNA{} means that the paper does not involve crowdsourcing nor research with human subjects.
        \item Depending on the country in which research is conducted, IRB approval (or equivalent) may be required for any human subjects research. If you obtained IRB approval, you should clearly state this in the paper. 
        \item We recognize that the procedures for this may vary significantly between institutions and locations, and we expect authors to adhere to the NeurIPS Code of Ethics and the guidelines for their institution. 
        \item For initial submissions, do not include any information that would break anonymity (if applicable), such as the institution conducting the review.
    \end{itemize}

\item {\bf Declaration of LLM usage}
    \item[] Question: Does the paper describe the usage of LLMs if it is an important, original, or non-standard component of the core methods in this research? Note that if the LLM is used only for writing, editing, or formatting purposes and does \emph{not} impact the core methodology, scientific rigor, or originality of the research, declaration is not required.
    \item[] Answer: \answerNA{}
    \item[] Justification: LLMs were not used as a component of the core
    methodology, scientific rigor, or originality of this work. LLM assistance
    was used only for writing, editing, and formatting, which per NeurIPS
    policy does not require declaration.
    \begin{itemize}
        \item The answer \answerNA{} means that the core method development in this research does not involve LLMs as any important, original, or non-standard components.
        \item Please refer to our LLM policy in the NeurIPS handbook for what should or should not be described.
    \end{itemize}

\end{enumerate}

\end{document}